\documentclass[twoside,11pt]{article}

\usepackage{blindtext}

\usepackage{jmlr2e}

\usepackage{lastpage}
\jmlrheading{23}{2025}{(Submitted)}{10/25; Revised -}{-}{21-0000}{Xi He}

\usepackage{color}
\usepackage{array}
\usepackage{amsmath}
\usepackage{amssymb}
\usepackage{graphicx}
\usepackage{varwidth}
\usepackage{float}

\usepackage{tikz}

\makeatletter

\providecommand{\tabularnewline}{\\}

\newcommand*\LyXZeroWidthSpace{\hspace{0pt}}
\providecommand{\tabularnewline}{\\}
\newenvironment{cellvarwidth}[1][t]
{\begin{varwidth}[#1]{\linewidth}}
	{\@finalstrut\@arstrutbox\end{varwidth}}
\floatstyle{ruled}
\newfloat{algorithm}{tbp}{loa}
\providecommand{\algorithmname}{Algorithm}
\floatname{algorithm}{\protect\algorithmname}

\newtheorem{fact}{Fact}

\ShortHeadings{Foundational theory for decision trees}{Xi He}
\firstpageno{1}

\begin{document}

\title{Foundational theory for optimal decision tree problems. \\
	I. Algorithmic and geometric foundation.}

\author{\name Xi He \email xihe@pku.edu.cn\\
	\addr School of Computer Science\\
	Peking University\\
	Beijing, 100084, China\\}

\editor{My editor}

\maketitle

\begin{abstract}
In 2017, Bertsimas and Dunn proposed the first optimal decision tree
(ODT) algorithm using general-purpose solvers, sparking significant
interest in subsequent studies. However, existing \emph{combinatorial
	methods}—primarily based on branch-and-bound method—focus on much
simpler variants, such as trees with binary feature data or axis-parallel
hyperplanes. Although the first combinatorial algorithm appeared in
1989, little subsequent work has directly built upon it. Researchers
have proposed various algorithms with different data structures and
frameworks, but few adopt a unified framework comparable to general-purpose
solvers. Many so-called optimal algorithms rely on ad-hoc algorithmic
processes and bounding strategies, with complex pseudo-code reducing
accessibility and ambiguously stated claims hindering independent
validation. However, optimal algorithms require rigorous verification—they
must provably guarantee optimal solutions for a given search space
and objective. Previous BnB algorithms often fail to meet this level
of rigor, rendering claims of fully reliable optimality, at best,
premature.

We argue that the root cause lies in the absence of a formal framework
for designing optimal algorithms using combinatorial methods and standardized
criteria for verifying correctness. Any algorithm claiming optimality
must be formally proven. To address this, in the first paper (part
I) of this series of two, we examined the ambiguity and informality
in studies of ODT problems and discussed the consequences of relying
on informal formalisms in optimal algorithm design. We then presented
criteria that an optimal algorithm should satisfy to ensure a verifiable
and transparent design process.

Beyond identifying limitations, our primary goal—and the central contribution
of this work—is to develop a foundational theory for solving ODT problems.
We aim to establish a comprehensive algorithmic framework for \emph{all}
optimal decision tree problems studied in machine learning, encompassing
\emph{axis-parallel, hyperplane}, \emph{hypersurface decision trees},
and those defined over \emph{binary feature data}. In Part I, we introduce
\emph{five} rigorous formulations of ODT problems: three for size-constrained
trees and one for depth-constrained trees, and one for trees over
binary feature data. These definitions are stated unambiguously through
\emph{executable recursive programs}, satisfying all criteria we propose
for a formal specification. In this sense, they resemble the ``standard
form'' used in the study of general-purpose solvers.

Grounded in algebraic programming theory—a relational formalism for
deriving correct-by-construction algorithms from specifications—we
can not only establish the \emph{existence} or \emph{nonexistence}
of dynamic programming solutions but also derive them \emph{constructively}
whenever they exist. Consequently, the five generic problem definitions
yield \emph{five optimal algorithms} for ODT problems with \emph{arbitrary
	splitting rules} satisfies the axioms and \emph{objective functions}
of a given form. These algorithms encompass the known ODT algorithms
as special case, while providing a unified, efficient, and elegant
solution for the general ODT problem.

\end{abstract}


\section{Introduction}

Decision trees are widely used models in various areas of computer
science, including \emph{binary space partitioning trees} \citep{toth2005binary,fan2018binary,de2000computational}
and the \emph{$K$-$D$-tree} \citep{bentley1975multidimensional}
in computational geometry, as well as decision tree models in machine
learning (ML) \citep{breiman1984classification}.

Due to their simplicity, decision trees are particularly attractive
for providing highly interpretable models of the underlying data—an
aspect that has gained increasing importance in recent years. \citet{breiman2001statistical}
aptly noted, ``On interpretability, trees rate an A+.'' The most common
approach to optimizing decision trees is to employ top-down greedy
algorithms, such as the Classification and Regression Tree (CART)
\citet{breiman1984classification} and C4.5 \citep{quinlan2014c4}
algorithms, to optimize the tree.

However, with the rapid development of ML research, the greedy top-down
optimization algorithms are no longer as effective compared with state-of-the-art
ML algorithms. To improve the accuracy of decision tree models, two
common approaches are typically adopted. The first is to find the
globally optimal decision tree (ODT) for the given data. Alternatively,
instead of restricting models to axis-parallel hyperplanes, more complex
splitting rules can be applied. For example, a generalization of the
classical decision tree problem is the hyperplane (or oblique) decision
tree, which employs hyperplane decision boundaries to potentially
simplify complex boundary structures.

However, the decision tree optimization is notoriously difficult.
\citet{laurent1976constructing} showed that finding the \emph{smallest
	decision tree} (defined by the path lengths) is NP-complete. Consequently,
finding an optimal axis-parallel decision tree for a given dataset
is NP-hard. Similarly, for hyperplane decision trees, even the greedy
top-down optimization approach is NP-hard, as the number of possible
hyperplane splits for a dataset of size $N$ in dimension $D$ is
$O\left(N^{D}\right)$, which grows exponentially with $D$.

However, instead of addressing the general decision tree problem that
seeks the \emph{smallest }tree, successful algorithms such as CART
and C4.5 were not designed with this objective. Rather, their stopping
criteria explicitly constrain either the tree size (i.e., the number
of branch nodes) or the tree depth. Interestingly, \citet{ordyniak2021parameterized}
demonstrated that the decision tree problem with fixed size or depth
is \emph{fixed-parameter tractable}. In other words, when either parameter
is fixed, the problem can be solved in \emph{polynomial time}.

In this first paper of a two-part series, we present a framework for
constructing efficient algorithms—polynomial in the worst case and
trivially parallelizable—for obtaining globally optimal solutions
to decision tree optimization problems. Our exposition offers solutions
to nearly\textbf{ all }ODT problems encountered in practice, including
classical \emph{axis-parallel hyperplane}, \emph{general hyperplane},
\emph{polynomial hypersurface}, and decision trees over binary feature
data. In Part II \citep{he2025FoODT_II}, we will detail the construction of the \emph{first
	optimal hypersurface decision tree} algorithms, emphasizing their
practical implementation and applicability.

As our goal is to provide a comprehensive algorithmic framework for
all decision tree problems relevant in practice, this paper is written
in a pedagogical style. Nonetheless, it includes several novel contributions.
A summary of our main contributions in Part I is as follows:
\begin{itemize}
	\item \textbf{Critique and proposed solutions}: In Section \ref{subsec: 2.1},
	We critically examine current approaches to designing algorithms and
	defining problems in the study of optimal decision trees, highlighting
	the consequences of relying on informal problem formulations paired
	with ad-hoc algorithmic solutions—a practice common in combinatorial
	studies of ODT problems, particularly those employing BnB methods.
	Beyond this critique, the central distinction of our research lies
	in philosophy: we argue that deriving truly optimal algorithms demands
	the highest levels of rigor and clarity. Accordingly, we propose \emph{criteria}
	for formally defining problems in the study of optimal algorithms,
	along with the \emph{key characteristics} that a good optimal algorithm
	should possess. These discussions provide a foundational step for
	advancing research on ODT problems using combinatorial methods.
	\item \textbf{Formal combinatorial definitions}: Following the \emph{proper
		decision tree framework} examined in \citet{he2025odt}, which defines
	decision trees through recursive datatypes and establishes four axioms,
	In Subsection \ref{subsec:Four-distinct-definitions of ODT}, we established
	\textbf{four}\emph{ formal combinatorial definitions} for the ODT
	problem—covering both \emph{size-} and \emph{depth-constrained} trees.
	Importantly, we rigorously defined the \emph{search space} and the
	\emph{objective function} of the ODT problems as directly executable
	programs. This is vitally important because it allows researchers
	to inspect the search space by running the program and verifying the
	resulting decision trees. Such transparency not only validates our
	solution but also fosters unambiguous communication. To our knowledge,
	no prior work on BnB algorithms for ODT problems has provided a problem
	definition with this level of rigor, combining \emph{axiomatic definitions
	}of decision trees with \emph{programmatic definitions} of the search
	space and objective function.
	\item \textbf{Generic optimal decision tree algorithms}: In Subsection \ref{subsec:Four solutions to ODT},
	We present \textbf{four}\emph{ generic algorithms} for solving ODT
	problems with \emph{arbitrary splitting rules} over \emph{proper decision
		trees}, capable of optimizing \emph{arbitrary objective functions}
	of a specified form. One algorithm was previously introduced in our
	earlier work \citet{he2025odt}, while the remaining three are new
	contributions in this exposition. All algorithms are derived constructively
	and presented in a formal manner. Specifically, we provide three algorithms
	for \emph{size-constrained} trees and one for \emph{depth-constrained}
	trees. Notably, our depth-constrained ODT algorithm not only offers
	a formal proof of the results in \citet{brita2025optimal} for the
	axis-parallel ODT problem but also generalizes their findings to broader
	settings.
	\item \textbf{Extension to decision tree over binary feature data}: \citet{he2025odt}
	showed that the ODT-BF problem is a non-proper decision tree problem,
	requiring a distinct formulation of its search space. Based on this
	formulation, we derived a provably optimal algorithm for solving it.
	To the best of our knowledge, this is the first optimal algorithm
	that guarantees exhaustive exploration of the ODT-BF search space,
	thereby establishing a solid foundation for addressing this problem.
	\item \textbf{Filter and thinning fusion}: In Subsection \ref{subsec: thinning and filtering},
	We propose two fusion theorems—\\
	filtering and thinning—to enhance the
	flexibility and practical applicability of ODT algorithms while maintaining
	formal rigor. The filtering fusion theorem allows additional constraints,
	such as \emph{depth }and\emph{ leaf-size constraints}, to be incorporated
	into ODT algorithms defined over size-constrained trees. Similarly,
	\emph{tree-size constraints} can be applied to ODT algorithms defined
	over depth-constrained trees. The thinning fusion theorem integrates
	well-known\emph{ bounding technique}s into the framework. Together,
	these theorems extend the practical utility of the algorithms without
	compromising optimality.
	\item \textbf{Proving the properness of hypersurface decision trees}: In
	Section \ref{sec:Geometric-foundation}, by generalizing \citet{he2023efficient}'s
	0-1 loss linear/hypersurface classification theorem from linear models
	to decision tree models, and combining it with the geometric characterization
	of polynomial hypersurfaces as combinations of data points, we demonstrate
	that hypersurface decision trees are proper decision trees and can
	therefore be solved exactly using our proposed algorithms.
	\item \textbf{Splitting rule generators}: In Subsection \ref{subsec:Splitting-rule-generator},
	leveraging the sophisticated combination generator introduced by \citet{he2025CGs},
	we develop efficient algorithms for generating both \emph{hypersurface
		splitting rules} and\emph{ mixed-splitting rules} for a given dataset.
	The \emph{hypersurface splitting rule generator} provides the foundation
	for constructing efficient hypersurface decision tree algorithms developed
	in Part II, while the \emph{mixed-splitting rule generator} establish
	a foundation for future work on constructing optimal ODT algorithms
	with mixed splitting rules.
	\item \textbf{Complexity analysis for optimal decision tree problems}: In
	Subsection \ref{subsec:Complexity-of ODT problems}, Most studies
	on ODT algorithms either do not analyze worst-case complexity or analyze
	it using incorrect combinatorial characterizations. Our analysis provides
	an upper bound for size-constrained decision trees, which also implies
	trivial lower bounds for depth-constrained axis-parallel decision
	tree problem. Our lower bound analysis reveals a discrepancy with
	the function used to count the combinatorial complexity in the source
	code of \citet{mazumder2022quant}. This observation suggests a possible
	limitation in the exhaustiveness of their algorithm, which could warrant
	further investigation to refine the approach.
\end{itemize}
A summary of Part I of this paper is as follows. In Section \ref{sec:How-to-define},
we examine the informality of problem definitions in existing studies
of optimal decision tree algorithms. Subsection \ref{subsec: 2.1}
analyzes the consequences of adopting such informal definitions. In
Subsection \ref{subsec:What-should-we}, we discuss potential remedies
and outline the desired characteristics of a sound optimal algorithm.
Section \ref{sec:Notations-and-auxiliary} introduces the notations
and auxiliary functions used throughout the paper.

Section \ref{sec:Theory of decision tree} presents the main \emph{algorithmic
	foundations}. We begin by defining a proper decision tree, which allows
decision trees to be formalized using datatypes and axioms. Based
on these axioms, Subsection \ref{subsec:Four-distinct-definitions of ODT}
constructs four formal definitions of both size-constrained and depth-constrained
ODT problems, each yielding a generic solution algorithm. The construction
of these algorithms is detailed in Subsection \ref{subsec:Four solutions to ODT}.
In Subsection \ref{subsec: thinning and filtering}, we extend the
framework by introducing filtering and thinning algorithms, enabling
constraints and bounding techniques to be applied via simple conditions
derived from our theorems. Subsection \ref{subsec:Comparison-of-four algs}
provides a thorough comparison of the four algorithms.

In Section \ref{sec:Geometric-foundation}, we extend the results
of \citet{he2023efficient} to establish the \emph{geometric foundation}
for hypersurface decision tree problems. We show that, when the data
are in general position, decision trees with hypersurfaces defined
by data points are proper. Building on the combination and nested
combination generators of \citet{he2025CGs}, Subsection \ref{subsec:Splitting-rule-generator}
introduces two splitting rule generators: one for \emph{hypersurface
	splits} and another for \emph{mixed splits}, allowing a single tree
to incorporate axis-parallel, hyperplane, and hypersurface divisions.

Finally, Section \ref{sec:Conclusion} presents the conclusion.

\section{How to define an optimal decision tree problem formally and why it
	matters\label{sec:How-to-define}}

A useful analogy for understanding formal problem definitions comes
from the study of general-purpose solvers, where strict rules ensure
that problems are expressed in a precise form, often referred to as
the ``standard form.'' For instance, the standard form of a linear
programming problem requires the objective function to be minimized,
all constraints to be equalities, and all variables to be non-negative.
These strict requirements ensure unambiguous communication among researchers
and reproducibility of results.

In contrast, there appears to be no widely accepted standard for formally
defining optimal algorithms that use combinatorial methods. To appreciate
the importance of formal definitions, it is first necessary to clarify
what constitutes an informal definition and why such informality can
lead to significant consequences. We then discuss the key criteria
that must be satisfied to resolve ambiguity, as well as the \emph{characteristics}
that a \emph{good} optimal algorithm should possess.

\subsection{Informality and consequences\label{subsec: 2.1}}

Data back to the dawn of 20th century, after \citet{godel1931formal}
announced his well-known incompleteness theorem, there is a problem
determining a formal definition of ``effectively calculable.'' To
address this, three foundational models of computation were proposed:
lambda calculus, invented by \citep{church1936unsolvable}, general
recursive functions, proposed by Gödel \citep{gandy1988confluence},
and Turing machine, by \citet{turing1936computable}. Church give
his definition of $\lambda$-definability and \emph{claim} all calculation
that is $\lambda$-definable are effective calculable. However, his
argument was unsatisfactory to Gödel. As \citet{wadler2015propositions}
noted:`` Church merely presented the definition of $\lambda$-definability
and baldly claimed it corresponded to effective calculability.''

In contrast, Turing conducted a detailed analysis of the capabilities
of a ``computer''—at the time, the term referred to a human performing
a computation assisted by paper and pencil. \citet{gandy1988confluence}
pointed out that Turing's argument amounts to a theorem asserting
any computation a human with paper and pencil can perform can also
be performed by a Turing machine. It was finally Turing's argument
convinced Gödel, and eventually prove the three definitions had been
proved equivalent.

Returning to the study of the optimal decision tree problem, a similar
issue arises in the design of optimal algorithms using combinatorial
methods. In many studies, \emph{algorithms are directly defined and
	claimed to be optimal for a given problem, yet a rigorous problem
	definition is never formally established}. While precise problem definitions
may seem negligible for most approximate or heuristic algorithms in
machine learning, their absence can have serious consequences in the
design of true optimal algorithms. Informality in problem definition
not only \emph{poses risks for high-stakes applications} that require
strict optimality but can also \emph{mislead subsequent research}.
This is because the design of globally optimal algorithms has a \textbf{singular
	goal}: to identify a configuration (or multiple configurations with
the same optimal objective) over the entire search space, for which
there exists a \emph{unique} optimal value.

Unfortunately, a common trend in the studies of optimal algorithms
using combinatorial methods are to focus on proving bounding techniques
without establishing the correctness of the main algorithm. However,
establishing the optimality of an algorithm requires more than just
validating a bounding technique. In essence, much of these works amounts
to the design of ``optimal bounding techniques'' rather than the development
of truly ``optimal algorithms'' for the decision tree problem. These
are fundamentally different research directions, as a \emph{bounding
	technique is meaningful only when it supports a provably correct algorithm
	in the first place}. We believe this inclination stems from a programming
philosophy that prioritizes efficiency as the primary, and sometimes
sole, objective. This tendency is further compounded by the lack of
widely accepted criteria for defining an optimal algorithm.

\subsubsection*{What is an ``informal'' definition, and what are the consequences
	of adopting such a definition in optimal algorithm design?}

For instance, studies of ODT algorithms using combinatorial methods
often leave the problem undefined\footnote{We claim that these studies leave the problem undefined because they
	commonly do not specify the solution space over which they are searching
	for an optimal decision tree. In some cases, the space is not even
	ambiguously defined, as in specification (\ref{eq: specification in literature}),
	making it unclear which ODT problem they are actually attempting to
	solve.} \citep{lin2020generalized,zhang2023optimal,van2023necessary,demirovic2022murtree,hu2019optimal}
or specify their problem ambiguously, as in \citep{mazumder2022quant,brita2025optimal},
where the problem is expressed in a form similar to
\begin{equation}
	s^{*}=\text{argmin}_{s\in\mathcal{S}}\sum_{n\in\mathcal{N}}E\left(s\left(x_{n}\right)\neq t\right)\label{eq: specification in literature}
\end{equation}
where $\mathcal{S}$ is often defined as the search space of all possible
decision trees with a constrained depth or size. $E$ is the objective
function, which is often defined as the \emph{0-1 loss} $E_{\text{0-1}}=\mathbf{1}\left(s\left(x_{n}\right)\neq t\right)$
(count the number of misclassification) (counting the number of misclassifications)
for classification problems, and the\emph{ squared loss} $E_{\text{0-1}}=\left(s\left(x_{n}\right)\neq t\right)^{2}$
for regression problems.

Although (\ref{eq: specification in literature}) appears intuitively
obvious, it suffers from several major flaws. First, the search space
$\mathcal{S}$ usually denoted by a symbol and defined through \emph{text
	description}. For example, \citet{brita2025optimal} wrote:`` $\mathcal{T}\left(\mathcal{D},d\right)$
describe the set of all decision trees for the datasets $\mathcal{D}$
with a maximum depth of $d$ .'' Adopting such a textual description
for $\mathcal{T}\left(\mathcal{D},d\right)$ is inherently ambiguous
and dangerous. Even though it appears simple, different researchers
may interpret the search space differently due to variations in the
constraints imposed on the trees, and some may even adopt entirely
different definitions of the decision tree. In contrast, MIP specifications
force researchers to define precisely what they mean by a decision
tree by specifying the constraints first, enabling an unambiguous
definition.

As a result, studies that adopt such informal problem definitions
cannot verify the most important step in proving the optimality of
an algorithm—showing that the algorithm provably searches \textbf{all
	possible solutions} in $\mathcal{S}$ ($\mathcal{T}\left(\mathcal{D},d\right)$
in \citet{brita2025optimal}), because $\mathcal{S}$ is not formally
defined. Consequently, research on BnB algorithms often focuses solely
on proving the correctness of the bounding techniques and then claims
the algorithm is ``optimal.''

Additionally, the ambiguous problem definition inevitably leads to
inconsistencies in the study of optimal algorithms. For example, although
\citet{mazumder2022quant} and \citet{brita2025optimal} claim to
solve the same problem, we found that their algorithms return different
trees with significantly different objective values on the same datasets—differences
too large to be attributed to numerical issues. This suggest that
either \citet{mazumder2022quant}'s algorithm is non-optimal or that
the two studies are solving different ODT problems. Evidence favors
the first explanation: a function in \citet{mazumder2022quant}'s
source code counts the number of possible trees in the search space,
which contradicts the lower bound we provide in Subsection \ref{subsec:Complexity-of ODT problems}.

Secondly, for most objectives in machine learning, comparisons between
two configurations are defined by a \emph{preorder}: for any pair
of configurations $a$ and $b$, $a\leq b$ and $b\leq a$ does not
implies $a=b$. In other words, two configurations with the same objective
value are not necessarily identical. As a result, the $\text{argmin}$
over a preorder does not always yield a unique solution. How we define
$\text{argmin}$ thus becomes critical, especially in ML applications
where we care about the predictive outcomes of the solution: different
configurations with the same objective can lead to completely different
predictions.

Furthermore, interpretations of what constitutes an ``optimal algorithm''
vary. Some restrict optimality to a specific objective, while others
do not. For instance, in classification tasks, most researchers consider
the solution minimizing the 0-1 loss as the standard for optimality,
because optimizing convex surrogates such as hinge or squared loss
is relatively straightforward. This explains why the support vector
machine, which is optimal for hinge loss, is not typically claimed
as ``optimal algorithm.'' In contrast, some studies on empirical risk
minimization problem for two-layer neural networks \citep{arora2016understanding,pilanci2020neural}
optimize \emph{convex losses} but claim the optimality. This highlights
the importance of establishing a common and unambiguous definition
of the objective function to avoid misunderstandings about the problem
setting.

Finally, if $\text{argmin}$ is left undefined, then $E$ is often
undefined as well. Although the definitions (or implementations) of
$E$ and $\text{argmin}$ may seem trivial, they are in fact crucial
for proving the correctness of an algorithm. As we will see in our
exploration, the existence of a dynamic programming or greedy algorithm
is almost entirely determined by the structure of\textbf{ $E$ }and
$\mathcal{S}$; in other words, it is governed by the \textbf{structural
	compatibility} between $E$ and $\mathcal{S}$.

\subsection{What should we do next?\label{subsec:What-should-we}}
\begin{quote}
	The difficulty I foresaw were a consequence of the absence of generally
	accepted criteria, and although I was convinced of the validity of
	the criteria I had chosen to apply, I feared that my review would
	I be refused or discarded as ``a matter of personal taste.'' I still
	think that such reviews would be extremely useful and i am longing
	to see them appear, for their accepted appearance would be a sure
	sign of maturity of the computing community. — \citet{dijkstra1972humble}
\end{quote}

\paragraph{Criteria of formal specification}

In \citet{dijkstra1972humble}'s 1972 talk, `` The Humble Programmer,''
Dijkstra criticized the programming philosophy that prioritizes efficiency
optimization above all else, arguing that it leads to unstructured
designs that are hard to program. Similarly, in our study of optimal
algorithms for machine learning, an exclusive focus on efficiency
when solving optimal algorithms has led to many oversights. These
issues, I argue, stem from the facts that algorithm designers also
prioritizes the theoretical computational efficiency and absence of
widely accepted criteria in the study of optimal algorithms, concerns
echoed in Dijkstra's discussion.

The lack of standardized criteria for validating optimality in studies
of combinatorial methods for optimal algorithms has consequences reminiscent
of Dijkstra's observations. I foresaw that my critique of this issue
might be dismissed as subjective, much like Dijkstra's fear that his
reviews would be rejected as ``a matter of personal taste.'' Despite
this, I believe such critical reviews are essential and would signal
the maturity of the optimization community if widely accepted.

Here, I summarize several key criteria that are essential for the
design of optimal algorithms but have often been overlooked. I believe
that omitting discussion of these points can lead to severe consequences
for the reproducibility and verifiability of optimal algorithm design.
To bolster the credibility of these arguments, I draw parallels with
insights from \citet{dijkstra1972humble} insights. These considerations
can be broadly categorized as follows:
\begin{enumerate}
	\item \textbf{Exact algorithms must be provable correct}: This is the central
	philosophy we aim to convey: an exact algorithm must be provably correct,
	which demands the highest level of rigor in establishing its correctness.
	As Dijkstra emphasized:`` We must not forget that it is not our business
	to make programs; it is our business to design classes of computations
	that will display a desired behavior.'' When designing optimal algorithms
	using combinatorial methods, it is essential to rigorously prove that
	the algorithm exhibits the intended behavior. Specifically, an optimal
	algorithm must be proved to either (i) exhaustively explore all valid
	solutions to guarantee identification of the optimal one with respect
	to the given objective, or (ii) soundly eliminate non-optimal solutions.
	Without such proof, any claim of optimality remains unfounded.
	\item \textbf{Proof requires an unambiguous problem specification}: Logically,
	a specification (definition) \textbf{cannot} be ``proved''; it must
	be \textbf{defined} or \textbf{introduced} as clearly as possible
	to allow for verification and rigorous reasoning. In computer science,
	one of the clearest ways to specify a problem is through an \emph{actual
		program} rather than \emph{pseudocode}, which can obscure crucial
	details and hinder reproducibility. For combinatorial problems, it
	is essential to answer: \emph{What exactly is the search space?} and
	\emph{What objective function is being optimized?} Ideally, the search
	space and objective function should either be defined by a program
	that are \emph{immediately executable} or expressed in a standard
	(MIP) form for studies using general-purpose solvers. Without such
	clarity, inconsistencies in problem definitions and their corresponding
	solutions are inevitable.
	\item \textbf{Algorithms should be intellectually manageable}: For instance,
	\citet{demirovic2022murtree}'s algorithm spans nearly five pages
	of pseudocode, making it intellectually challenging to understand,
	implement, and verify. Although a common strategy in algorithm design
	is to first construct a program and then verify that it satisfies
	the required properties, this process is often cumbersome and error-prone.
	As \citet{dijkstra1972humble} famously quote:`` Program testing can
	be a very effective way to show the presence of bugs, but it is hopelessly
	inadequate for showing their absence. The only effective way to raise
	the confidence level of a program significantly is to give a convincing
	proof of its correctness. But one should not first make the program
	and then prove its correctness, because then the requirement of providing
	the proof would only increase the poor programmer's burden.'' Similarly,
	research on combinatorial methods—particularly BnB algorithms—often
	adopts the practice of constructing ad-hoc algorithms and proving
	their correctness post-hoc using complex proofs (or simply without
	proof). However, this approach is problematic. Complex algorithms
	typically result in complex proofs, which not only make the algorithms
	less accessible but also harder to verify To address this, Dijkstra—and
	likewise, this work—advocates a \emph{constructive} approach to programming:
	derive programs from well-defined specifications, thereby ensuring
	correctness \emph{by design}. In combinatorial optimization, we begin
	with an unambiguous \emph{brute-force specification} of the optimization
	problem and systematically derive a more efficient algorithm. This
	guarantees that the final implementation inherits the properties of
	the original specification.
	\item \textbf{The importance of complexity analysis}: This issue arises
	from earlier challenges, such as the lack of clear problem definitions
	and bounding techniques, which make formal complexity analysis infeasible.
	The situation is further exacerbated by the fact that many ad-hoc
	algorithms designed in BnB research are not intellectually manageable.
	For instance, many \emph{operational components}—such as the use of
	depth-first search strategies—are implicitly embedded in the pseudocode
	of the designed algorithms. While these additional ingredients may
	provide advantages, such as reducing runtime memory usage, they also
	introduce side effects: the resulting algorithms become harder to
	analyze formally. Consequently, most studies on BnB algorithms do
	not provide time or space complexity analyses. However, complexity
	analysis remains critically important for algorithm design. Without
	it, the scalability and practical applicability of the algorithm remain
	unpredictable.
	\item \textbf{Acceleration techniques must be demonstrably effective}: This
	may seem counterintuitive: if a technique provably reduces the search
	space, how could it be ineffective? In practice, the overhead of applying
	the technique can outweigh its benefits. The key is to balance the
	cost of applying an acceleration method against the potential reduction
	in search space. For example, \citet{brita2025optimal} claim that
	caching is a major factor behind their algorithm's efficiency. However,
	\citet{he2025odt} provide a simple analysis showing that overlapping
	subproblems in decision tree construction are rare—particularly for
	the ODT problem with axis-parallel or more complex splitting rules—making
	caching or memoization largely ineffective. In our experiments, we
	even observed cases where \citet{brita2025optimal}'s algorithm explored
	over a million decision trees without a single cache hit, suggesting
	that in such scenarios, disabling caching entirely may improve performance.
	\item \textbf{Designing algorithms and designing bounding techniques are
		distinct tasks}: As discussed, there is widespread confusion in the
	community regarding the distinction between designing optimal algorithms
	and developing bounding techniques for them. Bounding techniques are
	only meaningful when applied to algorithms that are already optimal.
	Therefore, research on these two aspects should be analyzed and presented
	separately. Studies on bounding techniques should focus on assessing
	their effectiveness by providing theoretical guarantees or empirical
	evidence. In contrast, the design of optimal algorithms must first
	establish their correctness and optimality, followed by a rigorous
	analysis of their time and space complexity as well as their empirical
	efficiency—entirely independent of any bounding techniques.
\end{enumerate}

\paragraph{Characteristics of an good optimal algorithm}

Having discussed the critical criteria for designing optimal algorithms,
a natural question arises: if several algorithms satisfy these criteria,
which one is the best, and what characteristics should a good optimal
algorithm possess? Indeed, many algorithms have been proposed for
the optimal decision tree problem, but identifying the best approach
is often challenging, as comparisons are frequently conducted unfairly
or omit essential considerations.

By contrast, in the study of sorting algorithms, several well-established
characteristics allow researchers to assess an algorithm's quality
systematically.A good sorting algorithm typically aims to satisfy
four key properties, though achieving all simultaneously is nontrivial:
\begin{enumerate}
	\item \textbf{Fast}: The algorithm should be asymptotically optimal in the
	number of comparisons, and the constants involved in other operations
	should also be small.
	\item \textbf{Smooth}: Performance should improve on partially sorted inputs;
	the more sorted the input, the faster the algorithm performs.
	\item \textbf{Stable}: Records with equal keys should retain their relative
	order in the output.
	\item \textbf{Compact}: The algorithm should be economical in both space
	and runtime.
\end{enumerate}
For example, mergesort algorithm is fast, stable, and compact, but
not smooth, as it maintains $O\left(N\log\left(N\right)\right)$ complexity
even for already sorted inputs.

By analogy, we believe a good optimal algorithm also should satisfy
the following properties:\emph{ provably optimal}, \emph{conceptually
	simple}, \emph{computationally efficient}, and \emph{inherently parallelizable}.
In the context of ML research, such as the ODT problem, it is also
important to consider an additional property: the \emph{generality
	of the solution}. Optimizing a single objective or specific splitting
rules does not guarantee robust performance across diverse scenarios.
An algorithm capable of handling \emph{flexible objective functions}
or constructing \emph{decision trees with complex splitting rules}
is typically more practical and broadly applicable. While we have
already discussed the importance of simplicity and provable optimality,
we now clarify what we mean by efficiency, generality, and parallelizability.
\begin{enumerate}
	\item \textbf{Efficiency}: The efficiency of an optimal algorithm should
	not be judged solely by its theoretical worst-case complexity. Factors
	such as hidden constants in the big-O notation, and more importantly,
	compatibility with hardware and software, significantly influence
	practical performance. Algorithms with predictable blocked array usage
	and structured computational processes often achieve higher efficiency
	in practice. Thus, we argue that designing hardware-friendly algorithms
	is as important as designing algorithms that achieve superior theoretical
	efficiency. In Part II, we demonstrate that an algorithm with superior
	theoretical performance may still run slower than a less theoretically
	efficient one if the latter is better aligned with hardware structure.
	\item \textbf{Abstraction and generality}. Abstraction not only simplifies
	complex problems but also enhances the generality of solutions. As
	\citet{dijkstra1972humble} noted, ``the effective exploitation of
	his powers of abstraction must be regarded as one of the most vital
	activities of a competent programmer. the purpose of abstracting is
	not to be vague, but to create a new semantic level in which one can
	be absolutely precise.'' In algorithm design, abstraction provides
	substantial advantages. In part I, we will demonstrate that adopting
	a generic principle can lead to a deeper understanding of problem
	structure, ultimately enabling a generic algorithm capable of solving
	the ODT problem with arbitrary splitting rules and a broad class of
	objective functions.
	\item \textbf{Parallelism}: For problems with intractable combinatorics—such
	as NP-hard optimization problems—parallelism is perhaps the only practical
	way to scale the algorithm to handle large scale datasetss. We do
	\textbf{not} merely advocate designing complex algorithms and parallelizing
	them post-hoc; rather, we argue that algorithms should be inherently
	parallelizable \emph{from the beginning of the design process}. In
	other words, parallelizability should be treated as a fundamental
	property during algorithm design. Designing inherently parallel algorithms
	while maintaining optimality is challenging and requires a deep understanding
	of both the problem and algorithmic structure. The success of neural
	networks illustrates this well: GPU hardware significantly enhanced
	their effectiveness by aligning with their computational structure.
	Similarly, we argue that designing hardware-friendly algorithms is
	as important as achieving theoretical efficiency. For example, in
	Part I, we show that solving the ODT problem with intractable splitting
	rules depends heavily on the algorithm's ability to exploit parallelism.
	Furthermore, in Part II, we demonstrate that an algorithm with superior
	theoretical performance may still run slower than a less theoretically
	efficient one if the latter is better aligned with parallel hardware.
\end{enumerate}
Rather than merely offering criticism or vacuous suggestions, we aim
to demonstrate the effectiveness of the criteria we proposed. In the
following paper, we formally address the ODT problem, thereby ensuring
that these criteria are satisfied and that the resulting algorithms
are intellectually manageable. We show that approaching the problem
in such an intellectually manageable manner not only simplifies the
derivation and makes it more accessible, but also enables the construction
of more generic and efficient solutions that exhibit the desired properties
of a ``good optimal algorithm'' as we proposed—even though not all
properties may be fully achieved simultaneously.

\section{Notations and auxiliary functions\label{sec:Notations-and-auxiliary}}

\subsection{Notations}

In this paper, the symbol $:$ is used in two distinct contexts. When
appearing in code expressions, $:$ denotes the list-construction
(cons) operator, which prepends an element to an existing list. For
example, $a:\left[b,c\right]=\left[a,b,c\right]$. When appearing
in type signatures, the same symbol $:$ specifies the type of an
expression; for example, $x:\mathbb{N}$ means ``$x$ is a natural
number.''

The types of real and natural numbers are denoted by $\mathbb{R}$
and $\mathbb{N}$, respectively. We use square brackets $\left[\mathcal{A}\right]$
to denote the set of all finite lists of elements $a:\mathcal{A}$,
where $\mathcal{A}$ (or letters $B$ and $\mathcal{C}$ at the front
of the alphabet) represent type variables. A data point is denoted
as a $D$-tuple $x=\left(x_{1},\ldots,x_{D}\right):\mathbb{R}^{D}$.
Hence, $\left[\mathcal{R}\right]$ and $\left[\mathbb{R}^{D}\right]$
(We use $\mathcal{D}$ as a short-hand synonym for $\left[\mathbb{R}^{D}\right]$),
$\left[\mathcal{H}^{M}\right]$, denote the set of all finite lists
of splitting rules, lists of data in $\mathbb{R}^{D}$, and hypersurfaces
defined by degree $M$ polynomials, respectively. Two special type
synonyms, $C=\left[\mathbb{N}\right]$ and $NC=\left[\left[\mathbb{N}\right]\right]$
are used to denote combinations of natural numbers and nested combinations
(i.e., combinations of combinations).

Variables of these types are denoted using their corresponding lowercase
letters; for example, $r:\mathcal{R}$, $h:\mathcal{H}$, $h^{M}:\mathcal{H}^{M}$.
In particular, a degree-one polynomial hypersurface is defined as
$h=\boldsymbol{w}^{T}x+c$ , where $\boldsymbol{w},x:\mathbb{R}^{D}$
and $c\in\mathbb{R}$. Such a hypersurface is known as a hyperplane;
we denote $\mathcal{H}^{1}=\mathcal{H}$ by default. Similarly, $\mathcal{H}^{0}$
denotes the set of all axis-parallel hyperplanes, where an axis-parallel
hyperplane corresponds to fixing one coordinate xi to a constant value,
defined as $h=\boldsymbol{e}_{i}^{T}\boldsymbol{x}+c$, with $e_{i}$
being a unit vector along the $i$-th coordinate axis.

We denote a list of elements of the same type as $\mathit{as}:\left[\mathcal{A}\right]$;
for instance, $rs:\left[\mathcal{R}\right]$, $xs:\mathcal{D}$, $hs:\left[\mathcal{H}\right]$,
and $\mathit{hs}^{M}:\left[\mathcal{H}^{M}\right]$. The size of a
list as is denoted $\left|as\right|$ , and $as!d$ indexes the $d$-th
element of the list. Similarly, $\mathit{ass}$ denotes a list of
lists of elements of the same type. The operations $+\!\!+$ list
concatenation; for example, $a:\left[b,c\right]=\left[a\right]+\!\!+\left[b,c\right]=\left[a,b,c\right]$.

Both the affine $D$-space, for $D\geq1\in\mathbb{N}$, and the vector
$D$-space over $\mathbb{R}$ are denoted by $\mathbb{R}^{D}$. Both
are the set of all $D$-tuples of elements of $\mathbb{R}$. To distinguish
a vector in the vector space and a point in affine space, an element
$x=\left(x_{1},x_{2}\ldots,x_{D}\right)\in\mathbb{R}^{D}$ will be
called a \emph{point} in the affine space $\mathbb{R}^{D}$, an element
$x=\left(x_{1},x_{2}\ldots,x_{D}\right)^{T}\in\mathbb{R}^{D}$ is
called a \emph{vector} in the vector space $\mathbb{R}^{D}$, where
$x_{i}$ is called the \emph{coordinate} of $x$. Note that, except
for the data point vector $x$, we denote all other vectors, such
as the normal vector $\boldsymbol{w}$, and the polynomial coefficient
vector $\boldsymbol{\alpha}$, using bold symbols. The data point
$x$ is written in plain format, rather than bold, to maintain consistency
with the data list $\mathit{xs}$. Moreover, when referring to a data
point vector, we omit subscripts to avoid confusion with coordinate
indices $x_{i}$.

\subsection{Function deifinitions and point-free style functional composition}

One of the most common ways of defining a list or set is through list/set
comprehension. List/set comprehension is a mathematical notation for
defining a set by specifying a distinct property that all its members
must satisfy. Set comprehension takes the form
\[
\left\{ f\left(x\right)\mid x\in\mathcal{X},p\left(x\right)=\text{True}\right\} .
\]
This is read as:`` The set of all outputs $f\left(x\right)$ such
that element $x$ is a member of the source set $\mathcal{X}$ and
the predicate $p\left(x\right)=\text{True}$ holds.'' Similarly, list
comprehension takes the form

\[
\left[f\left(x\right)\mid x\leftarrow\mathit{xs},p\left(x\right)=\text{True}\right].
\]
This is read as:`` The list of all output $f\left(x\right)$ such
that element $x$ is drawn from source list $\mathit{xs}$ \emph{sequentially}
and the predicate $p\left(x\right)=\text{True}$ holds.''

Throughout the discussion, we will use some common functions from
functional programming languages. One frequently used function is
the map function over \emph{list}
\[
\mathit{mapL}\left(f,\mathit{xs}\right)=\left[f\left(x\right)\mid x\leftarrow\mathit{xs}\right],
\]
which applies function $f:\mathcal{A}\to\mathcal{B}$ to each element
$x$ in the list $\mathit{xs}$.

Similarly, another common function used frequently in our exploration
is the concatenate (or flatten) function 
\[
\mathit{concat}=[x\mid x\leftarrow\mathit{xs},\mathit{xs}\leftarrow\mathit{xss}],
\]
which flattens a list of lists into a single list. These two functions
can also be defined recursively. For example, the map function can
be written as:
\begin{align*}
	\mathit{mapL} & :\left(\mathcal{A}\to\mathcal{B}\right)\times\left[\mathcal{A}\right]\to\left[\mathcal{B}\right]\\
	\mathit{mapL} & \left(f,\mathit{xs}\right)=\left[f\left(x\right)\mid x\leftarrow\mathit{xs}\right]
\end{align*}
and the concatenate function as
\begin{align*}
	\mathit{concat} & :\left[\left[\mathcal{A}\right]\right]\to\left[\mathcal{A}\right]\\
	\mathit{concat} & \left(\left[\:\right]\right)=\left[\:\right]\\
	\mathit{concat} & \left(\mathit{as}:\mathit{ass}\right)=\mathit{as}+\!\!+\mathit{mapL}\left(\mathit{as}\right)
\end{align*}
Another useful function is the map function over the tree datatype,
which can be defined as follows

\[
\begin{aligned}\mathit{mapD} & :\left(\mathcal{B}\to\mathcal{C}\right)\times\mathit{DTree}\left(\mathcal{A},\mathcal{B}\right)\to\mathit{DTree}\left(\mathcal{A},\mathcal{C}\right)\\
	\mathit{mapD} & \left(f,\mathit{DL}\left(a\right)\right)=\mathit{DL}\left(f\left(a\right)\right)\\
	\mathit{mapD} & \left(f,\mathit{DN}\left(u,r,v\right)\right)=DN\left(\mathit{mapD}\left(f,u\right),r,\mathit{mapD}\left(f,v\right)\right),
\end{aligned}
\]
where the ``decision tree datatype'' $\mathit{DTree}\left(\mathcal{A},\mathcal{B}\right)$
will be formally introduced when we define decision trees.

To maintain the elegance of function definitions, we adopt\emph{ point-free
	style} programming paradigm, a style in which functions are defined
without explicitly naming their arguments. Instead, functions are
expressed solely through the composition of other functions, chaining
transformations without introducing intermediate variables.

For example, the function $\mathit{concatMap}$ can be defined as

\[
\mathit{concatMap}_{f}=\mathit{concat}\circ\mathit{mapL}_{f},
\]
where $\circ$ denotes function composition. This means that $\mathit{mapL}_{f}$
is applied first, followed by $\mathit{\mathit{concat}}$. Here, $\mathit{mapL}_{f}$
is an example of a section—a technique in Haskell for partially applying
an operator to create a new function by fixing one of its arguments.
Specifically, $\mathit{mapL}_{f}$ is a partial application of the
$\mathit{mapL}$ function, parameterized by a function $f:\mathcal{A}\to\mathcal{B}$,
resulting in a new function of type $\mathit{mapL}_{f}:\left[\mathcal{A}\right]\to\left[\mathcal{B}\right]$.
In the definition of $\mathit{concatMap}_{f}$, the input list $\mathit{xs}$
is not explicitly mentioned; this is why the style is called point-free.
Conceptually, $\mathit{concatMap}_{f}$ applies $f$ to each element
of the input list and concatenates the resulting lists.

More generally, when a function $g:\mathcal{X}\times\mathcal{Y}\to\mathcal{Z}$
is partially applied to an argument $x:\mathcal{X}$, we obtain a
new function $g_{x}:\mathcal{Y}\to\mathcal{Z}$. However, if the subscript
of a function name is already used for descriptive purposes—for example,
$g_{\text{SomeName}}\left(x,y\right):\mathcal{X}\times\mathcal{Y}\to\mathcal{Z}$—we
denote its partially applied form as $g_{\text{SomeName}}\left(x\right):\mathcal{Y}\to\mathcal{Z}$
to avoid ambiguity.

We acknowledge that point-free style may be unfamiliar to some readers.
Whenever a function is expressed in point-free form, we will provide
its type whenever possible to clarify its behavior.

\section{Algorithmic foundation\label{sec:Theory of decision tree}}

\subsection{Definition of decision tree \label{subsec:comb def of tree}}

Before discussing our formal definition of decision trees and the
ODT problem, we first extend the specification in (\ref{eq: specification in literature})
by introducing additional constraints. We then progressively explain
how to define decision trees and decision tree problems rigorously.

The ODT problem seeks to identify a \textbf{binary tree} $T$ that
minimizes the number of misclassifications, where the internal nodes
of $T$ are characterized by a list of splitting rules $\mathit{rs}:\left[\mathcal{R}\right]$.
In our exposition, $\mathit{rs}$ is a \emph{finite} set generated
from a finite set of data points $\mathit{xs}$. This \emph{size-constrained}
ODT problem typically involves two hyperparameters. The trade-off
between accuracy and tree complexity is managed either by:
\begin{enumerate}
	\item Fixing the tree size $\left|T\right|$ (the number of branching nodes,
	with the number of leaves equal to $\left|T\right|+1$), or
	\item Adding a penalty term $\lambda\left|T\right|$ to the objective function.
\end{enumerate}
Both approaches are equivalent because a fixed $\lambda$ will corresponds
to a tree with a fixed size $\left|T\right|$. In this exposition,
we focus on the case where $\left|T\right|=K$. The second hyperparameter
controls the minimum number of data items required in each leaf, denoted
$N_{\text{min}}$. The size-constrained ODT problem can therefore
be formulated as follows
\begin{equation}
	\begin{aligned}T^{*}= & \underset{T\in\mathcal{S}_{\text{size}}\left(K,\mathit{rs}\right)}{\text{ argmin }}E\left(T\right)\\
		\text{s.t., } & \left|l\right|\geq N_{\text{min}},\forall l\in\text{leaves}\left(T\right)
	\end{aligned}
	\label{eq: MIP-ODT-size}
\end{equation}
where $E\left(T\right)$ is the objective function. In classification
problems, it typically counts the number of misclassifications (i.e.,
the 0-1 loss), which will be the primary focus of this paper. The
set $\mathcal{S}_{\text{size}}\left(K,\mathit{rs}\right)$ denotes
the combinatorial search space of all decision trees of size $K$
with respect to a list of splitting rules $rs$, and $\text{leave}\left(T\right)$
represents all leaf nodes of the tree $T$.

Alternatively, one can define the ODT problem with a depth constraint
$D$ as

\begin{equation}
	\begin{aligned}T^{*}= & \underset{T\in\mathcal{S}_{\text{depth}}\left(D,\mathit{xs}\right)}{\text{ argmin }}E\left(T\right)\\
		\text{s.t., } & \left|l\right|\geq N_{\text{min}},\forall l\in\text{leaves}\left(T\right)
	\end{aligned}
	\label{eq: MIP-ODT-depth}
\end{equation}
where $\mathcal{S}_{\text{depth}}\left(D,\mathit{xs}\right)$ denotes
the search space of decision trees of depth $D$ with respect to the
data list $\mathit{xs}$ .

It is important to recognize that these definitions do not explicitly
define $T$ beyond stating that it is a binary tree. Different authors
may have slightly different interpretations of a decision tree. Some
define a decision tree purely in terms of its splitting rules, where
predictions are determined solely by the rules, while others consider
the predictions themselves as an essential part of the tree's structure.

The primary difference between our study and previous work on combinatorial
methods is that all previously undefined terms in (\ref{eq: MIP-ODT-size})
and (\ref{eq: MIP-ODT-depth})—such as $\text{leave}$, $E$, $T$,
and $\mathcal{S}_{\text{size}}$ (or alternatively, $\mathcal{S}_{\text{depth}}\left(K,\mathit{xs}\right)$
for depth-constrained tree)—will be defined unambiguously using \textbf{programs}
or \textbf{datatypes with formal axioms}.

Some astute readers may wonder why we define the search space of depth-constrained
trees using the data list $\mathit{xs}$ and size-constrained trees
using $\mathit{rs}$. This is an operational choice aimed at maintaining
elegance in the algorithm's expression, as we will explain in Subsection
\ref{subsec:Four-distinct-definitions of ODT}. Another reason for
treating $\mathit{xs}$ and $\mathit{rs}$ separately is to avoid
arbitrary choices of splitting rules, which could lead to constructing
an algorithm that does not terminate. If the set of possible splitting
rules were infinite, it would be impossible to design an algorithm
that constructs an ODT in finite time.

Therefore, the first and most important step is to define unambiguously
\emph{what we mean by a decision tree}\textbf{ }$T$. This step is
often overlooked in prior literature using combinatorial methods.
Without a formal definition of a ``decision tree,'' it is impossible
to rigorously define any other terms related to it.

\subsubsection*{Proper decision tree}

As mentioned, a decision tree is a binary tree that can be defined
recursively using the following datatype

\[
\mathit{DTree}\left(\mathcal{A},\mathcal{B}\right)=DL\left(\mathcal{B}\right)\mid DN\left(\mathit{DTree}\left(\mathcal{A},\mathcal{B}\right),\mathcal{A},\mathit{DTree}\left(\mathcal{A},\mathcal{B}\right)\right).
\]
This definition states that a binary decision tree $\mathit{DTree}\left(\mathcal{A},\mathcal{B}\right)$
is either a leaf $DL\left(b\right)$ with a label of type $b:\mathcal{B}$,
or an internal node$DN\left(u,a,v\right)$, where $u,v:\mathit{DTree}\left(\mathcal{A},\mathcal{B}\right)$
are the left and right subtrees, respectively, and $a:\mathcal{A}$
is the root node. A \emph{path} is defined as the sequence of nodes
along the edges from the root to a leaf. For each leaf $L$ in the
tree, there exists exactly one path from the root to $L$, which we
denote by $P_{L}$.

However, not every binary tree qualifies as a valid decision tree.
To characterize decision trees precisely, the datatype must be accompanied
by a set of \textbf{axioms}. We have previously formalized this notion
by defining a proper decision tree, associating $\mathit{DTree}\left(\mathcal{A},\mathcal{B}\right)$with
four axioms introduced in \citet{he2025odt}, as described below.
\begin{definition}
	Axioms for proper decision trees.\emph{ We call a decision tree consists
	of splitting rules that satisfies the following axioms, a proper
	decision tree:\label{def:Axioms-for Proper-DT}}
\end{definition}
\begin{enumerate}
	\item \emph{Structural constraint one} (Ambient space partition): Each branch
	node is defined by a single splitting rule $r:\mathcal{R}$, and each
	splitting rule subdivides the ambient space into two \emph{disjoint}
	and connected subspaces, $r^{+}$ and $r^{-}$.
	\item \emph{Structural constraint two} (Path intersection in leafs): Each
	leaf $L$ is defined by the intersection of subspaces $\bigcap_{p\in P_{L}}r_{p}^{\pm}$
	for all the splitting rules $\left\{ r_{p}\mid p\in P_{L}\right\} $
	in the path $P_{L}$ from the root to leaf $L$. The connected region
	(subspace) defined by $\bigcap_{p\in P_{L}}r_{p}^{\pm}$ is referred
	to as the \emph{decision} \emph{region}.
	\item \emph{Structural constraint three} (Partition transitivity): Part
	I: The ancestry relation between any pair of splitting rules $r_{i}\left(\swarrow\vee\searrow\right)r_{j}$
	is transitive; in other words, if $r_{i}\left(\swarrow\vee\searrow\right)r_{j}$
	and $r_{j}\left(\swarrow\vee\searrow\right)r_{k}$ then $r_{i}\left(\swarrow\vee\searrow\right)r_{k}$.
	Part II: Moreover, $\boldsymbol{K}_{ij}=\pm1$ if $r_{j}$ can be
	generated from $r_{i}^{\pm}$. As a result, any new decision rule
	$r$ added to a leaf must be generated within its corresponding decision
	region.
	\item \emph{Ancestral constraint} (Uniqueness of the ancestry relation):
	For any pair of splitting rules $r_{i}$ and $r_{j}$, only one of
	the following three cases is true: $r_{i}\swarrow r_{j}$, $r_{i}\searrow r_{j}$,
	and $r_{i}\overline{\left(\swarrow\vee\searrow\right)}r_{j}$; additionally,
	$r_{i}\overline{\left(\swarrow\vee\searrow\right)}r_{i}$ is always
	true; in other words, the possible value of $\boldsymbol{K}_{ij}\in\left\{ 1,0,-1\right\} $
	is unique determined for all $i,j$ , and $\boldsymbol{K}_{ii}=0$
	for all $i$.
\end{enumerate}
A key observation in \citet{he2025odt} is that if a tree's splitting
rules satisfy the structural constraints (Axioms 1-3), they define
a class of decision tree problems. However, the ancestral constraint
(Axiom 4), which specifies pairwise relationships between splitting
rules, implicitly determines the subclass of decision tree problems
under consideration. The proper decision tree framework adopts Axiom
4 as the core ancestral constraint because it encompasses most decision
tree models in machine learning, including axis-parallel hyperplane,
general hyperplane, and polynomial hypersurface decision trees. For
decision trees over binary feature data, however, Subsection \ref{subsec:Extension-to-decision}
explains how this problem violates Axiom 4, necessitating a different
algorithmic approach. Fortunately, thanks to the deep algorithmic
insights provided by our framework, this adaptation is achieved with
minimal changes.

Consider a tree with splitting rules $r$ defined as hyperplanes $h$,
referred as the \emph{decision hyperplanes} $h_{1}$, $h_{2}$, and
$h_{3}$ defined by normal vector $\boldsymbol{w}_{1}$, $\boldsymbol{w}_{2}$,
and $\boldsymbol{w}_{3}$, respectively. A tree of depth two can be
illustrated as follows

\begin{center}
	\begin{tikzpicture}[level distance=1.5cm, edge from parent/.style={draw,-latex}]
		\node[circle, draw] {$h_1$} 
		[sibling distance=6cm] 
		child {node[circle, draw] {$h_2$}
			[sibling distance=3.5cm] 
			child {node {${h_1^+\cap h_2^+}$} edge from parent node[left] {\(\bar{\boldsymbol{w}}_{2}^{T}\boldsymbol{x} \geq 0\)}}
			child {node {${h_1^+\cap h_2^-}$} edge from parent node[right] {\(\bar{\boldsymbol{w}}_{2}^{T}\boldsymbol{x} < 0\)}}
			edge from parent node[left] {\(\bar{\boldsymbol{w}}_{1}^{T}\boldsymbol{x} \geq 0\)}}
		child {node[circle, draw] {$h_3$}
			[sibling distance=3.5cm] 
			child {node {${h_1^-\cap h_3^+}$} edge from parent node[left] {\(\bar{\boldsymbol{w}}_{3}^{T}\boldsymbol{x} \geq 0\)}}
			child {node {${h_1^-\cap h_3^-}$} edge from parent node[right] {\(\bar{\boldsymbol{w}}_{3}^{T}\boldsymbol{x} < 0\)}}
			edge from parent node[right] {\(\bar{\boldsymbol{w}}_{1}^{T}\boldsymbol{x} < 0\)}};
	\end{tikzpicture}
\end{center}Each hyperplane $h_{i}$ partitions $\mathbb{R}^{D}$ into two regions:
the positive region\\
 $h_{i}^{+}=\left\{ x\mid x\in\mathbb{R}^{D},\boldsymbol{w}_{i}^{T}\bar{x}\geq0\right\} $
and the negative region $h_{i}^{-}=\left\{ x\in\mathbb{R}^{D},\boldsymbol{w}_{i}^{T}\bar{x}<0\right\} $,
where $\bar{x}=\left(x,1\right)\in\mathbb{R}^{D+1}$ denotes the data
points in \emph{homogeneous coordinates}. For brevity, we omit the
labels on the branch arrows $\boldsymbol{w}^{T}x<0$ and $\boldsymbol{w}^{T}x\geq0$,
assuming by default that the left branch corresponds to the ``less
than'' case and the right branch to the ``greater than or equal to''
case.

In ML, the most commonly used decision rules are defined by axis-parallel
hyperplanes, denoted by $\mathcal{H}^{0}$. In our discussion we will
generalize the classical use of $\mathcal{H}^{0}$ to \emph{$M$-degree
	polynomial hypersurfaces} $\mathcal{H}^{M}$ and \emph{mixed splitting
	rules}\textbf{ }$\mathcal{H}^{m}$, for all $0\leq m\leq M$.

The set $\mathcal{H}^{m}$ consists of arbitrary polynomial hypersurfaces
defined as $h^{m}=\left\{ x\mid\boldsymbol{w}^{T}\tilde{\boldsymbol{x}}<0\right\} $,
where $\tilde{\boldsymbol{x}}=\rho_{m}\left(\bar{\boldsymbol{x}}\right)$
is the $m$-tuple embedding of $\bar{\boldsymbol{x}}$ (The embedding
will be defined formally in Subsection \ref{subsec: combinatorial geometry}).
This definition holds for all $0\leq m\leq M$, implying that a decision
tree can \emph{simultaneously incorporate} axis-parallel hyperplane
$\mathcal{H}^{0}$, general hyperplane $\mathcal{H}$, and hypersurfaces
$\mathcal{H}^{M}$.

\subsubsection*{Ancestry relation matrix}

\begin{figure}[h]
	\begin{centering}
		\includegraphics[viewport=200bp 250bp 1080bp 580bp,clip,scale=0.35]{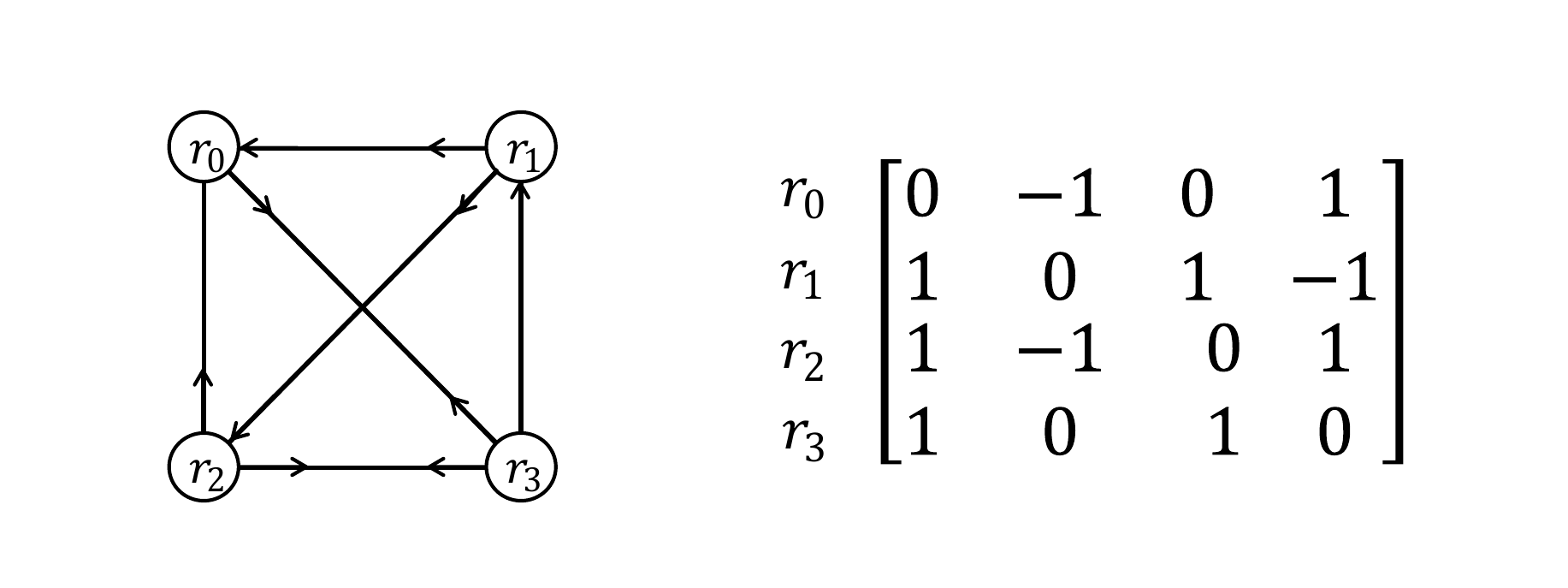}
		\par\end{centering}
	\caption{The \emph{ancestry} \emph{relation} \emph{graph} (left) captures all
		ancestry relations between four splitting rules $\left[r_{0},r_{1},r_{2},r_{3}\right]$.
		In this graph, nodes represent rules, and arrows represent ancestral
		relations. An incoming arrow from $r_{j}$ to a node $r_{i}$ indicates
		that $r_{j}$ is the right-child of $r_{i}$. The absence of an arrow
		indicates no ancestral relation. An outgoing arrows from $r_{i}$
		to a node $r_{j}$ indicates that $r_{j}$ is the left-child of $r_{i}$.
		The ancestral relation matrix (right) $\boldsymbol{K}$, where the
		elements $\boldsymbol{K}_{ij}=1$, $\boldsymbol{K}_{ij}=-1$, and
		$\boldsymbol{K}_{ij}=0$ indicate that $r_{j}$ lies on the positive
		side, negative side of $r_{i}$, or that there is no ancestry relation
		between them, respectively. \label{fig:Equivalent-representation-relation}}
\end{figure}

A key observation is that arrows $\searrow$ and $\swarrow$ are\emph{
	binary relations}. In particular, we enforce arrows $\searrow$ and
$\swarrow$ must be read from left to right, because $r_{i}\swarrow r_{j}$
and $r_{i}\searrow r_{j}$ do not imply $r_{j}\searrow r_{i}$ and
$r_{j}\swarrow r_{i}$. In other words, $\searrow$ and $\swarrow$
are \textbf{not} \emph{commutative }relations. Binary relations can
be characterized as \emph{Boolean} \emph{matrices}. However, to encode
two binary relations,$\searrow$ and $\swarrow$ in one matrix, the
values $1$ and $-1$ are used to distinguish them. We define the
ancestry relation matrix as follows.
\begin{definition}
	\emph{Ancestry relation matrix}. Given a list of $K$ rules
	$\mathit{rs}=\left[r_{1},r_{2},\ldots r_{K}\right]$, the ancestry
	relations between any pair of rules can be characterized as a $K\times K$
	square matrix $\boldsymbol{K}$, with elements defined as follows:
\end{definition}
\begin{itemize}
	\item $\boldsymbol{K}_{ij}=1$ if $r_{i}\swarrow r_{j}$ (i.e.,$r_{j}$
	is in the left subtree of $r_{i}$),
	\item $\boldsymbol{K}_{ij}=-1$ if $r_{i}\searrow r_{j}$ (i.e.,$r_{j}$
	is in the right subtree of $r_{i}$),
	\item $\boldsymbol{K}_{ij}=0$ if $r_{i}\overline{\left(\swarrow\vee\searrow\right)}r_{j}$,
	where $\overline{R}$ represent the complement relation of $R$. According
	to De Morgan's law $r_{i}\overline{\left(\swarrow\vee\searrow\right)}r_{j}=\left(r_{i}\overline{\swarrow}r_{j}\right)\wedge\left(r_{i}\overline{\searrow}r_{j}\right)$.
	In other words, $\boldsymbol{K}_{ij}=0$ if and only if $h_{j}$ is
	not a branch node in both the left and right subtree of $h_{i}$ and
	$\boldsymbol{K}_{ij}\neq0$, $i\neq j$, if $r_{i}$ is the ancestor
	of $r_{j}$.
\end{itemize}
As we will see shortly, abstracting the ancestry relation using a
matrix provides a compact and elegant representation for comparing
different splitting rules and serves as a key component in constructing
efficient algorithms for solving the size-constrained ODT problem.

\subsection{Four definitions for optimal decision tree problems\label{subsec:Four-distinct-definitions of ODT}}

After clearly defining what we mean by a ``decision tree,'' we now
turn to constructing an unambiguous and intellectually manageable
definition for optimal decision tree problem. In combinatorial optimization,
the brute-force algorithm (also known as the exhaustive enumeration
paradigm) is widely recognized as one of the most transparent approaches
to solving combinatorial optimization problems. It offers a simple
yet provably correct formalism: systematically enumerate all candidate
solutions in the search space and select the optimal one.

In this section, we first introduce a general formalism for programmatically
specifying brute-force algorithms and explain its significance. We
then demonstrate how the ODT problem can be specified unambiguously
using brute-force algorithms.

\subsubsection{Defining combinatorial optimization problem formally through brute-force
	enumeration}

In the theory of transformational programming (also known as constructive
algorithmics) \citep{bird1996algebra}, combinatorial optimization
problems such as (\ref{eq: specification in literature}) is solved
using the following, generic, \emph{generate-filter-select} paradigm,

\begin{equation}
	s^{*}=\mathit{min}_{E}\left(\mathit{filter}_{p}\left(\mathit{gen}\left(\mathit{xs}\right)\right)\right)\label{eq: combinatorial definition}
\end{equation}
The \emph{combinatorial generator} $\mathit{gen}:\left[\mathcal{A}\right]\to\left[S\right]$
generates a \emph{list} of all solutions of type $S$ corresponding
to elements in the search space $\mathcal{S}\left(\mathit{xs}\right)$
(parameterized by $\mathit{xs}$), with respect to a data list $\mathit{xs}:\left[\mathcal{A}\right]$.
Infeasible solutions are then filtered out using the function $\mathit{filter}:\left(S\to\mathit{Bool}\right)\times\left[S\right]\to\left[S\right]$
by retaining only those solutions that satisfy a predicate $p:S\to\mathit{Bool}$.
Finally the function $\mathit{min}_{E}$ select the optimal solution
with respect to objective $E$. In other words, (\ref{eq: combinatorial definition})
is essentially a brute-force program.

Astute readers may observe that (\ref{eq: combinatorial definition})
differs from (\ref{eq: specification in literature}) only in \textbf{form};
they are not fundamentally different. In other words, specification
(\ref{eq: combinatorial definition}) remains ambiguous, since we
have not yet defined $\mathit{min}_{E}$, $\mathit{filter}_{p}$,
and $\mathit{gen}$. The definitions of $\mathit{min}_{E}$ and $\mathit{filter}_{p}$
are relatively straightforward. We can define $\mathit{min}$ recursively
with respect to an objective function $E$ as

\begin{equation}
	\begin{aligned}\mathit{min}_{E} & \left(\left[a\right]\right)=a\\
		\mathit{min}_{E} & \left(a:\mathit{as}\right)=\mathit{smaller}_{E}\left(a,\mathit{min}_{E}\left(\mathit{as}\right)\right),
	\end{aligned}
\end{equation}
where $\mathit{smaller}_{E_{\text{0-1}}}\left(a,b\right)=a$ if $\ensuremath{E_{\text{0-1}}\left(a\right)\leq E_{\text{0-1}}\left(b\right)}$
and $b$ otherwise. Similarly the $\mathit{filter}_{p}$ is defined
recursively as

\begin{equation}
	\begin{aligned}\mathit{filter}_{p} & \left(\left[\:\right]\right)=\left[\:\right]\\
		\mathit{filter}_{p} & \left(a:\mathit{as}\right)=\begin{cases}
			a:\mathit{filter}_{p}\left(\mathit{as}\right) & p\left(a\right)=\mathit{True}\\
			\mathit{filter}_{p}\left(\mathit{as}\right) & p\left(a\right)=\mathit{False}
		\end{cases},
	\end{aligned}
\end{equation}
The definition of $\mathit{gen}$ will be the main focus of the remainder
of this subsection. We provide four distinct ways of defining $\mathit{gen}$
each yielding a corresponding definition of the optimal decision tree
problem. Before presenting the definitions of $\mathit{gen}$ in the
context of ODT problems, we first consider the following three questions:
\begin{quote}
	Why is (\ref{eq: combinatorial definition}) important for the study
	of optimal algorithms? How does (\ref{eq: combinatorial definition})
	differ from MIP specifications or the ambiguous definitions used in
	BnB methods? Most importantly, how can we construct an efficient program
	that guarantees solutions aligned with Equation (\ref{eq: combinatorial definition})?
\end{quote}

\paragraph{Importance of brute-force specifications in deriving efficient solutions}

The primary challenge in combinatorial optimization is to construct
efficient, optimal solutions that satisfy the constrains of the original
problem definition, such as that given in (\ref{eq: combinatorial definition}).
A MIP specification in standard form can directly yield efficient
solutions when using a general-purpose solver. However, a brute-force
specification, as defined in (\ref{eq: combinatorial definition})
may initially seem computationally infeasible due to the intractable
combinatorics of many problems. The deceptive simplicity of the brute-force
algorithm can make it seem irrelevant to the construction of efficient
combinatorial algorithms, such as dynamic programming, branch-and-bound
(BnB), greedy methods, or divide-and-conquer (D\&C). However, the
relationship between brute-force specifications and efficient combinatorial
optimization algorithms is far more nuanced than this deceptive appearance
suggests.

Within the constructive programming community \citep{bird1996algebra,de1994categories,meertens1986algorithmics},
it has long been recognized that brute-force algorithms can serve
as a foundation for deriving efficient optimization algorithms through
\emph{constructive proofs}. \emph{A constructive proof is a form of
	mathematical reasoning in which one explicitly demonstrates how to
	construct the object or solution being claimed to exist. }In other
words, it not only establishes existence but also provides a computational
procedure, typically derived through equations between programs. In
computer science, this process, known as \emph{program calculus},
involves establishing equivalences between programs. The \emph{Bird-Meertens
	formalism} \citep{bird1987introduction,meertens1986algorithmics,bird1996algebra}
is a notable framework that systematically derives programs from their
specifications through equational reasoning.

Using this formalism, the derived program inherently satisfies the
properties of the original specification, ensuring both computational
efficiency and formal correctness. Thus, for combinatorial optimization
problems, a rigorous and systematic approach begins with a provably
correct brute-force specification, from which efficient solutions
are derived constructively. This method ensures that the resulting
algorithms are both computationally efficient and formally verified,
serving as the central formalism for algorithm derivation in this
paper.

A natural question arises: Can efficient solutions (efficient combinatorial
optimization algorithms) always be derived from brute-force specifications,
and how comprehensive is this derivation paradigm across combinatorial
methods? \citet{bird1996algebra} formalized classical combinatorial
methods, such as dynamic programming, greedy algorithms, and D\&C,
using a generic recursive program (hylomorphism) and derived efficient
solutions as theorems. Building on this, \citet{he2025ROF} extended
the framework to include BnB by incorporating \emph{search strategies}
into Bird's theorems, further broadening the applicability of this
approach. In practice, one only needs to verify the conditions of
these theorems, after which program fusion follows automatically.

Furthermore, \citet{he2025ROF} established the following inclusion
relationships among different combinatorial optimization methods:

\begin{equation}
	\text{SDP}\subseteq\text{Greedy algorithm}\subseteq\text{BnB}\subseteq\text{General SDP}\subseteq\text{ DP},\text{Classical D\&C}\subseteq\text{\text{General D\&C}},\label{eq: inclusion relations for CO methods}
\end{equation}
where SDP stands for ``sequential decision process,'' a special form
of recursion characterized by sequential and deterministic recursive
steps. In simple, the inclusion relations in (\ref{eq: inclusion relations for CO methods})
shows that a general definition of D\&C algorithms encompasses all
classical combinatorial optimization methods. Both DP and classical
D\&C correspond to problems with overlapping and non-overlapping subproblems,
respectively, and both can be viewed more generally than the BnB method,
which is a sequential decision process combined with additional techniques,
such as thinning and search strategies. Greedy algorithms represent
the simplest instance, corresponding to the fused program of an SDP
generator with the selector $\mathit{min}$.

Importantly, the inclusion (\ref{eq: inclusion relations for CO methods})
is \emph{not} based on intuitive classification but on a rigorous
formal abstraction of each method using abstract recursion, and these
inclusion relations naturally emerge from the abstract level of these
recursions. See Section II.2.9 of \citet{he2025ROF} for precise definitions
of these terms and explanations.

In order to construct an efficient combinatorial optimization algorithm
using constructive approach, two key aspects must be considered in
the design of the generator: (i) the efficiency of the generator itself,
and (ii) its feasibility for program fusion, i.e., the compatibility
between the structures of $\mathit{gen}$ and the objective function
$E$.

A carefully designed generator whose structure aligns with both the
algorithmic process and the underlying hardware/software (e.g., enabling
vectorization) often yields more efficient solutions than algorithms
with better theoretical complexity but poor hardware compatibility.
Therefore, when designing $\mathit{gen}$, hardware compatibility
and parallelizability should be considered from the outset, rather
than designing an ad-hoc sequential algorithm and attempting to optimize
and parallelize it post-hoc. Moreover, structural compatibility between
the objective function $E$ and the generator $\mathit{gen}$ enables
program fusion for $\mathit{min}_{E}$ and $\mathit{gen}$. Such fusion
produces programs that are not only more succinct but also more efficient,
as infeasible or suboptimal configurations can be pruned before being
extended into complete solutions.

\paragraph{Comparison of brute-force and mixed integer programming Specifications}

Compared to a MIP specification, such as that presented in \citet{bertsimas2017optimal},
a brute-force specification defined programmatically (rather than
through pseudo-code) can be equally or more rigorous. Programmatic
definitions often provide greater clarity in computational contexts
because the minimization function in MIP specifications may lack precise
definition due to its reliance on set theory. Set theory, while mathematically
robust, is less suitable for computer science applications, as algorithms
require deterministic procedures that account for element ordering,
which set theory typically disregards.

In contrast to BnB algorithms, where the search space is often ambiguously
described using symbolic notation, a brute-force specification can
define the search space $\mathcal{S}\left(\mathit{xs}\right)$ clearly
using a combinatorial generator $\mathit{gen}\left(\mathit{xs}\right)$.
This allows researchers to verify whether the search space aligns
with the problem's requirements. If the search space $\mathcal{S}\left(\mathit{xs}\right)$
for an ODT problem is not explicitly defined, or if an algorithm is
proposed without a clear problem formulation, verifying its correctness
becomes impossible. In the following subsection, we present three
distinct definitions for $\mathcal{S}_{\text{size}}\left(K,\mathit{rs}\right)$
(size-constrained ODT) and one definition for $\mathcal{S}_{\text{depth}}\left(d,\mathit{xs}\right)$
(depth-constrained ODT). These definitions serve as the foundation
for future research on ODT problems. They not only resolve ambiguities
of problem definitions in previous studies that relied on the BnB
method but also enable researchers to approach ODT problems from different
perspectives.

\subsubsection{Defining optimal decision tree problem through brute-force specifications}

Given a list of rule $\mathit{rs}:\left[\mathcal{R}\right]$, and
the size constraint $K:\mathbb{N}$, if we can define a combinatorial
generator $\mathit{genDTSs}\left(K,\mathit{rs}\right)$ (short for
``generate decision trees with size constraints'') that exhaustively
enumerates all possible decision trees with exactly $K$ leaves with
respect to $\mathit{rs}$, we can define the size-constrained decision
tree problem as
\begin{align}
	\mathit{odk}_{\text{size}}\left(K\right) & :\left[\mathcal{R}\right]\to\mathit{DTree}\left(\mathcal{R},\mathcal{D}\right)\label{eq: specification of odt_size}\\
	\mathit{odk}_{\text{size}}\left(K\right) & =\mathit{min}_{E}\circ\mathit{genDTSs}\left(K\right)
\end{align}
Since there are three possible definitions for $\mathit{odk}_{\text{size}}$,
we replace its subscript with the name of the construction method.
For example, $\mathit{odk}_{\text{x}}$ denotes the ``x method for
optimal decision tree algorithm with size constraints''. Similarly,
the depth-constrained tree is defined as

\begin{align}
	\mathit{odk}_{\text{depth}}\left(K\right) & :\mathcal{D}\to\mathit{DTree}\left(\mathcal{R},\mathcal{D}\right)\label{eq: specification of odt_depth}\\
	\mathit{odk}_{\text{depth}}\left(K\right) & =\mathit{min}_{E}\circ\mathit{genDTDs}\left(K\right)
\end{align}
where $\mathit{genDTSDs}:\mathbb{N}\times\mathcal{D}\to\left[\mathit{DTree}\left(\mathcal{R},\mathcal{D}\right)\right]$
(short for ``generate decision trees with depth constraints'') exhaustively
enumerates all possible decision trees of a given depth. We note that
$\mathit{odk}_{\text{depth}}$ and $\mathit{odk}_{\text{size}}$ take
different types of input for the sake of program elegance. The input
$\mathit{rs}:\left[\mathcal{R}\right]$ of $\mathit{odk}_{\text{size}}$
can be generated from data, which will be discussed in Subsection
\ref{subsec:Splitting-rule-generator}.

For a decision tree satisfying Axiom \ref{def:Axioms-for Proper-DT},
\citet{he2025odt} demonstrated the following important theorem, which
reveals a deep connection between the combinatorics of \emph{decision
	trees} and \emph{permutations}.
\begin{theorem}
	\emph{A decision tree consisting of $K$ splitting rules corresponds
		to a unique $K$-permutation if and only if it is proper. In other
		words, there exists an injective mapping from proper decision trees
		to valid $K$-permutation (i.e., $K$-permutations that satisfies
		the proper decision tree axiom).}
\end{theorem}
\begin{proof}
	See \citet{he2025odt} for more details.
\end{proof}
As a consequence, the following lemma is immediate.
\begin{lemma}
	\emph{The search space of decision trees with $K$ splitting rule
		$\mathcal{S}_{\text{size}}\left(K,\mathit{rs}\right)$ with respect
		to a list of rules $rs$ is included in $\mathcal{S}_{\text{kperms}}\left(K,\mathit{rs}\right)$,
		the space of $K$-permutations with respect to $rs$. In other
		words, we have}
		\begin{equation}
			\mathcal{S}_{\text{size}}\left(K,\mathit{rs}\right)\subseteq\mathcal{S}_{\text{kperms}}\left(K,\mathit{rs}\right)
		\end{equation}	
\end{lemma}
Moreover, it can be shown that a decision tree attains maximal combinatorial
complexity, i.e., $\mathcal{S}_{\text{DTree}}\left(K,\mathit{rs}\right)=\mathcal{S}_{\text{kperms}}\left(K,\mathit{rs}\right)$,
when any rule can serve as the root and each branch node has exactly
one child \citet{he2025odt}.

Since permutations are among the most extensively studied combinatorial
structures, this characterization provides a fundamental basis for
analyzing the combinatorial and algorithmic properties of decision
trees. \citet{he2025CGs} present an elegant and efficient $K$-permutation
generator expressed in a divide-and-conquer form. For simplicity,
we denote this generator as $\mathit{kperms}\left(K,rs\right)$.

As a result, we can formally define the search space $\mathcal{S}_{\text{size}}\left(K,\mathit{rs}\right)$
using a combinatorial generator $\mathit{genDTSs}_{\text{kperms}}$
(short for ``generate size-constrained decision trees based on $K$-permutation.''),
defined as

\begin{equation}
	\begin{aligned}\mathit{genDTSs}_{\text{kperms}}\left(K\right) & :\left[\mathcal{R}\right]\to\left[\left[\mathcal{R}\right]\right]\\
		\mathit{genDTSs}_{\text{kperms}}\left(K\right) & =\mathit{filter}_{p}\circ\mathit{kperms}_{K}.
	\end{aligned}
	\label{eq: specification of genDTKs based on kperms}
\end{equation}
The program first generates all possible $K$-permutations using $\mathit{kperms}_{K}$,
and then filters out, via $\mathit{filter}_{p}$, those permutations
that cannot be used to construct proper decision trees. This two-step
process ensures that only \emph{valid permutations}—i.e., those $K$-permutations
of splitting rules satisfying the decision tree axiom \ref{subsec:Four-distinct-definitions of ODT},
and meeting the structural and combinatorial requirements of proper
decision trees—are retained for further computations.

Therefore, the following specification for the optimal decision tree
problem follows immediately:
\begin{align*}
	\mathit{odk}_{\text{size}}\left(K\right) & :\left[\mathcal{R}\right]\to\left[\mathcal{R}\right]\\
	\mathit{odk}_{\text{size}}\left(K\right) & =\mathit{min}_{E_{\text{0-1}}}\circ\mathit{genDTSs}_{\text{kperms}}\left(K\right)
\end{align*}
where $\mathit{odk}_{\text{kperms}}\left(K,\mathit{rs}\right)$ select
the optimal $K$-permutation from all feasible $K$-permutations generated
by $\mathit{genDTKs}_{\text{kperms}}\left(K,\mathit{rs}\right)$.

\paragraph{Definition of the generator based on tree datatype for size-constrained
	tree}

The definition based on $K$-permutations remains unsatisfactory.
The number of proper decision trees for a given set of rules is typically
much smaller than the total number of permutations, and determining
whether a permutation is feasible is often non-trivial. In the hyperplane
decision tree problem, for example, the feasibility test requires
$O\left(K^{2}D^{2}\right)$ operations in the worst-case, where $D$
is the dimension of the data. As predicted by \citet{he2025odt},
if each hyperplane classifies a data point into the positive or negative
class with equal probability (i.e., 1/2 for each class), the probability
of a decision tree achieving maximal combinatorial complexity is only
$O\left(\left(\frac{1}{2}\right)^{D\times K^{2}}\right)$.

Nevertheless, the following observation suggests a more efficient
approach. Consider the search space $\mathcal{S}_{\text{size}}\left(K,\mathit{rs}\right)$
for $\left|\mathit{rs}\right|\geq K$. The following equivalence holds

\begin{equation}
	\mathcal{S}_{\text{size}}\left(K,\mathit{rs}\right)=\bigcup_{\mathit{rs}_{K}\in\mathcal{S}_{\text{kcombs}}\left(K,\mathit{rs}\right)}\mathcal{S}_{\text{size}}\left(K,\mathit{rs}_{K}\right),\label{eq: size_k equivalence}
\end{equation}
where $\mathit{rs}_{K}=\left[r_{1},r_{2},\ldots r_{K}\right]$ denote
a list of $K$ rules. By assuming that $\mathcal{S}_{\text{size}}\left(K,\mathit{rs}_{K}\right)$
can be generated from $\mathit{genDTs}\left(\mathit{rs}_{K}\right)$
(short for ``generate decision trees with exactly $K$ splitting rules'',
where the parameter $K$ is implicit, since $\left|\mathit{rs}_{K}\right|=K$)
where $\mathit{genDTs}\left(\mathit{rs}_{K}\right)=\mathit{genDTSs}_{\text{kperms}}\left(K,\mathit{rs}_{K}\right)$
is one possible definition, we can define \ref{eq: size_k equivalence}
programmatically

\begin{equation}
	\mathit{genDTSs}\left(K\right)=\mathit{concatMapL}_{\mathit{genDTs}}\circ\mathit{kcombs}\left(K\right),\label{eq: new-specification of genDTKs-origin}
\end{equation}
Interestingly, \citet{he2025odt} showed that $\mathcal{S}_{\text{size}}\left(K,\mathit{rs}_{K}\right)$
can be generated more efficiently by using a function $\mathit{genDTs}_{\text{rec}}$
(short for ``recursive method for generating decision tree with exactly
$K$ splitting rules'') than with $\mathit{genDTSs}_{\text{kperms}}\left(K,\mathit{rs}_{K}\right)$.
The recursive generator $\mathit{genDTs}_{\text{rec}}:\left[\mathcal{R}\right]\times\mathcal{D}\to\left[\mathit{DTree}\left(\mathcal{R},\mathcal{D}\right)\right]$
is based on the tree datatype $\mathit{DTree}\left(\mathcal{R},\mathcal{D}\right)$,
which is defined as

\begin{equation}
	\begin{aligned}\mathit{genDTs}_{\text{rec}} & :\mathcal{D}\times\left[\mathcal{R}\right]\to\left[\mathit{DTree}\left(\mathcal{R},\mathcal{D}\right)\right]\\
		\mathit{genDTs}_{\text{rec}} & \left(\mathit{xs},\left[\;\right]\right)=\left[\mathit{DL}\left(\mathit{xs}\right)\right]\\
		\mathit{genDTs}_{\text{rec}} & \left(\mathit{xs},\left[r\right]\right)=\left[DN\left(\mathit{DL}\left(\mathit{xs}^{+}\right),r,\mathit{DL}\left(\mathit{xs}^{-}\right)\right)\right]\\
		\mathit{genDTs}_{\text{rec}} & \left(\mathit{xs},\mathit{rs}\right)=\bigg[\mathit{DN}\left(\mathit{mapD}{}_{\cap_{\mathit{xs}^{+}}}\left(u\right),r_{i},\mathit{mapD}{}_{\cap_{\mathit{xs}^{-}}}\left(v\right)\right)\mid\\
		& \begin{aligned} & \quad\quad\quad\left(\mathit{rs}^{+},r_{i},\mathit{rs}^{-}\right)\leftarrow\mathit{splits}\left(\mathit{rs}\right),u\leftarrow\mathit{genDTs}\left(\mathit{xs},\mathit{rs}^{+}\right),v\leftarrow\mathit{genDTs}\left(\mathit{xs},\mathit{rs}^{-}\right)\bigg].\end{aligned}
	\end{aligned}
\end{equation}
where $splits\left(\mathit{rs}\right)=\left[\left(\mathit{rs}^{+},r_{i},\mathit{rs}^{-}\right)\mid r_{i}\leftarrow rs,\mathit{all}\left(r_{i},\mathit{rs}\right)=\mathit{True}\right]$
($\mathit{all}\left(r_{i},rs\right)$ returns true if all rules $r_{j}$
in $rs$ satisfy $K_{ij}\neq0$ for $i\neq j$, and false otherwise),
$\mathit{rs}^{\pm}=\left[r_{j}\mid r_{j}\leftarrow rs,\boldsymbol{K}_{ij}=\pm1\right]$,
$\mathit{xs}\cap\mathit{ys}=\left[x\mid x\leftarrow\mathit{xs},x\in\mathit{ys}\right]$
(all elements that both in list $\mathit{xs}$ and $\mathit{ys}$),
and $\mathit{xs}^{\pm}=\mathit{xs}\cap r^{\pm}=\left[x\mid x\leftarrow\mathit{xs},x\in r^{\pm}\right]$
($\mathit{xs}^{+}$ collect all data points lies in the space defined
by $r^{+}$) . Assume that the ancestry relation matrix $\boldsymbol{K}$
can be pre-stored in memory (as shown for hypersurface splitting rules
in Part II). Under this assumption, the $\mathit{splits}$ function
requires only $O\left(K\right)$ operations.

The $\mathit{genDTs}$ generator function recursively constructs larger
proper decision trees $N\left(u,r_{i},v\right)$ from smaller proper
decision trees $\mathit{genDTs}\left(\mathit{rs}^{+}\right)$ and
$\mathit{genDTs}\left(\mathit{rs}^{-}\right)$, with the $\mathit{splits}$
function ensuring that only feasible splitting rules can become subtree
roots during recursion. At each step, $\mathit{genDTs}$ accumulates
information by creating a root $r_{i}$ for every proper decision
subtree generated by $\mathit{genDTs}\left(rs^{+},xs\right)$, using
the $\mathit{mapD}{}_{\cap_{r_{i}^{+}}}$ function. The $\mathit{genDTs}$
function will generate $K!$ number of tree in the worst-case, the
worst-case is achieved is precisely when the tree has a ``chain-like''
structure (every nodes of the tree has at most one child), in this
case the proper decision tree established a one-to-one correspondence
to the $K$-permutation.

This above observation leads to the following program for generating
$\mathcal{S}_{\text{size}}\left(K,\mathit{rs}\right)$

\begin{equation}
	\mathit{genDTSs}_{\text{rec}}\left(K,\mathit{rs}\right)=\mathit{concatMapL}_{\mathit{genDTs}_{\text{rec}}\left(\mathit{xs}\right)}\left(\mathit{kcombs}\left(K,\mathit{rs}\right)\right),\label{eq: new-specification of genDTKs}
\end{equation}
such that $\left|\mathit{rs}\right|\geq K$, where $\mathit{genDTSs}_{\text{rec}}\left(K\right)$
are short for ``recursive method for generating decision trees of
$K$ splitting rules.''

After formalizing the search space $\mathcal{S}_{\text{size}}\left(K,\mathit{rs}\right)$
by defining it as a concrete program (\ref{eq: new-specification of genDTKs}),
now we can formally define what we mean by optimal decision problem
by the following brute-force algorithm

\begin{equation}
	\begin{aligned}\mathit{odt}_{\text{rec}}\left(K\right) & :\left[\mathcal{R}\right]\to\mathit{DTree}\left(\mathcal{R},\mathcal{D}\right)\\
		\mathit{odt}_{\text{rec}}\left(K\right) & =\mathit{min}_{E_{\text{0-1}}}\circ\mathit{genDTSs}_{K}.\\
		& =\mathit{min}_{E_{\text{0-1}}}\circ\mathit{concatMapL}{}_{\mathit{genDTs}_{\text{rec}}\left(\mathit{xs}\right)}\circ\mathit{kcombs}\left(K\right).
	\end{aligned}
	\label{specification based on dtree}
\end{equation}
Alternatively, if the definition $\mathit{genDTs}$ is defined by
$\mathit{genDTSs}_{\text{kperms}}\left(K,\mathit{rs}_{K}\right)$,
we have 
\begin{align*}
	\mathit{odt}_{\text{kperms}}\left(K\right) & =\mathit{min}_{E_{\text{0-1}}}\circ\mathit{concatMapL}{}_{\mathit{genDTSs}_{\text{kperms}}\left(K\right)}\circ\mathit{kcombs}_{K}\\
	& =\mathit{min}_{E_{\text{0-1}}}\circ\mathit{concatMapL}{}_{\mathit{filter}_{p}\circ\mathit{kperms}_{K}}\circ\mathit{kcombs}_{K}.
\end{align*}

\paragraph{A sequential definition based on tree datatype for size-constrained
	tree}

For a size-constrained decision tree, the number of leaves is fixed,
so we know that there are at most $K$ recursive steps to construct
a full decision tree with $K$-fixed splitting rules. This provides
us an iterative method of defining the decision tree generator

\begin{equation}
	\begin{aligned}\mathit{genDTs}_{\text{vec}} & :\mathcal{D}\times\left[\mathcal{R}\right]\to\left[\mathit{DTree}\left(\mathcal{R},\mathcal{D}\right)\right]\\
		\mathit{genDTs}_{\text{vec}} & \left(\mathit{xs},\left[\;\right]\right)=\left[\mathit{DL}\left(\mathit{xs}\right)\right]\\
		\mathit{genDTs}_{\text{vec}} & \left(\mathit{xs},\mathit{rs}\right)=\mathit{concat}\circ\left[\mathit{updates}\left(r,\mathit{rs}^{\prime},xs\right)|\left(r,\mathit{rs}^{\prime}\right)\longleftarrow\mathit{candidates}\left(\mathit{rs}\right)\right]
	\end{aligned}
\end{equation}
where $\mathit{candidates}\left(\mathit{rs}\right)=\left[\left(r,\mathit{rs}\backslash r\right)\mid r\leftarrow\mathit{rs}\right]$
generates all possible rules in $\mathit{rs}$ and $\mathit{rs}\backslash r$
denotes eliminating $r$ from list $\mathit{rs}$, and $\mathit{updates}$
is defined as
\begin{equation}
	\begin{aligned}\mathit{updates} & :\mathcal{R}\times\left[\mathcal{R}\right]\to\left[\mathit{DTree}\left(\mathcal{R},\mathcal{D}\right)\right]\\
		\mathit{updates} & \left(r,\mathit{rs}\backslash r,\mathit{xs}\right)=\mathit{catMaybes}\left[\mathit{update}\left(r,t\right)\mid t\leftarrow\mathit{genDTs}_{\text{vec}}\left(\mathit{rs}\backslash r,\mathit{xs}\right)\right]
	\end{aligned}
\end{equation}
which call $\mathit{genDTs}_{\text{vec}}$ recursively, and append
a new rule $r$ to every tree $t$ generated by $\mathit{genDTs}_{\text{vec}}\left(\mathit{rs}\backslash r,xs\right)$
by using $\mathit{update}$ function

\begin{equation}
	\begin{aligned}\mathit{update} & :\mathcal{R}\to\mathit{DTree}\left(\mathcal{R},\mathcal{D}\right)\to\mathit{Maybe}\left(\mathit{DTree}\left(\mathcal{R},\mathcal{D}\right)\right)\\
		\mathit{update} & \left(r,\mathit{DL}\left(\mathit{xs}\right)\right)=\mathit{Just}\left(DN\left(\mathit{DL}\left(xs^{+}\right),r,\mathit{DL}\left(xs^{-}\right)\right)\right)\\
		\mathit{update} & \left(r,\mathit{DN}\left(u,s,v\right)\right)=\begin{cases}
			\begin{cases}
				\mathit{DN}\left(\mathit{update}\left(r,u\right),s,v\right) & \mathit{update}\left(r,u\right)\neq\mathit{Nothing}\\
				\mathit{Nothing} & \text{otherwise}
			\end{cases} & \boldsymbol{K}_{sr}=1\\
			\begin{cases}
				\mathit{DN}\left(u,s,\mathit{update}\left(r,v\right)\right) & \mathit{update}\left(r,v\right)\neq\mathit{Nothing}\\
				\mathit{Nothing} & \text{otherwise}
			\end{cases} & \boldsymbol{K}_{sr}=-1\\
			\mathit{Nothing} & \text{otherwise}
		\end{cases}.
	\end{aligned}
\end{equation}
where $\mathit{Maybe}\left(\mathcal{A}\right)=\mathit{Nothing}\mid\mathit{Just}\left(\mathcal{A}\right)$
represents a value that may or may not be present and $\texttt{\ensuremath{\mathit{catMaybes}}}:\left[\mathit{Maybe}\left(\mathcal{A}\right)\right]\to\left[\mathcal{A}\right]$
filters out $\mathit{Nothing}$ from a list of $\mathit{Maybe}$ values
and extracts the contents of$\mathit{Just}$ values into a plain list.

Although the definition of e $\mathit{update}$ update appears complex,
it is applied recursively along the path of the old tree $\mathit{DN}\left(u,s,v\right)$
to determine whether adding a new splitting rule $r$ would result
in a non-proper tree. If the tree remains proper, $r$ is added to
the leaf; otherwise, $\mathit{Nothing}$ is returned, indicating a
non-proper decision tree.

Because the computation in $\mathit{genDTs}_{\text{vec}}$ is sequential,
it can be fully vectorized. The generator $\mathit{genDTs}_{\text{vec}}$
(short for ``vectorized method generating decision trees with $K$
splitting rules.'') benefits particularly from this property, as it
allows full utilization of modern hardware, such as GPUs. We discuss
this advantage in detail in Subsection \ref{subsec:Comparison-of-four algs}.

One can check $\mathit{genDTs}_{\text{vec}}$ will generate the same
\emph{set} of proper decision trees as $\mathit{genDTs}_{\text{rec}}$,
but will produce different \emph{lists} of trees. This is because
\emph{different valid $K$-permutation can correspondent to the same
	tree (vice is not true)}. For instance, the tree

\begin{center}
	\begin{tikzpicture}
		\node[draw, circle] (r1) at (0, 0) {$r_1$};
		\node[draw, circle] (r2) at (-1.5, -1.5) {$r_2$};
		\node[draw, circle] (r3) at (1.5, -1.5) {$r_3$};
		\draw (r1) -- (r2);
		\draw (r1) -- (r3);
	\end{tikzpicture}
\end{center} can be constructed by sequentially traversing list $\left[r_{1},r_{2},r_{3}\right]$
or $\left[r_{1},r_{3},r_{2}\right]$. To remove duplicate trees, we
need an additional $\mathit{unique}:\left[\mathit{DTree}\left(\mathcal{R},\mathcal{D}\right)\right]\to\left[\mathit{DTree}\left(\mathcal{R},\mathcal{D}\right)\right]$
operation which removes all duplicate trees in a list. This gives
us

\begin{equation}
	\begin{aligned}\mathit{genDTs}_{\text{vec}} & :\mathcal{D}\times\left[\mathcal{R}\right]\to\left[\mathit{DTree}\left(\mathcal{R},\mathcal{D}\right)\right]\\
		\mathit{genDTs}_{\text{vec}} & \left(\mathit{xs},\left[\;\right]\right)=\left[\mathit{DL}\left(\mathit{xs}\right)\right]\\
		\mathit{genDTs}_{\text{vec}} & \left(\mathit{xs},\mathit{rs}\right)=\mathit{unique}\circ\mathit{concat}\circ\left[\mathit{updates}\left(r,\mathit{rs}^{\prime},xs\right)|\left(r,\mathit{rs}^{\prime}\right)\longleftarrow\mathit{candidates}\left(\mathit{rs}\right)\right]
	\end{aligned}
\end{equation}
Now this new $\mathit{genDTs}_{\text{vec}}$ will produce the same
trees as $\mathit{genDTs}_{\text{rec}}$, albeit in different ordering.

Finally, a sequential definition (a brute-force program) for the size-constrained
optimal decision tree problem is defined as
\begin{equation}
	\mathit{odt}_{\text{vec}}\left(K\right)=\mathit{min}_{E_{\text{0-1}}}\circ\mathit{concatMapL}{}_{\mathit{genDTs}_{\text{vec}}\left(\mathit{xs}\right)}\circ\mathit{kcombs}_{K}.\label{eq: specfication based on vectorzied}
\end{equation}

\paragraph{Definition based on tree datatype for depth-constrained tree}

Depth-constrained decision trees are often used in the study of ODT
problems with smaller combinatorial complexity, such as axis-parallel
decision trees \citep{mazumder2022quant,brita2025optimal} or decision
trees over binary feature data \citet{demirovic2022murtree,verwer2019learning,nijssen2007mining,aglin2020learning,nijssen2010optimal}.
However, none of these studies formally define their problem or explain
why their algorithm correctly solves the desired problem. In this
subsection, we provide a formal characterization of decision trees
defined in this manner.

Moreover, while depth-constrained trees are widely used in previous
studies on axis-parallel and binary-data decision trees, they are
unsuitable for ODT problems with more complex splitting rules, such
as hyperplanes or hypersurfaces. We explain this limitation from an
algorithmic perspective in Subsection \ref{subsec:Comparision-between-leave-depth}.

A depth-constrained tree generator $\mathit{genDTDs}_{\text{depth}}$
receives a pre-specified tree depth $d:\mathbb{N}$ and a list of
data $\mathit{xs}$ as input, and is defined as

\begin{equation}
	\begin{aligned}\mathit{genDTDs}_{\text{depth}} & :\mathbb{N}\times\mathcal{D}\to\left[\mathit{DTree}\left(\mathcal{R},\mathcal{D}\right)\right]\\
		\mathit{genDTDs}_{\text{depth}} & \left(0,\mathit{xs}\right)=\left[\mathit{DL}\left(\mathit{xs}\right)\right]\\
		\mathit{genDTDs}_{\text{depth}} & \left(1,\mathit{xs}\right)=\left[DN\left(\mathit{DL}\left(\mathit{xs}^{+}\right),r,\mathit{DL}\left(\mathit{xs}^{-}\right)\right)\mid r\leftarrow\mathit{gen_{splits}}\left(\mathit{xs}\right)\right]\\
		\mathit{genDTDs}_{\text{depth}} & \left(d,\mathit{xs}\right)=\big[\mathit{DN}\left(u,r,v\right)\mid\\
		 u\leftarrow&\mathit{genDTDs}_{\text{depth}}\left(\mathit{xs}^{+},d-1\right),v\leftarrow\mathit{genDTDs}_{\text{depth}}\left(\mathit{xs}^{-},d-1\right),r\leftarrow\mathit{gen_{splits}}\left(\mathit{xs}\right)\big]
	\end{aligned}
	\label{spec: ODT_depth}
\end{equation}
This definition states that all possible decision trees of depth $d$
with respect to data list $\mathit{xs}$ can be constructed from decision
trees of depth $d-1$ on the subsets $\mathit{xs}^{+}$ and $\mathit{xs}^{-}$.
Currently, a concrete definition for $\mathit{gen_{splits}}$ is not
provided, as the splitting rule $r$ defining the decision tree problem
has not yet been formalized. For the classical axis-parallel ODT problem,
$\mathit{gen_{splits}}$ is straightforward: splitting rules can be
generated from axis-parallel hyperplanes based on the data points.
Equivalently, \citet{brita2025optimal,mazumder2022quant} use the
medians of adjacent pairs of data points for each dimension $D$.
A more detailed discussion of the various definitions of $\mathit{gen_{splits}}$
for solving different ML problems is provided in Subsection \ref{subsec:Splitting-rule-generator}.

For a data list of size $N$, there are $N\times D$ possible splitting
rules. Each rule $r$ corresponds to an integer value defined by the
$d$-th feature of the data points $x\in\mathbb{R}^{D}$, with $r^{+}=\left\{ x\mid x_{d}>r\right\} $
and $r^{-}=\left\{ x\mid x_{d}\leq r\right\} $. We will define $\mathit{gen_{splits}}$
for hypersurface decision trees in Subsection \ref{subsec:Splitting-rule-generator},
after discussing the geometry of hypersurface splitting rules.

Therefore, the optimal decision tree problem with depth constraint
$d:\mathbb{N}$ is defined as
\begin{equation}
	\begin{aligned}\mathit{odt}_{\text{depth}} & :\mathbb{N}\times\mathcal{D}\to\mathit{DTree}\left(\mathcal{R},\mathcal{D}\right)\\
		\mathit{odt}_{\text{depth}} & \left(d\right)=\mathit{min}{}_{E}\circ\mathit{genDTDs}_{\text{depth}}\left(d\right)
	\end{aligned}
	\label{specification based on depth constrain}
\end{equation}

Some readers may find the recursion \ref{spec: ODT_depth} closely
related to the algorithm given by \citet{brita2025optimal}, which
has similar recursive structure. We will see shortly that \citet{brita2025optimal}'s
algorithm can indeed be derived from the specification (\ref{specification based on depth constrain}),
which proves why \citet{brita2025optimal}'s algorithm is optimal.

\subsection{Four solutions for optimal decision tree problems\label{subsec:Four solutions to ODT}}

As mentioned earlier, the efficiency of an algorithm for solving a
problem is largely determined by the design of its generator, which
implicitly defines the problem. In this section, we show that the
four definitions introduced in \ref{subsec:Four-distinct-definitions of ODT}
each lead to an efficient algorithm for solving the ODT problem.

\subsubsection{Optimal decision tree algorithms for size-constrained tree}

\paragraph{Simplified decision tree problem}

The ODT problem defined in (\ref{specification based on dtree}) is
inherently difficult to solve directly. Fortunately, \citet{he2025odt}
demonstrated that the original ODT problem can be reduced to solving
a simplified optimal decision tree problem, without loss of generality.
This reduction not only makes the resulting algorithm significantly
more efficient but also provides greater flexibility for designing
parallel algorithms.
\begin{theorem}
        Simplified decision tree problem.\emph{ Given a list of rules $\mathit{rs}:\left[\mathcal{R}\right]$
		and a size constraint $K:\mathbb{N}$, we have following equivalence
		\begin{equation}
			\mathit{min}_{E}\circ\mathit{concatMapL}{}_{\mathit{sodt}\left(\mathit{xs}\right)}\circ\mathit{kcombs}_{K}=\mathit{min}_{E}\circ\mathit{concatMapL}{}_{genDTs\left(\mathit{xs}\right)}\circ\mathit{kcombs}_{K}\label{eq: simplified ODT problem}
		\end{equation}
		where $\mathit{sodt}\left(\mathit{xs}\right)=\mathit{min}_{E_{\text{0-1}}}\circ\mathit{genDTs}\left(\mathit{xs}\right)$.}
\end{theorem}
\begin{proof}
	The result is a simple consequence of distributivity 
	\[
	\begin{aligned} & \mathit{min}{}_{E}\circ\mathit{mapL}{}_{genDTs\left(\mathit{xs}\right)}\circ\mathit{kcombs}{}_{K}\\
		\equiv & \text{ distributivity \ensuremath{\mathit{min}_{E}\circ\mathit{concat}=\mathit{min}_{E}\circ\mathit{mapL}_{\mathit{min}_{E}}}}\\
		& \mathit{min}{}_{E}\circ\mathit{mapL}{}_{\mathit{min}_{E}\circ genDTs\left(\mathit{xs}\right)}\circ\mathit{kcombs}{}_{K}\\
		\equiv & \text{ define \ensuremath{\mathit{sodt}\left(\mathit{xs}\right)=\mathit{min}_{E}\circ genDTs\left(\mathit{xs}\right)}}\\
		& \mathit{min}_{E}\circ\mathit{mapL}_{\mathit{sodt}\left(\mathit{xs}\right)}\circ\mathit{kcombs}{}_{K}.
	\end{aligned}
	\]
\end{proof}
The function $\mathit{sodt}\left(\mathit{xs},\mathit{rs}_{K}\right)$
returns the optimal decision tree with respect to a \textbf{specific}
$K$-combination of splitting rules $\mathit{rs}_{K}$ over the data
list $\mathit{xs}$. By applying $\mathit{sodt}\left(\mathit{xs},\mathit{rs}_{K}\right)$
to \textbf{each} $K$-combination of rules $rs_{K}\in\mathit{kcombs}{}_{K}\left(\mathit{rs}\right)$
and then selecting the globally optimal one, we obtain the optimal
solution to the ODT problem—the left-hand side of (\ref{eq: simplified ODT problem}).
This solution is equivalent to that obtained by $\mathit{odt}\left(K\right)$
(right-hand side of (\ref{eq: simplified ODT problem})).

The left-hand formulation in (\ref{eq: simplified ODT problem}) is
computationally more efficient than the right-hand formulation because,
for each $K$-combination of $rs$, only a \textbf{single} optimal
decision tree is produced by $\mathit{sodt}$. Consequently, the $\mathit{min}{}_{E}$
function on the left-hand side of (\ref{eq: simplified ODT problem})
needs to select the optimal decision tree only from the set generated
by $\mathit{mapL}_{\mathit{sodt}\left(\mathit{xs}\right)}\circ\mathit{kcombs}{}_{K}$,
which is significantly smaller than the set of all possible decision
trees of size $K$.

Obviously, the three size-constrained generators—based on $K$-permutations,
based on recursive tree structure, and sequential process, as discussed
previously—yield three distinct definitions for $\mathit{sodt}$.
For instance, the $K$-permutation generator can be factorized into
the following two programs:
\begin{equation}
	\mathit{kperms}_{K}=\mathit{concatMap}_{\mathit{perms}}\circ\mathit{kcombs}_{K}\label{eq: k-perms factorization}
\end{equation}
where $\mathit{perms}:\left[\mathcal{A}\right]\to\left[\left[\mathcal{A}\right]\right]$
receives a list and returns all its permutations. In other words,
the set of all possible $K$-permutations is equivalent to the set
obtained by first generating all $K$-combinations and then computing
the permutations of each combination.

In particular, $\mathit{perms}$ can be defined recursively as
\begin{equation}
	\begin{aligned}\mathit{perms} & \left(\left[\:\right]\right)=\left[\left[\:\right]\right]\\
		\mathit{perms} & \left(r:\mathit{rs}\right)=\mathit{ins}\left(r,\mathit{perms}\left(\mathit{rs}\right)\right)
	\end{aligned}
\end{equation}
where function $\mathit{ins}$ insert $r$ into all possible positions
for each partial permutations in $\mathit{perms}\left(\mathit{rs}\right)$
( see subsection II.1.2.4 in \citet{he2025ROF} for details). Comparing
(\ref{eq: k-perms factorization}) with (\ref{eq: simplified ODT problem}),
we can define $\mathit{sodt}$ using $\mathit{perms}$ through the
following equivalence

\begin{equation}
	\begin{aligned}\mathit{sodt}_{\text{kperms}} & :\left[\mathcal{R}\right]\to\left[\mathcal{R}\right]\\
		\mathit{sodt}_{\text{kperms}} & =\mathit{min}_{E}\circ\mathit{genDTs}_{\text{kperms}}\\
		& =\mathit{min}_{E}\circ\mathit{filter}_{p}\circ\mathit{kperms}_{K}\\
		& =\mathit{min}_{E}\circ\mathit{filter}_{p}\circ\mathit{perms}
	\end{aligned}
	\label{eq:sodt-kperms}
\end{equation}
Similarly, by substituting $\mathit{genDTs}_{\text{rec}}$ and $\mathit{genDTs}_{\text{vec}}$
into the definition of \emph{$\mathit{sodt}$}, we obtain the corresponding
recursive and vectorized formulations of \emph{$\mathit{sodt}$}:
\begin{align}
	\mathit{sodt}_{\text{rec}}\left(\mathit{xs}\right) & =\mathit{min}_{E}\circ\mathit{genDTs}_{\text{rec}}\left(\mathit{xs}\right)\label{eq: sodt-rec}\\
	\mathit{sodt}_{\text{vec}}\left(\mathit{xs}\right) & =\mathit{min}_{E}\circ\mathit{genDTs}_{\text{vec}}\left(\mathit{xs}\right)\label{eq: sodt-vec}
\end{align}

\paragraph{Structure of objective, monotonicity, and dynamic programming}

Once we have a formal definition, it is natural to ask whether an
efficient solution—such as dynamic programming (DP) or a greedy algorithm—can
be derived for solving $\mathit{sodt}$. It has long been recognized
that the existence of a DP or greedy solution relies on identifying\emph{
	monotonicity} \citep{bird1996algebra}. Specifically, when the objective
function $E$ has an algorithmic structure compatible with that of
$\mathit{gen}$, monotonicity can be exploited to enable program fusion.
In such cases, the $\mathit{min}_{E}$ function applied within $\mathit{genDTs}$
yields an efficient recursive definition for the original specification.
The resulting solution is often referred to as a \emph{dynamic programming}
or \emph{greedy} algorithm.

To illustrate this, we show that when the objective function is defined
according to a given general scheme, two DP algorithms can be derived
from specifications (\ref{eq: sodt-rec}) and (\ref{spec: ODT_depth}).
Moreover, we can not only prove the existence of DP algorithms in
this case, but also demonstrate the non-existence of a DP formulation
for (\ref{eq:sodt-kperms}) and (\ref{eq: sodt-vec}).

We consider that the decision tree problem conforms to the following
general scheme:
\begin{equation}
	\begin{aligned}E & :\left(\mathcal{D}\to\mathbb{R}\right)\times\left(\mathbb{R}\times\mathbb{R}\to\mathbb{R}\right)\times\mathit{DTree}\left(\mathcal{R},\mathcal{D}\right)\to\mathbb{R}\\
		E & \left(f,g,\mathit{DL}\left(\mathit{xs}\right)\right)=f\left(\mathit{xs}\right)\\
		E & \left(f,g,\mathit{DN}\left(u,r,v\right)\right)=g\left(E\left(\mathit{mapD}{}_{\cap_{r^{+}}}\left(u\right)\right),E\left(\mathit{mapD}{}_{\cap_{r^{-}}}\left(v\right)\right)\right).
	\end{aligned}
	\label{eq: objective function value scheme}
\end{equation}
such that $g\left(a,b\right)\geq\max\left(a,b\right)$. For example,
the classical objective function used in classification tasks—central
to the discussion in Part II of this paper—is the \emph{0-1 loss objective},
which can be defined as
\begin{equation}
	\begin{aligned}E_{\text{0-1}} & :\mathit{DTree}\left(\mathcal{R},\mathcal{D}\right)\to\mathbb{N}\\
		E_{\text{0-1}} & \left(\mathit{DL}\left(\mathit{xs}\right)\right)=\sum_{\left(x_{i},y_{i}\right)\in\mathit{xs}}\mathbf{1}\left[\hat{y}\neq y_{i}\right]\\
		E_{\text{0-1}} & \left(\mathit{DN}\left(u,r,v\right)\right)=E_{\text{0-1}}\left(\mathit{mapD}{}_{\cap_{r^{+}}}\left(u\right)\right)+E_{\text{0-1}}\left(\mathit{mapD}{}_{\cap_{r^{-}}}\left(v\right)\right),
	\end{aligned}
\end{equation}
where $\hat{y}=\underset{k\in\mathcal{K}}{\text{argmax}}\sum_{\left(x_{i},y_{i}\right)\in\mathit{xs}}\mathbf{1}\left[y_{i}=k\right]$,
which represents the majority class in a leaf. The following lemma
trivially holds.
\begin{lemma}
	Monotonicity in the decision tree problem\emph{. Given left subtrees
		$u$ and $u^{\prime}$ and right subtrees $v$ and $v^{\prime}$ rooted
		at $r$, the implication }
\begin{equation}
	\begin{aligned}
		E\left(\mathit{mapD}_{\cap_{r^{+}}}(u)\right) \leq &  E\left(\mathit{mapD}_{\cap_{r^{\prime+}}}(u^{\prime})\right) \wedge E\left(\mathit{mapD}_{\cap_{r^{-}}}(v)\right) \leq E\left(\mathit{mapD}_{\cap_{r^{\prime-}}}(v^{\prime})\right) \\
		&\quad \Longrightarrow E\left(DN(u,r,v)\right) \leq E\left(DN(u^{\prime},r^{\prime},v^{\prime})\right)
	\end{aligned}
\end{equation}
	\emph{only holds, in general, if $r=r^{\prime}$.\label{Monotonicity}}
\end{lemma}

\paragraph{Dynamic programming algorithm for size-constrained decision tree}

Astute readers may have noticed that the recursive structure of $E$
aligns with that of $\mathit{genDTs}_{\text{rec}}$. We now show that
this alignment implies a dynamic programming (DP) solution. By leveraging
the monotonicity property \ref{Monotonicity}, we can fuse the minimizer
$\mathit{min}_{E}$ with $\mathit{genDTs}_{\text{rec}}$ in the simplified
decision tree problem defined in \ref{eq: sodt-rec}. This yields
a DP solution $\mathit{sodt}_{\text{rec}}$ for the problem $\mathit{min}_{E}\circ\mathit{genDTs}_{\text{rec}}$
\ref{eq: sodt-rec}. Thus providing an efficient algorithm for solving
the optimal decision tree problem by substituting the DP solution
$\mathit{sodt}_{\text{rec}}$ into $\mathit{odt}_{\text{rec}}$.
\begin{theorem}
	Dynamic programming algorithm for size-constrained decision tree.\emph{
		\label{thm: DP for leaf tree} Given a list of $K$ rules $\mathit{rs}_{K}:\left[\mathcal{R}\right]$.
		If the objective function $E$ can be defined according to (\ref{eq: objective function value scheme})—that
		is, if Lemma Lemma \ref{Monotonicity} holds—then the solution obtained
		by}
	\begin{equation}
		\begin{aligned}\mathit{sodt}_{\text{rec}} & :\mathcal{D}\times\left[\mathcal{R}\right]\to\mathit{DTree}\left(\mathcal{R},\mathcal{D}\right)\\
			\mathit{sodt}_{\text{rec}} & \left(\mathit{xs},\left[\;\right]\right)=\left[\mathit{DL}\left(\mathit{xs}\right)\right]\\
			\mathit{sodt}_{\text{rec}} & \left(\mathit{xs},\left[r\right]\right)=\left[\mathit{DN}\left(\mathit{DL}\left(xs^{+}\right),r,\mathit{DL}\left(xs^{-}\right)\right)\right]\\
			\mathit{sodt}_{\text{rec}} & \left(\mathit{xs},\mathit{rs}\right)=\mathit{min}{}_{E}\left[\mathit{DN}\left(\mathit{sodt}_{\text{rec}}\left(\mathit{xs}^{+},\mathit{rs}^{+}\right),r,\mathit{sodt}_{\text{rec}}\left(\mathit{xs}^{-},\mathit{rs}^{-}\right)\right)\mid\left(\mathit{rs}^{+},r,\mathit{rs}^{-}\right)\leftarrow\mathit{splits}\left(\mathit{rs}\right)\right].
		\end{aligned}
		\label{eq: DP solution to sodt_rec}
	\end{equation}
	\emph{provides a solution to the brute-force specification (\ref{eq: sodt-rec}).
		In other words, we have following relation
		\begin{equation}
			\mathit{sodt}_{\text{rec}}\left(\mathit{xs}\right)\subseteq\mathit{min}_{E}\circ\mathit{genDTs}_{\text{rec}}\left(\mathit{xs}\right),
		\end{equation}
		where the symbol ``$\subseteq$'' means the solution on the left is
	}one of the solution\emph{ on the right.}
\end{theorem}
\begin{proof}
	Due to the monotonicity of the problem, we can now derive the program
	by following equational reasoning
	
	\begin{align*}
		& \mathit{min}_{E}\circ genDTs_{\text{rec}}\left(\mathit{xs}\right)\\
		\equiv & \text{ definition of \ensuremath{genDTs_{\text{rec}}}}\\
		& \begin{aligned} & \mathit{minR}{}_{E}\bigg[\mathit{DN}\left(\mathit{mapD}{}_{\cap_{\mathit{xs}^{+}}}\left(u\right),r,\mathit{mapD}{}_{\cap_{\mathit{xs}^{-}}}\left(v\right)\right)\mid\\
			& \quad\quad\left(\mathit{rs}^{+},r,\mathit{rs}^{-}\right)\leftarrow\mathit{splits}\left(\mathit{rs}\right),u\leftarrow\mathit{genDTs}_{\text{rec}}\left(\mathit{xs},\mathit{rs}^{+}\right),v\leftarrow\mathit{genDTs}_{\text{rec}}\left(\mathit{xs},\mathit{rs}^{-}\right)\bigg]
		\end{aligned}
		\\
		\subseteq & \text{ monotonicity}\\
		& \mathit{minR}_{E}\bigg[\mathit{DN}\bigg(\mathit{minR}_{E}\left[\mathit{mapD}{}_{\cap_{\mathit{xs}^{+}}}\left(u\right)\mid u\leftarrow\mathit{genDTs}_{\text{rec}}\left(\mathit{xs},\mathit{rs}^{+}\right)\right],r,\\
		& \quad\quad\mathit{minR}_{E}\left[\mathit{mapD}{}_{\cap_{\mathit{xs}^{-}}}\left(u\right)\mid u\leftarrow\mathit{genDTs}_{\text{rec}}\left(\mathit{xs},\mathit{rs}^{-}\right)\right]\bigg)\mid\left(\mathit{rs}^{+},r,\mathit{rs}^{-}\right)\leftarrow\mathit{splits}\left(rs\right)\bigg]\\
		\equiv & \text{ definition of \ensuremath{\mathit{mapL}}}\\
		& \begin{aligned} & \mathit{minR}_{E}\bigg[\mathit{DN}\bigg(\mathit{minR}_{E}\left(mapL_{\mathit{mapD}{}_{\cap_{\mathit{xs}^{+}}}}\left(\mathit{genDTs}_{\text{rec}}\left(\mathit{xs},\mathit{rs}^{+}\right)\right)\right),r,\\
			& \quad\quad\mathit{minR}_{E}\left(mapL_{\mathit{mapD}{}_{\cap_{\mathit{xs}^{-}}}}\left(\mathit{genDTs}_{\text{rec}}\left(\mathit{xs},\mathit{rs}^{-}\right)\right)\right)\bigg)\mid\left(\mathit{rs}^{+},r,\mathit{rs}^{-}\right)\leftarrow\mathit{splits}\left(\mathit{rs}\right)\bigg]
		\end{aligned}
		\\
		\equiv & \text{ definition of \ensuremath{\mathit{mapD}}}\\
		& \begin{aligned} & \mathit{minR}_{E}\bigg[\mathit{DN}\bigg(\mathit{minR}_{E}\left(\mathit{genDTs}_{\text{rec}}\left(\mathit{r}^{-}\cap\mathit{xs},\mathit{rs}^{+}\right)\right),r,\\
			& \quad\quad\mathit{minR}_{E}\left(\mathit{genDTs}_{\text{rec}}\left(\mathit{r}^{-}\cap\mathit{xs},\mathit{rs}^{-}\right)\right)\bigg)\mid\left(\mathit{rs}^{+},r_{i},\mathit{rs}^{-}\right)\leftarrow\mathit{splits}\left(\mathit{rs}\right)\bigg]
		\end{aligned}
		\\
		\equiv & \text{ definition of \ensuremath{\mathit{sodt}_{\text{rec}}}}\\
		& \mathit{minR}_{E}\bigg[\mathit{DN}\bigg(\mathit{sodt}_{\text{rec}}\left(\mathit{xs}^{+},\mathit{rs}^{+}\right),r,\mathit{sodt}_{\text{rec}}\left(\mathit{xs}^{-},\mathit{rs}^{-}\right)\bigg)\mid\left(\mathit{rs}^{+},r,\mathit{rs}^{-}\right)\leftarrow\mathit{splits}\left(\mathit{rs}\right)\bigg].
	\end{align*}
\end{proof}
Clearly, the solution produced by $\mathit{sodt}_{\text{rec}}$ is
more efficient than the brute-force algorithm $\mathit{min}_{E}\circ genDTs_{\text{rec}}$.
The key difference is that the brute-force approach first exhaustively
enumerates all possible solutions using $genDTs_{\text{rec}}$ and
then selects the best one, whereas $\mathit{sodt}_{\text{rec}}$ recursively
constructs the optimal solution $\mathit{sodt}_{\text{rec}}\left(\mathit{xs},\mathit{rs}\right)$
from the optimal sub-solutions $\mathit{sodt}_{\text{rec}}\left(\mathit{xs}^{+},\mathit{rs}^{+}\right)$
and $\mathit{sodt}_{\text{rec}}\left(\mathit{xs}^{-},\mathit{rs}^{-}\right)$.
\citet{he2025odt} showed that $\mathit{sodt}_{\text{rec}}$ has a
worst-case time complexity of $O\left(K!\times N\right)$, assuming
that $\mathit{splits}\left(\mathit{rs}\right)$ can be computed in
linear time with respect to the size of $\mathit{rs}$, and that $\mathit{xs}^{\pm}$
can be computed in $O\left(N\right)$ time for a data list $\mathit{xs}$
of size $N$.

Consequently, the optimal decision tree problem can be solved using
(\ref{eq: simplified ODT problem}), where the $\mathit{sodt}$ is
defined by using (\ref{eq: DP solution to sodt_rec}).

\subsubsection{Optimal decision tree algorithm for depth-constrained tree}

A similar reasoning applies to the ODT problem with depth constraints,
commonly used in axis-parallel and binary-feature decision trees.
\citet{brita2025optimal} present an algorithm for this problem;
however, instead of formally deriving it, they present it directly.
While their approach can be seen as a special case of the algorithm
derived here, a formal derivation is important for understanding the
underlying principles and for establishing a rigorous foundation.

The following theorem explains why the algorithm of \citet{brita2025optimal}
is correct and how it can be systematically derived from (\ref{specification based on depth constrain}).
\begin{theorem}
	Dynamic programming algorithm for depth-constrained decision tree.\emph{
		\label{thm: DP for depth tree}Given a depth constraint $d:\mathbb{N}$
		and a data list $\mathit{xs}$. If the objective function $E$ can
		be defined according to (\ref{eq: objective function value scheme}),
		i.e., if Lemma \ref{Monotonicity} holds, then the solution obtained
		by}
	
	\begin{equation}
		\begin{aligned}\mathit{odt}_{\text{depth}} & :\mathbb{N}\times\mathcal{D}\to\mathit{DTree}\left(\mathcal{R},\mathcal{D}\right)\\
			\mathit{odt}_{\text{depth}} & \left(0,\mathit{xs}\right)=\mathit{DL}\left(\mathit{xs}\right)\\
			\mathit{odt}_{\text{depth}} & \left(1,\mathit{xs}\right)=\mathit{min}_{E}\left[DN\left(\mathit{DL}\left(\mathit{xs}^{+}\right),r,\mathit{DL}\left(\mathit{xs}^{-}\right)\right)\mid r\leftarrow\mathit{gen_{splits}}\left(\mathit{xs}\right)\right]\\
			\mathit{odt}_{\text{depth}} & \left(d,\mathit{xs}\right)=\mathit{min}_{E}\left[\mathit{DN}\left(\mathit{odt}_{\text{depth}}\left(d-1,\mathit{xs}^{+}\right),r,\mathit{odt}_{\text{depth}}\left(d-1,\mathit{xs}^{-}\right)\right)\mid,r\leftarrow\mathit{gen_{splits}}\left(\mathit{xs}\right)\right]
		\end{aligned}
		\label{eq: DP solution to odt_Depth}
	\end{equation}
\end{theorem}
is a solution of $\mathit{min}_{E}\circ\mathit{genDTDs}_{\text{depth}}\left(d\right)$.
In other words, we have 
\[
\mathit{odt}_{\text{depth}}\subseteq\mathit{min}_{E}\circ\mathit{genDTDs}_{\text{depth}}\left(d\right).
\]

\begin{proof}
	The base cases ($d=1$ and $d=0$) follow directly from the definition
	of an optimal decision tree:
	
	For $d=0$, no classification is performed; all data points are assigned
	to a single leaf.
	
	For $d=1$, the optimal decision tree consists of a single splitting
	rule. Hence, it is equivalent to selecting the optimal splitting rule
	from $\mathit{gen_{splits}}\left(\mathit{xs}\right)$
	
	For the recursive case, we have:
	\begin{align*}
		\mathit{odt}_{\text{depth}}\left(d,\mathit{xs}\right)= & \mathit{min}_{E}\circ\mathit{genDTDs}_{\text{depth}}\left(d\right)\\
		= & \{\text{definition of \ensuremath{\mathit{genDTDs}_{\text{depth}}}}\}\\
		& \mathit{min}_{E}\big[\mathit{DN}\left(u,r,v\right)\mid u\leftarrow\mathit{genDTDs}_{\text{depth}}\left(d-1,\mathit{xs}^{+}\right),\\
		&\quad\quad v\leftarrow\mathit{genDTDs}_{\text{depth}}\left(d-1,\mathit{xs}^{-}\right),r\leftarrow\mathit{gen_{splits}}\left(\mathit{xs}\right)\big]\\
		\subseteq & \text{ \{ monotonicity (\ref{Monotonicity}) \}}\\
		& \big[\mathit{DN} \big( \mathit{min}_{E}\left(\mathit{genDTDs}_{\text{depth}}\left(d-1,\mathit{xs}^{+}\right)\right),\\
		& \quad \quad r,\mathit{min}_{E}\left(\mathit{genDTDs}_{\text{depth}}\left(d-1,\mathit{xs}^{-}\right)\right) \big) \mid r\leftarrow\mathit{gen_{splits}}\left(\mathit{xs}\right)\big]\\
		= & \text{ \{ definition of \ensuremath{\mathit{odt}_{\text{depth}}} \}}\\
		& \left[\mathit{DN}\left(\mathit{odt}_{\text{depth}}\left(d-1,\mathit{xs}^{+}\right),r,\mathit{odt}_{\text{depth}}\left(d-1,\mathit{xs}^{-}\right)\right)\mid r\leftarrow\mathit{gen_{splits}}\left(\mathit{xs}\right)\right]
	\end{align*}
	
	This proved that $\mathit{odt}_{\text{depth}}\left(d,\mathit{xs}\right)$
	can be computed recursively using optimal subtrees for $\mathit{xs}^{+}$
	and $\mathit{xs}^{-}$ of depth $d-1$.
\end{proof}
Interestingly, for AODT and ODT over binary-feature data, \citet{brita2025optimal}
and \citet{demirovic2022murtree} observed that $\mathit{odt}_{\text{depth}}\left(2,\mathit{xs}\right)$
can be solved more efficiently than the classical recursive method.
Expressing $\mathit{odt}_{\text{depth}}\left(2,\mathit{xs}\right)$
in a more compact form and substituting it into the recursion \ref{eq: DP solution to odt_Depth},
may help convey their algorithmic insights more clearly.

\subsubsection{Non-existence of monotonicity}

Another advantage of a rigorous formulation is that it allows us not
only to prove the existence of dynamic programming algorithms but
also to demonstrate their non-existence when the objective function
is restricted to the form. The following theorem formalizes the non-existence
of monotonicity for $\mathit{sodt}_{\text{kperms}}$ and $\mathit{sodt}_{\text{rec}}$.
\begin{theorem}
	Non-existence of monotonicity.\emph{ For objective function defined
		according to (\ref{eq: objective function value scheme}), the one-step
		recursive process in $\mathit{sodt}_{\text{kperms}}$ and $\mathit{sodt}_{\text{rec}}$
		does not have the monotonicity.}
\end{theorem}
\begin{proof}
	It suffices to show that the following implication does not hold:
	\[
	E\left(t\right)\leq E\left(t^{\prime}\right)\nRightarrow E\left(\mathit{update}\left(r,t\right)\right)\leq E\left(\mathit{update}\left(r,t^{\prime}\right)\right),
	\]
	where $\mathit{update}\left(r,t\right)$ denotes adding the rule $r$
	to the corresponding leaf of tree $t$. Intuitively, this means that
	even if $E\left(t\right)\leq E\left(t^{\prime}\right)$ the new tree
	obtained by adding $r$ to $t$ may create a new tree with higher
	error than $E\left(t^{\prime}\right)$.
	
	Similarly, for $\mathit{sodt}_{\text{kperms}}$, the following implication
	fails:
	\[
	E\left(\mathit{rs}\right)\leq E\left(\mathit{rs}^{\prime}\right)\nRightarrow E\left(\mathit{insR}\left(r,\mathit{rs}\right)\right)\leq E\left(\mathit{insR}\left(r,\mathit{rs}^{\prime}\right)\right)
	\]
	where $\mathit{insR}:\left[\mathcal{R}\right]\to\left[\mathcal{R}\right]$
	is the relational version of the $\mathit{ins}$ operation \citet{he2025ROF},
	which inserts $r$ to a random position in a existing permutation
	of rules $\mathit{rs}$. The implication fails because the newly inserted
	rule $r$ can become the root of the resulting tree, which can change
	the error non-monotonically.
	
	Therefore, dynamic programming or greedy algorithms do not exist for
	$\mathit{sodt}_{\text{kperms}}$ and $\mathit{sodt}_{\text{vec}}$,
	when the objective function is defined according to (\ref{eq: objective function value scheme}).
	However, this does not imply that DP or greedy solutions are impossible
	for objectives defined in other form. For instance, if the objective
	function is incrementally definable—such as the tree size—then a greedy
	solution exist.
\end{proof}

\subsubsection{The use of memoization or caching}

Given a DP solution for the optimal decision tree problem and the
shared subtrees across different trees, it is natural to consider
applying memoization, a well-established technique for improving DP
efficiency.

Caching provides a practical alternative when a bottom-up recursion
or iterative DP is either unavailable or too complex to implement
efficiently. This approach has been widely adopted in ODT algorithms
over binary feature data. For instance, \citet{aglin2021pydl8} employ
caches over \emph{itemsets} (collections of Boolean features), while
\citet{demirovic2022murtree} note a trade-off between \emph{branch}
and \emph{dataset caching}, as introduced by \citet{nijssen2007mining}.

The correctness of these caching techniques in ODT studies remains
questionable. The derivations of $\mathit{odt}_{\text{depth}}$ and
$\mathit{odt}_{\text{size}}$ depend on Lemma \ref{Monotonicity},
which holds only when the roots of two trees, $r$ and $r^{\prime}$
are equivalent. Therefore, correct use of memoization requires storing
both the root and its corresponding subtrees, because optimality of
subtrees $u$ and $v$ over subregions $r^{\pm}$ implies optimality
of the full tree $DN\left(u,r,v\right)$. Thus, no trade-off exists
between branch and dataset caching: both subtrees (depth- or size-constrained)
and datasets must be stored. In this context, the caching strategy
in \citet{brita2025optimal} appears correct, as their ConTree algorithm
reuses cached solutions defined both by sub-dataset and subtree depth.

However, \citet{he2025odt} predicted that memoization becomes ineffective
when the number of possible splitting rules is large. This is because
each rule $r$ generates a distinct sub-dataset, resulting in few
identical optimal subtrees for a given sub-dataset. Denote the set
of possible splitting rules with respect to data list $\mathit{xs}$
as $\mathcal{S}_{\mathit{\text{split}}}\left(\mathinner{xs}\right)$,
to enable full memoization for constructing an optimal decision tree
of size $K$, one would need to store at least:
\begin{equation}
	O\left(\left|\mathcal{S}_{\mathit{\text{split}}}\left(\mathinner{xs}\right)\right|\times\sum_{k=0}^{K-1}\left|\mathcal{S}_{\text{size}}\left(k,\mathinner{xs}\right)\right|\right)
\end{equation}
Similarly, for depth-constrained trees, the combinatorial complexity
becomes:
\begin{equation}
	O\left(\left|\mathcal{S}_{\mathit{\text{split}}}\left(\mathinner{xs}\right)\right|\times\sum_{i=0}^{d-1}\left|\mathcal{S}_{\text{depth}}\left(i,\mathinner{xs}\right)\right|\right)
\end{equation}
This results in formidable combinatorial complexity for most ODT problems.
For example, in the AODT problem, the number of possible rules is
$\mathcal{S}_{\mathcal{H}^{0}}=O\left(N\times D\right)$. For hyperplane
ODT problems, the number of possible splits grows to $\mathcal{S}_{\mathcal{H}}=O\left(N^{^{D}}\right)$
(see Section \ref{sec:Geometric-foundation}). Consequently, the caching
approach in \citet{brita2025optimal}, which stores only a few thousand
optimal sub-trees, has a very low probability of cache hits. Our tests
on multiple data lists confirm this: in extreme cases, their algorithm
examined millions of decision trees without achieving a single cache
hit. This suggests that the efficiency claimed for their ConTree algorithm
due to caching is overstated; most gains likely arise from factors
other than caching.

\subsection{Extension to decision tree problem over binary feature data (non-proper
	decision tree)\label{subsec:Extension-to-decision}}

\subsubsection{Definition of the search space}

A key advantage of adapting a general axiomatic framework for defining
decision trees is that it provides an unambiguous method for classifying
decision tree problems. The optimal decision tree problem over binary
feature data (ODT-BF) aims to construct a tree that encodes a mapping
over binary feature datasets, i.e., $t:\left\{ -1,1\right\} ^{D}\to L$,
where $L$ denotes the set of labels.

The ODT-BF problem is perhaps the most extensively studied case of
applying combinatorial methods \citet{hu2019optimal,lin2020generalized,zhang2023optimal,demirovic2022murtree,aglin2021pydl8,aglin2020learning,nijssen2007mining,nijssen2010optimal}.
The tree structure of ODT-BF is defined in the same way as proper
decision trees; in other words, ODT-BF also satisfies the structural
constraints of decision trees (Axioms 1–3) and is therefore a valid
decision tree problem. However, this problem has much simpler combinatorics,
and indeed \citet{he2025ROF} showed that its complexity is independent
of the input data size, and empirical results from \citet{hu2019optimal}
further indicate that, when the feature dimension is fixed, the algorithm
scales linearly with data size.
\paragraph{Characterization of splitting rules}

Unlike decision tree problems over continuous datasets, where splitting
rules can often be interpreted as geometric partitions—such as a hyperplane
dividing the space into two regions—ODT-BF cannot naturally be viewed
this way. Some researchers classify ODT-BF within the class of axis-parallel
hyperplanes, but this characterization is misleading, as it obscures
the geometric meaning of axis-parallel splits. Binary feature data
correspond to vertices of a hypercube, and splits cannot be defined
by midpoints between vertices, since the data are discrete rather
than continuous.

PA more natural interpretation is to view each splitting rule in ODT-BF
as a logical question. Specifically, given a data point $x\in\mathbb{R}^{D}$,
each splitting rule in ODT-BF problem corresponds to determining whether
the $i$-th feature of a data point $x\in\mathbb{R}^{D}$ is 0 or
1—equivalently, whether $x$ contains feature $i$?

An immediate consequence of this characterization is that the number
of splitting rules is independent of the dataset size, as a hypercube
in $\mathbb{R}^{D}$ has at most $2^{D}$ vertexes. Indeed, in the
ODT-BF problem, many data points are duplicates and occupy the same
vertices of the cube, and there are many bounding techniques are developed
based on this observation \citep{lin2020generalized,zhang2023optimal,hu2019optimal}.

\paragraph{Ancestral constraints and splitting function}

Having characterized the splitting rules in ODT-BF as simple logical
questions, we now analyze the ancestral constraints between pairs
of rules. Suppose a rule $r_{i}$ is an ancestor of another rule $r_{j}$.
The question of whether $x$ contains feature $i$ is independent
of whether it contains feature $j$. This means that fixing an ancestor
rule $r_{i}$ does not influence whether $r_{j}$ should go to the
left or right subtree of $r_{i}$. In other words, if $r_{i}$ is
the ancestor of $r_{j}$, both $r_{i}\swarrow r_{j}$ and $r_{i}\searrow r_{j}$
are valid. Consequently, Axiom 4 of the proper decision tree no longer
holds, and the ODT-BF problem cannot be solved using the $\mathit{odt}_{K}$
program.

As a result, ODT-BF satisfies only the structural constraints of decision
trees, with no ancestral constraints. It therefore reduces to an ordinary
labeled binary tree problem with no explicit constraints, yielding
a complexity of $K!\times\mathit{Catalan}\left(K!\right)$ for an
input of size $K$, since there are $\mathit{Catalan}\left(K!\right)$
possible tree shapes and $K!$ possible labelings.

Without ancestral constraints, the split function for ODT-BF is straightforward:
once a root $r$ is fixed, how should the two sublists $\mathit{rs}^{+}$
and $\mathit{rs}^{-}$ for its subtrees be determined from $\mathit{rs}/r$?
The answer is straightforward: \textbf{any subset of} $\mathit{rs}/r$
\textbf{is valid}! The following function produces all possible sublists
$\mathit{xs}^{\prime}$ along with their complementary subsets $\mathit{xs}/\mathit{xs}^{\prime}$
as a pair.
\[
\begin{aligned}\mathit{subsPair} & \left(\left[\:\right],\mathit{ys}\right)=\left[\left(\left[\:\right],\mathit{ys}\right)\right]\\
	\mathit{subsPair} & \left(x:\mathit{xs},\mathit{ys}\right)=\mathit{subsPair}\left(\mathit{xs},\mathit{ys}\right)\cup\mathit{mapL}\left(f,\mathit{subsPair}\left(\mathit{xs},\mathit{ys}\right)\right)
\end{aligned}
\]
where $f\left(a,\mathit{as},\mathit{bs}\right)=\left(a:\mathit{as},\mathit{bs}/a\right)$.

For instance, running $\mathit{subsPair}\left(\left[1,2,3\right],\left[1,2,3\right]\right)$
returns\\
 $\left[\left(\left[\:\right],\left[1,2,3\right]\right),\left(\left[3\right],\left[1,2\right]\right),\left(\left[2\right],\left[1,3\right]\right),\left(\left[2,3\right],\left[1\right]\right),\left(\left[1\right],\left[2,3\right]\right),\left(\left[1,3\right],\left[2\right]\right),\left(\left[1,2\right],\left[3\right]\right),\left(\left[1,2,3\right],\left[\:\right]\right)\right]$,
the $2^{3}$ possible sublists of $\left[1,2,3\right]$.
\begin{definition}
	The split function for the ODT-BF problem is defined as
	\begin{equation}
		\mathit{splits}_{\text{BF}}\left(\mathit{rs}\right)=\mathit{subsPair}\left(\mathit{rs},\mathit{rs}\right)
	\end{equation}
\end{definition}

\paragraph{Definition of the search space}

With the help of splitting function, \citet{he2025odt} proposed the
following generator to exhaustively generate the search space for
decision trees over binary feature data, as formalized below.
\begin{definition}
	The generator for the search space of the ODT-BF problem over the
	binary tree datatype is defined as
	\begin{equation}
		\begin{aligned}\mathit{genDTBFs} & :\mathcal{D}\times\left[\mathcal{R}\right]\to\left[\mathit{DTree}\left(\mathcal{R},\mathcal{D}\right)\right]\\
			\mathit{genDTBFs} & \left(\mathit{xs},\left[\:\right]\right)=\left[\:\right]\\
			\mathit{genDTBFs} & \left(\mathit{xs},\left[r\right]\right)=\left[\mathit{DN}\left(\mathit{DL}\left(xs^{+}\right),r,\mathit{DL}\left(xs^{-}\right)\right)\right]\\
			\mathit{genDTBFs} & \left(\mathit{xs},\mathit{rs}\right)=\big[\mathit{DN}\left(\mathit{mapD}{}_{\cap_{\mathit{xs}^{+}}}\left(u\right),r_{i},\mathit{mapD}{}_{\cap_{\mathit{xs}^{-}}}\left(v\right)\right)\mid r_{i}\leftarrow\mathit{rs},\\
			& \left(\mathit{rs}^{+},\mathit{rs}^{-}\right)\leftarrow\mathit{splits}_{\text{BF}}\left(\mathit{rs}/r_{i}\right),u\leftarrow\mathit{\mathit{genDTBFs}}\left(\mathit{rs}^{+}\right),v\leftarrow\mathit{\mathit{genDTBFs}}\left(\mathit{rs}^{-}\right)\big].
		\end{aligned}
		\label{eq: generator of ODT-BF}
	\end{equation}
\end{definition}
The combinatorial analysis in \citet{he2025odt} showed that $\mathit{genDTBFs}$
will generate $f_{\text{BF}}\left(K\right)=K!\times\mathit{Catalan}\left(K\right)$,
align with our interpretation of labeled binary tree's complexity.

Consequently the ODT-BF problem can be defined by following brute-force
program
\begin{equation}
	\mathit{odt}_{\text{BF}}\left(\mathit{xs}\right)=\mathit{min}_{E}\circ\mathit{genDTBFs}\left(\mathit{xs}\right)\label{eq: specification of odt-bf}
\end{equation}

\subsubsection{optimal decision tree algorithm over binary feature data}

As a result, following an exactly same derivation process as the Theorem
\ref{spec: ODT_depth} by changing $\mathit{splits}$ as $\mathit{splits}_{\text{BF}}$,
we can derive the following dynamic programming algorithm for solving
the \ref{eq: specification of odt-bf} problem.
\begin{theorem}
	Dynamic programming algorithm for size-constrained decision tree.\emph{
		\label{thm: DP for ODT-BF} Given a list of $K$ rules $\mathit{rs}_{K}:\left[\mathcal{R}\right]$.
		If the objective function $E$ can be defined according to (\ref{eq: objective function value scheme})—that
		is, if Lemma Lemma \ref{Monotonicity} holds—then the solution obtained
		by}
	\begin{equation}
		\begin{aligned}\mathit{odt}_{\text{BF}} & :\mathcal{D}\times\left[\mathcal{R}\right]\to\mathit{DTree}\left(\mathcal{R},\mathcal{D}\right)\\
			\mathit{odt}_{\text{BF}} & \left(\mathit{xs},\left[\;\right]\right)=\left[\mathit{DL}\left(\mathit{xs}\right)\right]\\
			\mathit{odt}_{\text{BF}} & \left(\mathit{xs},\left[r\right]\right)=\left[\mathit{DN}\left(\mathit{DL}\left(xs^{+}\right),r,\mathit{DL}\left(xs^{-}\right)\right)\right]\\
			\mathit{odt}_{\text{BF}} & \left(\mathit{xs},\mathit{rs}\right)=\mathit{min}{}_{E}\big[\mathit{DN}\left(\mathit{odt}_{\text{BF}}\left(\mathit{xs}^{+},\mathit{rs}^{+}\right),r,\mathit{odt}_{\text{BF}}\left(\mathit{xs}^{-},\mathit{rs}^{-}\right)\right)\mid\\
			& r \leftarrow \mathit{rs}, \left(\mathit{rs}^{+},r,\mathit{rs}^{-}\right)\leftarrow\mathit{splits}_{\text{BF}}\left(\mathit{rs}\right)\big].
		\end{aligned}
	\end{equation}
	\emph{provides a solution to the brute-force specification (\ref{eq: specification of odt-bf}).}
\end{theorem}

\subsection{Further acceleration—Thinning and filtering \label{subsec: thinning and filtering}}

\subsubsection{Filtering}

In machine learning, a common way to prevent overfitting is to require
each leaf node to contain at least $N_{\min}$ data points, avoiding
situations where a leaf has only a very small number of samples. A
straightforward way to enforce this constraint is to introduce a filtering
process:
\begin{align}
	\mathit{sodt}_{\text{filt}} & =\mathit{min}_{E}\circ\mathit{filter}_{p}\circ\mathit{genDTs}\label{eq: sodt_filt specification}\\
	\mathit{odt}_{\text{filtDep}} & =\mathit{min}_{E}\circ\mathit{filter}_{p}\circ\mathit{genDTDs}_{\text{depth}}\label{eq: odt_depth filt specification}
\end{align}
where $p:\mathit{DTree}\left(\mathcal{R},\mathcal{D}\right)\to\mathit{Bool}$
is a Boolean-valued predicate.

However, this direct approach is inefficient, as $\mathit{genDTs}$
and $\mathit{genDTDs}_{\text{depth}}$ generate a huge number of trees,
making post-generation filtering costly. When the predicate $p$ satisfies
certain conditions, filtering can be fused with tree generation, avoiding
unnecessary computation. Specifically, the following condition that
help identify how filtering can be fused into the generation process.
\begin{definition}
	\emph{Prefix-closed predicate}. A predicate $p:\mathit{DTree}\left(\mathcal{R},\mathcal{D}\right)\to\mathit{Bool}$
	is called\emph{ prefix-closed} if for any tree $\mathit{DN}\left(u,r,v\right)$,
	the following implication holds: 
	\begin{equation}
		p\left(\mathit{DN}\left(u,r,v\right)\right)\implies p\left(u\right)\wedge p\left(v\right).\label{eq: prefix-closed definition}
	\end{equation}
	For example, the predicate that checks whether a tree contains a leaf
	with at least $N_{\text{min}}$ data points is prefix-closed, because
	if a tree $\mathit{DN}\left(u,r,v\right)$ violates this condition,
	then its subtrees then $u$ and $v$ also violate it. Similarly, a
	tree-size predicate, which measures the number of splitting rules,
	is prefix-closed: $\mathit{DN}\left(u,r,v\right)$ has fewer than
	$K$ rules, then both $u$ and $v$ also have fewer than $K$ rules.
	The same holds for depth constraints: if a tree has depth less than
	$d$, its subtrees also satisfy this bound.
\end{definition}
Prefix-closed predicates provide an effective mechanism to fuse the
filtering operation $\mathit{filter}_{p}$ with the generator $\mathit{gen}$.
In particular, the following filter fusion theorem enables the integration
of depth- or size-constrained conditions into leaf- or size-constrained
ODT problems.
\begin{theorem}
	Filtering fusion through prefix-predicate. \emph{If a predicate $p$
		is prefix-closed, Then, the solution obtained using $\mathit{sodt}_{\text{filtRec}}$,
		defined as
		\begin{equation}
			\begin{aligned}\mathit{sodt}_{\text{filtRec}} & :\left[\mathcal{R}\right]\times\mathcal{D}\to\mathit{DTree}\left(\mathcal{R},\mathcal{D}\right)\\
				\mathit{sodt}_{\text{filtRec}} & \left(\left[\;\right],\mathit{xs}\right)=\left[\mathit{DL}\left(\mathit{xs}\right)\right]\\
				\mathit{sodt}_{\text{filtRec}} & \left(\left[r\right],\mathit{xs}\right)=\left[\mathit{DN}\left(\mathit{DL}\left(xs^{+}\right),r,\mathit{DL}\left(xs^{-}\right)\right)\right]\\
				\mathit{sodt}_{\text{filtRec}} & \left(\mathit{rs},\mathit{xs}\right)=\mathit{min}{}_{E}\big(\mathit{filter}_{p}\big[\mathit{DN}\left(\mathit{sodt}_{\text{filtRec}}\left(\mathit{rs}^{+},\mathit{xs}^{+}\right),r,\mathit{sodt}_{\text{filtRec}}\left(\mathit{rs}^{-},\mathit{xs}^{-}\right)\right)\mid\\
				& \quad \quad \left(\mathit{rs}^{+},r,\mathit{rs}^{-}\right)\leftarrow\mathit{splits}\left(\mathit{rs}\right)\big]\big).
			\end{aligned}
		\end{equation}
		provides a solution to $\mathit{min}_{E}\circ\mathit{filter}_{p}\circ\mathit{genDTs}_{\text{rec}}$.
		Similarly, the solution obtained using $\mathit{odt}_{\text{filtDep}}$,
		defined as}
	
	\emph{
		\begin{equation}
			\begin{aligned}\mathit{odt}_{\text{filtDep}} & :\mathbb{N}\times\mathcal{D}\to\mathit{DTree}\left(\mathcal{R},\mathcal{D}\right)\\
				\mathit{odt}_{\text{filtDep}} & \left(0,\mathit{xs}\right)=\mathit{DL}\left(\mathit{xs}\right)\\
				\mathit{odt}_{\text{filtDep}} & \left(1,\mathit{xs}\right)=\mathit{min}_{E}\left[DN\left(\mathit{DL}\left(\mathit{xs}^{+}\right),r,\mathit{DL}\left(\mathit{xs}^{-}\right)\right)\mid r\leftarrow\mathit{gen_{splits}}\left(\mathit{xs}\right)\right]\\
				\mathit{odt}_{\text{filtDep}} & \left(d,\mathit{xs}\right)=\mathit{min}_{E}\big(\mathit{filter}_{p}\big[\mathit{DN}\left(\mathit{odt}_{\text{filtDep}}\left(d-1,\mathit{xs}^{+}\right),r,\mathit{odt}_{\text{filtDep}}\left(d-1,\mathit{xs}^{-}\right)\right)\mid\\
				& \quad \quad r\leftarrow\mathit{gen_{splits}}\left(\mathit{xs}\right)\big]\big)
			\end{aligned}
			\label{eq: odt_Depth with filter}
		\end{equation}
		provides a solution to (\ref{eq: odt_depth filt specification}).\label{thm: filter fusion theorem}}
\end{theorem}
\begin{proof}
	The proof follows by substituting the fused generator and following
	the procedure of Theorems (\ref{thm: DP for leaf tree}) and (\ref{thm: DP for depth tree}).
	For more general results and formal proofs, see \citet{he2025ROF}.
\end{proof}
Theorem \ref{thm: filter fusion theorem} provides an efficient and
convenient method to incorporate any prefix-closed constraint into
the ODT algorithm. Any partial solution that does not satisfy the
prefix-closed constraint can be filtered out before extending it into
a complete tree, thus reducing unnecessary computation while ensuring
constraints are incorporated.

The derivation for $\mathit{sodt}_{\text{vec}}$ and $\mathit{sodt}_{\text{kperms}}$
similar to Theorem \ref{thm: filter fusion theorem}, but requires
a different prefix-closed predicate definition, as the recursive structure
of this two algorithms are sequential. See \citet{he2025ROF} for
more general results.

\subsubsection{Thinning algorithm (bounding technique)}

\paragraph{Thin introduction}

The \emph{thinning algorithm} is almost identical to the exploitation
of dominance relations in the algorithm design literature \citep{ibaraki1977power,bird1996algebra}.
Both thinning and classical dominance relations aim to improve the
time complexity of naive dynamic programming algorithms \citep{bird1996algebra}.
The key difference is that thinning considers not only the theoretical
effectiveness of bounds but also their practical implementation. This
aspect is often overlooked in studies of branch-and-bound algorithms,
where researchers focus on the existence alone of a bounding relation
but neglect its efficiency. As noted in subSection \ref{subsec: 2.1},
a bounding technique must be proven effective: the computational time
it saves must outweigh the time required to execute it. If a bounding
check is computationally expensive yet reduces the solution space
only marginally, its application becomes ineffective.

The thinning technique leverages the fundamental observation that
certain partial configurations are superior to others. Extending non-optimal
partial configurations is thus a waste of computational resources.
Formally, we define a function $\mathit{thin}:\left(\mathcal{A}\times\mathcal{A}\to\mathit{Bool}\right)\times\left[\mathcal{A}\right]\to\left[\mathcal{A}\right]$
which receives a relation $R:\mathcal{A}\times\mathcal{A}\to\mathit{Bool}$,
referred to as a dominance relation, and a sequence $\mathit{xs}$.
The output is a \emph{subsequence} $\mathit{ys}\subseteq\mathit{xs}$,
satisfying:

\begin{equation}
	\forall x\in\mathit{xs}:\exists y\in\mathit{ys}:R\left(y,x\right)\label{eq: universal property of thinning}
\end{equation}
In other words, for every element in the original sequence, there
exists a dominating element in the thinned subsequence, ensuring that
only potentially optimal configurations are retained for further computation.
Properties (1) and (2) are known as the \emph{universal property}
of the thinning algorithm \citep{bird1996algebra}.

Thinning is similar to, but distinct from, $\mathit{min}_{E}$. Indeed,
the $\mathit{min}_{E}$ function can be understood as a special case
of thinning with respect to a \emph{total order} defined by the objective
function $E$. In contrast, thinning is based on a \emph{preorder}
$R$, because \emph{transitivity} is required, and total orders are
too restrictive for the purposes of thinning \citep{he2025ROF,bird2020algorithm}.
In a preorder relation, some configurations may be incomparable. Consequently,
$\mathit{thin}_{R}:\left[\mathcal{A}\right]\to\left[\mathcal{A}\right]$
receives a list and returns a list, whereas $\mathit{min}_{E}:\left[\mathcal{A}\right]\to\mathcal{A}$
always returns a single element.

The thinning algorithm must satisfy the following laws:
\begin{align}
	\mathit{min}_{E} & =\mathit{min}_{E}\circ\mathit{thin}_{R_{2}}\label{eq:thin-intro}\\
	\mathit{mapL}_{f}\circ\mathit{thin}_{R_{1}} & \subseteq\mathit{thin}_{R_{2}}\circ\mathit{mapL}_{f}\label{eq: thin-map law}
\end{align}
such that $R_{1}\left(a,b\right)\implies R_{2}\left(a,b\right)$.
The first law (\ref{eq:thin-intro}) is called \emph{thin introduction}.
It states that applying thinning does not change the optimal solution:
the optimal solution after thinning is the same as the original optimal
solution. The second law (\ref{eq: thin-map law}), called \emph{thin-map},
governs the interaction between a map operation and thinning: mapping
over a thinned sequence produces a result contained within the thinning
of the mapped sequence.

\paragraph{Monotonicity and thinning}

Similar to monotonicity used in designing efficient combinatorial
optimization algorithms, it also plays a crucial role in identifying
valid dominance relations that enable thinning fusion. Intuitively,
monotonicity ensures that if a partial configuration $a$ dominates
another configuration $b$, then extending both configurations in
the same way preserves the dominance relation. Without this property,
thinning could discard potentially optimal solutions.
\begin{definition}
	\emph{Monotonicity in thinning}. Let $f:\mathcal{A}\to\mathcal{A}$
	be a \emph{one-step update relation} in a recursion (sequential decision
	process), We say that $f$ is monotonic with respect to a relation
	$R:\mathcal{A}\times\mathcal{A}\to\mathit{Bool}$ if 
	\begin{equation}
		R\left(a,b\right)\Longrightarrow R\left(f\left(a\right),f\left(b\right)\right)
	\end{equation}
	In words, if configuration $a$ dominates configuration $b$, then
	after applying the same update $f$ to both, the dominance relation
	still holds.
\end{definition}
For instance, in $\mathit{sodt}_{\text{vec}}$, the semantic interpretation
of monotonicity in thinning is expressed as: $R\left(t,t^{\prime}\right)\Longrightarrow R\left(\mathit{update}\left(r,t\right),\mathit{update}\left(r,t^{\prime}\right)\right)$.
Similarly, in $\mathit{sodt}_{\text{kperms}}$, we have: $R\left(\mathit{rs},\mathit{rs}^{\prime}\right)\Longrightarrow R\left(\mathit{insR}\left(r,\mathit{rs}\right),\mathit{insR}\left(r,\mathit{rs}^{\prime}\right)\right)$.
Likewise, in $\mathit{sodt}_{\text{rec}}$, the semantic interpretation
of monotonicity in thinning can be expressed as \\
$R\left(u,u^{\prime}\right)\wedge R\left(v,v^{\prime}\right)\subseteq R\left(\mathit{DN}\left(u,r,v\right),\mathit{DN}\left(u^{\prime},r,v^{\prime}\right)\right)$.

\paragraph{Finite dominance relation and global upper bound}

There are many ways to define a thinning algorithm \citet{he2025ROF},
and its effectiveness depends on the context. For our purposes—the
study of optimal decision tree problems—the dominance relation allows
us to implement $\mathit{thin}$ simply using the $\mathit{filter}$
function. Therefore, we will not discuss the general construction
of thinning, but instead focus on formalizing the specific dominance
relations studied in ODT algorithms for machine learning research.

\citet{he2025ROF} showed that dominance relations commonly used in
ML can be abstracted into two types: \emph{finite dominance relation}
(FDR) and\emph{ global upper bound} (GUB). Both relations rely on
two auxiliary functions defined as follows.
\begin{definition}
	\emph{Pessimistic upper bound} and \emph{optimistic lower bound}.
	Given a configuration $a:\mathcal{A}$, and let $f:\mathcal{A}\to\mathcal{A}$
	be a \emph{one-step update relation}\footnote{In the ODT problem, for $\mathit{sodt}_{\text{rec}}$ and $\mathit{odt}_{\text{depth}}$,
		$f$ corresponds to the tree constructor $\mathit{DN}$ with the same
		root $r$ and left/right subtrees $t$ but different right/left subtrees
		$x$ and $y$. In $\mathit{sodt}_{\text{vec}}$ and $\mathit{sodt}_{\text{kperms}}$,
		$f$ is a parameterized function, either $\mathit{update}\left(r\right)$
		or $\mathit{insR}\left(r\right)$ for adding a new splitting rule
		$r$ to a existing tree.} in a recursion. Assume smaller values are better: a function $\mathit{ub}$
	is called \emph{pessimistic} \emph{upper bound }if $\mathit{ub}\left(f^{n}\left(a\right)\right)\leq\mathit{ub}\left(a\right)$.
	A function $\mathit{lb}:\mathcal{A}\times\mathcal{D}\to\mathbb{R}$
	is called \emph{optimistic lower bound }if $\mathit{lb}\left(f^{n}\left(a\right)\right)\geq\mathit{lb}\left(a\right)$.
\end{definition}
Intuitively, these bounds capture the idea that as more information
is accumulated, the bounds become tighter, assume small is better:
\begin{itemize}
	\item The pessimistic upper bound provides a ``worst-case'' estimate of
	$a$, so any update $f^{n}\left(a\right)$ should yield a tighter
	(smaller) upper bound.
	\item The optimistic lower bound provides a ``best-case'' estimate of $a$,
	so any update $f^{n}\left(a\right)$ should yield a looser (larger)
	lower bound.
\end{itemize}
Many existing studies on ODT problems focus on developing effective
optimistic lower bounds or pessimistic upper bounds. \citet{he2025ROF}
demonstrated that the following two generic dominance relations, derived
using $\mathit{ub}$ and $\mathit{lb}$, satisfy the thinning laws.
These abstract the dominance relations used in most optimal decision
tree algorithms for machine learning.
\begin{theorem}
	Finite dominance relations and global upper bound.\emph{ Given two
		partial configuration $a,b:\mathcal{A}$, and a global upper-bound
		constant $\mathit{ub}^{*}:\mathcal{A}$, defines}
\end{theorem}
\begin{enumerate}
	\item \textbf{Finite dominance relation}:\emph{ 
		\[
		\mathit{fdom}\left(a,b\right)=\mathit{ub}\left(a\right)\leq\mathit{lb}\left(b\right).
		\]
	}
	\item \textbf{Global upper bounds dominance} \textbf{relation}: 
	\[
	\mathit{gubdom}\left(a\right)=\mathit{ub}^{*}\leq\mathit{lb}\left(a\right).
	\]
\end{enumerate}
Then both $\mathit{fdom}$ and $\mathit{gubdom}$ satisfies the monotonicity.
\begin{proof}
	To prove the dominance valid, we nee to verify the thinning algorithm
	defined using $\mathit{fdom}$ satisfies the universal property, which
	is equivalent to verifying the following \emph{monotonicity} \citep{bird1996algebra}:
	\[
	\mathit{ub}\left(a\right)\leq\mathit{lb}\left(b\right)\Longrightarrow\mathit{ub}\left(f^{n}\left(a\right)\right)\leq\mathit{lb}\left(f^{n}\left(b\right)\right)
	\]
	The implication holds because $\mathit{ub}\left(f^{n}\left(a\right)\right)\leq\mathit{ub}\left(a\right)\leq\mathit{lb}\left(b\right)\leq\mathit{lb}\left(f^{n}\left(b\right)\right)$.
	Similarly, to prove $\mathit{gubdom}$, we need to verify 
	\[
	\mathit{ub}^{*}\leq\mathit{lb}\left(a\right)\Longrightarrow\mathit{ub}^{*}\leq\mathit{lb}\left(f^{n}\left(a\right)\right)
	\]
	which holds because $\mathit{ub}^{*}\leq\mathit{lb}\left(a\right)\leq\mathit{lb}\left(f^{n}\left(a\right)\right)$.
\end{proof}
For the global upper bound dominance relation, the thinning function
$\mathit{thin}:\left[\mathcal{A}\right]\to\left[\mathcal{A}\right]$
can be efficiently implemented using a simple $\mathit{filter}$ \citet{he2025ROF}.

\paragraph{Examples of lower bound in literature}

We now introduce two useful and generic lower bounds proposed by \citet{demirovic2022murtree}
and \citet{zhang2023optimal} for classification and regression problems,
respectively.

One particularly interesting and widely applicable bounding technique
was introduced by \citet{demirovic2022murtree,lin2020generalized,brita2025optimal}
for classification problems. This technique is based on the following
observation, although it relies on an assumption that has not yet
been formally proven.
\begin{definition}
	\emph{Similarity lower bound}. Let $\mathit{xs}_{1}$ and $\mathit{xs}_{2}$
	denotes two different sets, and $\mathit{xs}_{\text{1}}-\mathit{xs}_{2}$
	denote the list of elements in $\mathit{xs}_{1}$ but not in $\mathit{xs}_{2}$,
	let $E_{\text{0-1}}\left(\mathit{odt}_{\text{depth}}\left(d,\mathit{xs}_{1}\right),\mathit{xs}_{2}\right)$
	denote the number of misclassifications when the optimal decision
	tree of fixed depth $d$ trained on data list $\mathit{xs}_{1}$,
	is evaluated on data list $\mathit{xs}_{2}$, (For readability, we
	write $\mathit{odt}_{\text{depth}}\left(\mathit{xs}\right)$ with
	$d$ fixed by default.)
\end{definition}
The \emph{similarity lower bound} is based on the following \emph{observation}:
\begin{equation}
	E_{\text{0-1}}\left(\mathit{odt}_{\text{depth}}\left(\mathit{xs}_{\text{2}}\right),\mathit{xs}_{\text{2}}\right)\geq E_{\text{0-1}}\left(\mathit{odt}_{\text{depth}}\left(\mathit{xs}_{1}\cap\mathit{xs}_{\text{2}}\right),\mathit{xs}_{1}\cap\mathit{xs}_{\text{2}}\right)\geq E_{\text{0-1}}\left(\mathit{odt}_{\text{depth}}\left(\mathit{xs}_{\text{1}}\right)\right)-\left|\mathit{xs}_{1}-\mathit{xs}_{2}\right|\label{eq: contree bound}
\end{equation}
The last inequality states that removing samples $\mathit{xs}_{1}-\mathit{xs}_{2}$
from data list $\mathit{xs}_{1}$ can reduce the number of misclassifications
by at most $\left|\mathit{xs}_{1}-\mathit{xs}_{2}\right|$. Although
looks intuitively obvious. In the worst case, assume all removed samples
were misclassified. Then: 
\begin{equation}
	E_{\text{0-1}}\left(\mathit{odt}_{\text{depth}}\left(\mathit{xs}_{1}\cap\mathit{xs}_{\text{2}}\right),\mathit{xs}_{1}\cap\mathit{xs}_{\text{2}}\right)=E_{\text{0-1}}\left(\mathit{odt}_{\text{depth}}\left(\mathit{xs}_{\text{1}}\right),\mathit{xs}_{\text{1}}\right)-\left|\mathit{xs}_{1}-\mathit{xs}_{2}\right|
\end{equation}
This equality implicitly suggests that the optimal decision tree trained
on the larger data list $\mathit{xs}_{\text{1}}$ is also optimal
on the smaller data list $\mathit{xs}_{1}\cap\mathit{xs}_{\text{2}}$.
If this were not true, then there would exist another tree $t_{\text{opt}}^{\prime}$
such that 
\begin{equation}
	E_{\text{0-1}}\left(t_{\text{opt}}^{\prime},\mathit{xs}_{1}\cap\mathit{xs}_{\text{2}}\right)<E_{\text{0-1}}\left(\mathit{odt}_{\text{depth}}\left(\mathit{xs}_{\text{1}}\right),\mathit{xs}_{\text{1}}\right)-\left|\mathit{xs}_{1}-\mathit{xs}_{2}\right|,
\end{equation}
contradicting the bound. This assumption—that the optimal tree for
a larger data list remains optimal for its subset—is not formally
proven and is somewhat counterintuitive. Clarifying or proving this
property is necessary to ensure the validity of the similarity lower
bound.

\paragraph{Example 2: $K$-means lower bound}

Another interesting lower bound was introduced in the context of the
optimal regression tree problem by \citet{zhang2023optimal}. It captures
the key insight that the regression error within a leaf cannot be
lower than the optimal solution to the one-dimensional $K$-means
clustering problem over the corresponding labels.
\begin{theorem}
	$K$-means lower bound\emph{. Given a set of rules $\mathit{rs}_{K}$
		rules, a list of data points $\mathit{xs}$, and a list of corresponding
		labels $\mathit{ys}:\left[\mathbb{R}\right]$ associated with $\mathit{xs}$,
		let the objective function $E_{2}$ denote the ordinary squared loss.
		Define $\text{kmeans-1D}\left(k,\mathit{ys}\right)$ as the optimal
		cost of performing $k$-means clustering over a one-dimensional data
		list $\mathit{ys}:\left[\mathbb{R}\right]$. Then, we have the following
		lower bound:
		\begin{equation}
			E_{2}\left(\mathit{sodt}\left(\mathit{xs},\mathit{rs}_{K}\right),\mathit{ys}\right)\geq\text{kmeans-1D}\left(K+1,\mathit{ys}\right)
		\end{equation}
	}
\end{theorem}
\begin{proof}
	The result follows directly, as a decision tree can be viewed as a
	method of clustering a one-dimensional data list—a label vector. Its
	clustering performance cannot surpass that of exact 1D $K$-means
	clustering, which is optimal with respect to squared loss. Consequently,
	any other clustering over $\mathit{ys}$, including that induced by
	the decision tree, will exhibit a greater squared loss.
\end{proof}

\paragraph{Fusion for thinning algorithm}

We now derive the program of fusing thinning algorithm in the recursion
of $\mathit{odt}_{\text{depth}}$ when the dominance relation is defined
as $\mathit{fdom}$ or $\mathit{gubdom}$. The thinning algorithm
for $\mathit{sodt}_{\text{rec}}$ can be derived in a similar manner.
\begin{lemma}
	Thinning fusion by finite dominance relation\emph{. Given a pessimistic}
	\emph{upper bound $\mathit{ub}$ and an optimistic lower bound $\mathit{lb}$,
		which defines a }finite dominance relation\emph{ $\mathit{fdom}_{1}\left(t,\mathit{t}^{\prime}\right)=\mathit{ub}\left(d,\mathit{xs}\right)\leq\mathit{lb}\left(d,\mathit{xs}\right)$
		and $\mathit{fdom}_{2}\left(t,\mathit{t}^{\prime}\right)=\mathit{ub}\left(\mathit{odt}_{\text{depth}}\left(d,\mathit{xs}\right)\right)\leq\mathit{lb}\left(\mathit{odt}_{\text{depth}}\left(d,\mathit{xs}\right)\right)$,
		then the solution of 
		\begin{equation}
			\begin{aligned}\mathit{odt}_{\text{thindep}} & :\mathbb{N}\times\mathcal{D}\to\mathit{DTree}\left(\mathcal{R},\mathcal{D}\right)\\
				\mathit{odt}_{\text{thindep}} & \left(0,\mathit{xs}\right)=\mathit{DL}\left(\mathit{xs}\right)\\
				\mathit{odt}_{\text{thindep}} & \left(1,\mathit{xs}\right)=\mathit{min}_{E}\left[DN\left(\mathit{DL}\left(\mathit{xs}^{+}\right),r,\mathit{DL}\left(\mathit{xs}^{-}\right)\right)\mid r\leftarrow\mathit{gen}_{\mathit{splits}}\left(\mathit{xs}\right)\right]\\
				\mathit{odt}_{\text{thindep}} & \left(d,\mathit{xs}\right)=\mathit{min}_{E}\left(\mathit{mapL}_{f}\left(\mathit{thin}_{\mathit{fdom}_{1}}\left(\left[\left(r,d,\mathit{xs}\right)\mid r\leftarrow\mathit{gen_{splits}}\left(\mathit{xs}\right)\right]\right)\right)\right)
			\end{aligned}
			\label{eq: odt_Depth with thinning}
		\end{equation}
	}\textup{ where }\emph{$\ensuremath{f\left(r,d,\mathit{xs}\right)=\mathit{DN}\left(\mathit{odt}_{\text{depth}}\left(d-1,\mathit{xs}^{+}\right),r,\mathit{odt}_{\text{depth}}\left(d-1,\mathit{xs}^{-}\right)\right)}$,
		is also a solution of $\mathit{odt}_{\text{depth}}$. In particular
		when $r$ is the defined as the global upper bound with a fiction
		configuration $\mathit{ub}^{*}$, we can implement $\mathit{thin}$
		using $\mathit{filter}$, then we have have 
\[
\begin{aligned}
	\mathit{odt}_{\text{thindep}} & : \mathbb{N} \times \mathcal{D} \to \mathit{DTree}\left(\mathcal{R}, \mathcal{D}\right) \\
	\mathit{odt}_{\text{thindep}} & \left(0, \mathit{xs}\right) = \mathit{DL}\left(\mathit{xs}\right) \\
	\mathit{odt}_{\text{thindep}} & \left(1, \mathit{xs}\right) = \mathit{min}_{E} \left[ \mathit{DN}\left( \mathit{DL}\left(\mathit{xs}^{+}\right), r, \mathit{DL}\left(\mathit{xs}^{-}\right) \right) \mid r \leftarrow \mathit{gen}_{\mathit{splits}}\left(\mathit{xs}\right) \right] \\
	\mathit{odt}_{\text{thindep}} & \left(d, \mathit{xs}\right) = \begin{cases}
		t^{*} & \mathit{ub}^{*} \leq \mathit{lb}\left(d, \mathit{xs}\right) \\
		\begin{aligned}
			& \mathit{min}_{E} \Big( \Big[ \mathit{DN}\Big( \mathit{odt}_{\text{depth}}\left(d-1, \mathit{xs}^{+}\right), r, \mathit{odt}_{\text{depth}}\left(d-1, \mathit{xs}^{-}\right) \Big) \mid\\
			& \quad \quad r \leftarrow \mathit{gen}_{\mathit{splits}}\left(\mathit{xs}\right) \Big] \Big)
		\end{aligned} & \text{otherwise}
	\end{cases}
\end{aligned}
\]
		where $E\left(t^{*}\right)=\inf$.}
\end{lemma}
\begin{proof}
	We prove \ref{eq: odt_Depth with thinning} by following equational
	reasoning 
	\[
	\begin{aligned} & \mathit{min}_{E}\left(\left[\mathit{DN}\left(\mathit{odt}_{\text{depth}}\left(d-1,\mathit{xs}^{+}\right),r,\mathit{odt}_{\text{depth}}\left(d-1,\mathit{xs}^{-}\right)\right)\mid r\leftarrow\mathit{gen_{splits}}\left(\mathit{xs}\right)\right]\right)\\
		= & \text{ \{ thin-introduction law\}}\\
		& \mathit{min}_{E}\left(\mathit{thin}_{R_{2}}\left(\left[\mathit{DN}\left(\mathit{odt}_{\text{depth}}\left(d-1,\mathit{xs}^{+}\right),r,\mathit{odt}_{\text{depth}}\left(d-1,\mathit{xs}^{-}\right)\right)\mid r\leftarrow\mathit{gen_{splits}}\left(\mathit{xs}\right)\right]\right)\right)\\
		= & \text{ \{ \ensuremath{\mathit{odt}_{\text{depth}}\left(d,\mathit{xs}\right)} is abstracted as \ensuremath{\mathit{mapL}_{f}}\}}\\
		& \mathit{min}_{E}\left(\mathit{thin}_{R_{1}}\left(\mathit{mapL}_{f}\left(\left[\left(r,d,\mathit{xs}\right)\mid r\leftarrow\mathit{gen_{splits}}\left(\mathit{xs}\right)\right]\right)\right)\right)\\
		\subseteq & \text{ \{ thin-map law, when \ensuremath{R_{1}=\mathit{fdom}_{1}}\}}\\
		& \mathit{min}_{E}\left(\mathit{mapL}_{f}\left(\mathit{thin}_{fdom_{1}}\left(\left[\left(r,d,\mathit{xs}\right)\mid r\leftarrow\mathit{gen_{splits}}\left(\mathit{xs}\right)\right]\right)\right)\right)
	\end{aligned}
	\]
	In particular, when $\ensuremath{\ensuremath{R_{1}=\mathit{gubdom}}}$,
	then thinning algorithm $\mathit{thin}_{R_{2}}$ is just the filter
	function $\mathit{filter}$, then we have 
\[
\begin{aligned}
	& \mathit{min}_{E}\left(\mathit{thin}_{R_{1}}\left(\mathit{mapL}_{f}\left(\left[\left(r,d,\mathit{xs}\right) \mid r \leftarrow \mathit{gen_{splits}}\left(\mathit{xs}\right)\right]\right)\right)\right) \\
	\subseteq & \text{ \{ thin-map law, when \ensuremath{R_{1}=\mathit{gubdom}}\}} \\
	& \mathit{min}_{E}\left(\mathit{mapL}_{f}\left(\mathit{filter}_{p_{\left(lb,\mathit{ub}^{*}\right)}}\left(\left[\left(r,d,\mathit{xs}\right) \mid r \leftarrow \mathit{gen_{splits}}\left(\mathit{xs}\right)\right]\right)\right)\right) \\
	= & \text{ \{ definition of \ensuremath{\mathit{filter}}\}} \\
	& \begin{cases}
		t^{*} & \mathit{ub}^{*} \leq \mathit{lb}\left(d, \mathit{xs}\right) \\
		\begin{aligned}
			& \mathit{min}_{E} \Big( \Big[ \mathit{DN}\Big( \mathit{odt}_{\text{depth}}\left(d-1, \mathit{xs}^{+}\right), d r, \mathit{odt}_{\text{depth}}\left(d-1, \mathit{xs}^{-}\right) \Big) \mid \\
				& \quad \quad r \leftarrow \mathit{gen}_{\mathit{splits}}\left(\mathit{xs}\right) \Big] \Big)
		\end{aligned} & \text{otherwise}.
	\end{cases}
\end{aligned}
\]
\end{proof}
Note that the thinning algorithm does not necessarily guarantee a
speedup compared to the original definition of $\mathit{odt}_{\text{depth}}$.
Its performance depends on both the \textbf{implementation} of $\mathit{thin}$
and the \textbf{effectiveness of the relation} $R$. \citet{he2025ROF}
provides various recursive definitions of $\mathit{thin}$ and the
effectiveness of each definition is problem-dependent.

\subsection{Comparison of four methods\label{subsec:Comparison-of-four algs}}

\subsubsection{Comparison between size- and depth-constrained trees algorithms\label{subsec:Comparision-between-leave-depth}}

As noted, depth-constrained decision trees are commonly used in ODT
studies with relatively small combinatorics, such as axis-parallel
trees \citep{mazumder2022quant,brita2025optimal} trees over binary-feature
data \citep{demirovic2022murtree,verwer2019learning,nijssen2007mining,aglin2020learning,nijssen2010optimal}.
These existing methods typically process recursion in a \emph{depth-first}
manner, unlike the \emph{breadth-first} formulation we defined in
\ref{eq: DP solution to odt_Depth}. In contrast, the size-constrained
ODT algorithm has been far less studied; \citet{hu2019optimal,lin2020generalized,zhang2023optimal}
have applied it to solve the ODT problem over binary-feature data,
but the correctness of their algorithms remains unresolved.

This preference is not coincidental. The discussion below clarifies
why researchers favor depth-constrained approaches and why the depth-constrained
ODT algorithm \ref{eq: DP solution to odt_Depth} is unsuitable for
problems with intractable combinatorics, such as hypersurface decision
trees. The primary reason is that such problems require parallelization
to achieve reasonable solution times, yet depth-constrained trees
are difficult to implement in an \emph{embarrassingly parallel} manner—parallelization
requires minimal or \textbf{no communication} between processors.

We first analyze the size-constrained decision tree and explain why
it is suitable for parallel implementation. In \ref{eq: simplified ODT problem},
the original size-constrained ODT problem can be factorized as:
\begin{equation}
	\mathit{min}_{E}\circ\mathit{concatMapL}{}_{\mathit{sodt}\left(\mathit{xs}\right)}\circ\mathit{kcombs}_{K}=\mathit{odt}_{K}
\end{equation}
This factorization indicates that solving the challenging $\mathit{odt}$
problem can be achieved by solving many smaller instances of $\mathit{sodt}$.
As \citet{he2025CGs} demonstrated, there exists an elegant and efficient
divide-and-conquer definition for $\mathit{kcombs}_{K}$, which decomposes
the problem into smaller, independent subproblems. These subproblems
can be processed in parallel and mapped efficiently onto modern memory
hierarchies, particularly when combined with recursively blocked arrays
\citep{gupta2000automatic}. Unfortunately, no similar factorization
has yet been found for the depth-constrained ODT problem that would
enable embarrassingly parallel execution. Consequently, the optimal
decision tree algorithm derived from \ref{spec: ODT_depth} can only
be executed sequentially or parallelized post hoc, making it difficult
to achieve the same level of speed-up as embarrassingly parallel algorithms.

For hyperplane decision trees, each recursive step may involve up
to $O\left(N^{D}\right)$ possible rules. Adopting a breadth-first
search strategy would result in extremely slow branching and high
memory consumption. When exact solutions are infeasible—as is the
case for most data lists with large $d$ and $N$—depth-first search
becomes necessary, as it can provide approximate solutions within
a fixed time limit.

Although depth-first search effectively reduces memory usage, it significantly
complicates parallelization compared to breadth-first strategies.
Previous research on ODT problems with smaller combinatorics remains
tractable even without parallelization because the number of possible
splits at each recursive step is limited to $O\left(N\times D\right)$
for axis-parallel trees. However, this approach becomes ineffective
when applied to the larger combinatorics of \ref{spec: ODT_depth}
where splitting rules are defined as hyperplanes. Hence, parallelization
is likely the only viable solution for problems with such large combinatorial
complexity.

Nevertheless, developing a combinatorial characterization for depth-constrained
decision trees that allows parallel factorization, similar to the
size-constrained ODT problem, represents an interesting direction
for future research.

\subsubsection{Comparison between size-constrained optimal decision tree algorithms\label{subsec:Comparsion-between-rec and seq}}

\begin{table}
	\begin{centering}
		\footnotesize%
		\begin{tabular}{>{\centering}p{2cm}>{\raggedright}p{4cm}>{\raggedright}p{4cm}>{\raggedright}p{4cm}}
			Aspect & $\mathit{sodt}_{\text{rec}}$ & $\mathit{sodt}_{\text{vec}}$ & $\mathit{sodt}_{\text{kperms}}$\tabularnewline
			\hline 
			Advantages & \begin{itemize}
				\item Lower computational cost
				\item Elegant and easily implemented recursive definition
				\item Facilitates recursive tree evaluation
				\item Low probability of worst-case complexity
				\item Memory efficient
			\end{itemize}
			& \begin{itemize}
				\item Trackable, iterative computation enabling full vectorization
				\item Support incremental objective evaluation
			\end{itemize}
			& \begin{itemize}
				\item Trackable, iterative computation enabling full vectorization
				\item Moderate probability of worst-case complexity
			\end{itemize}
			\tabularnewline
			\hline 
			Disadvantages & \begin{itemize}
				\item Risk of stack overflow in deep recursion due to stack frame allocation
				\item Limited support for tail call optimization (TCO) in imperative languages
				\item Challenges in parallelization due to sequential call stack
			\end{itemize}
			& \begin{itemize}
				\item Expensive $\textit{unique}$ operation for large $K$
				\item Costly feasibility testing ($O\left(K^{2}\right)$ in worst case)
				\item High probability of worst-case complexity
				\item Memory intensive due to vectorized implementation
			\end{itemize}
			& \begin{itemize}
				\item Generates duplicate $K$-permutations corresponds to the same tree
				\item Costly feasibility testing ($O\left(K^{2}\right)$ in worst case)
				\item Memory intensive due to vectorized implementation
				\item Does not support incremental objective evaluation
			\end{itemize}
			\tabularnewline
		\end{tabular}
		\par\end{centering}
	\caption{Comparison of advantages and disadvantages of $\mathit{sodt}_{\text{rec}}$,
		$\mathit{sodt}_{\text{vec}}$ and $\mathit{sodt}_{\text{kperms}}$
		algorithms \label{tab:Comparison of leaf-algs}}
\end{table}
\begin{table}
	\begin{centering}
		\begin{tabular}{ccccc}
			\hline 
			Methods & \begin{cellvarwidth}[t]
				\centering
				Theoretical
				
				Computational
				
				Efficiency
			\end{cellvarwidth} & \begin{cellvarwidth}[t]
				\centering
				Memory
				
				Usage
			\end{cellvarwidth} & Parallelizability & \begin{cellvarwidth}[t]
				\centering
				Recursive
				
				Evaluability
			\end{cellvarwidth}\tabularnewline
			\hline 
			$\mathit{odt}_{\text{kperms}}$ & 2 & 3 & 1 & 5\tabularnewline
	
			$\mathit{odt}_{\text{rec}}$ & 1 & 2 & 2 & 1\tabularnewline

			$\mathit{odt}_{\text{vec}}$ & 2 & 4 & 1 & 3\tabularnewline

			$\mathit{odt}_{\text{depth}}$ & 1 & 1 & 4 & 1\tabularnewline
	
		\end{tabular}
		\par\end{centering}
	\caption{Comparison $\mathit{odt}_{\text{rec}}$, $\mathit{odt}_{\text{vec}}$,
		$\mathit{odt}_{\text{kperms}}$, and $\mathit{odt}_{\text{depth}}$
		algorithms, we provide a ranking on a scale from 1 to 5: 1 indicates
		easy, 2 indicates relatively easy but less efficient than 1, 3 indicates
		difficult but feasible, 4 indicates very difficult or involving a
		highly complex process, and 5 indicates infeasible or requiring prohibitive
		computational resources.\label{tab: ranks of different methods}}
\end{table}

Having discussed the advantages of depth- and size-constrained decision
trees, we believe size-constrained ODT algorithms are more suitable
candidates for solving complex ODT problems. We present three distinct
definitions of $\mathit{odt}_{\text{size}}$: $\mathit{odt}_{\text{rec}}$,
$\mathit{odt}_{\text{vec}}$ and $\mathit{odt}_{\text{kperms}}$—naturally
raising the question: which is most suitable?

In machine learning research, the interplay between software, hardware,
and algorithms is often overlooked. An algorithm's success may derive
not from its intrinsic superiority but from the tools—hardware and
software—used in its implementation. This phenomenon, termed the ``hardware
lottery'' by \citet{hooker2021hardware} highlights how the siloed
development of hardware, software, and algorithms has historically
shaped which approaches prevail. A prominent example is the success
of deep neural networks, which benefited substantially from optimized
hardware and software.

Similarly, in the ODT problem, a comparable ``lottery'' exists between
the vectorized method $\mathit{odt}_{\text{vec}}$ and recursive method
$\mathit{odt}_{\text{rec}}$ for solving optimal decision tree problems.
Because $\mathit{odt}_{\text{size}}$ can be factorized into many
smaller instances solved using $\mathit{sodt}$, the algorithmic behavior
is largely determined by the performance of $\mathit{sodt}$. Therefore,
we focus our discussion on the behavior of these algorithms.

Compare with $\mathit{sodt}_{\text{rec}}$, $\mathit{sodt}_{\text{vec}}$
and $\mathit{sodt}_{\text{kperms}}$, $\mathit{sodt}_{\text{rec}}$
not only requires provably fewer computations due to its DP solution
but also has a lower probability of incurring worst-case complexity.
Due to the use of tree datatype, the recursive pattern of $\mathit{sodt}_{\text{rec}}$
behaves like a divide-and-conquer algorithm, so the combinatorial
complexity of subtrees decreases \textbf{logarithmically} with respect
to the length of $\mathit{rs}$ in typical cases, whereas $\mathit{sodt}_{\text{vec}}$
decreases linearly due to its intrinsic sequential nature. Furthermore,
$\mathit{sodt}_{\text{rec}}$ creates no duplicate trees, eliminating
the need for a $\mathit{unique}$ operation. However, its drawbacks
are significant: it is difficult to optimize and carries a risk of
stack overflow due to recursive calls. Moreover, $\mathit{sodt}_{\text{rec}}$
is a \emph{data-dependent recursion}, in which the function's behavior
and termination depend on the structure or values of the input, produces
irregular and unpredictable branching and memory-access patterns that
are challenging to parallelize efficiently.

In contrast, $\mathit{sodt}_{\text{vec}}$ is purely sequential and
can be fully vectorized, offering substantial advantages. It enables
incremental evaluation of the tree objective, as feasible splits are
added recursively to leaves, requiring only simple leaf splits. Nevertheless,
it is less efficient than $\mathit{sodt}_{\text{rec}}$ because no
DP solution exists for the objective that we concerns, and additional
computations arise from feasibility checks using the ancestry relation
matrix $\boldsymbol{K}$ and the $\mathit{unique}$ operation. This
method is also more memory-intensive, particularly when processing
data in batches. For instance, storing $M$ trees with $K$ splitting
rules in a batch requires $N\times M\times\left(K+1\right)$ memory
for predictions—a significant burden for large $M$ and $K$. In practice,
this can quickly exceed GPU memory limits.

The $\mathit{odt}_{\text{kperms}}$ algorithm faces similar challenges
to $\mathit{sodt}_{\text{vec}}$. Representing decision trees via
$K$-permutations produces duplicate trees. Unlike $\mathit{sodt}_{\text{vec}}$,
duplicates can sometimes be removed using a $\mathit{unique}$ operation,
but permutations such as $\left[1,2,3\right]$ and $\left[1,3,2\right]$
correspond to the same tree require reconstructing the tree structure
first, which incurs additional computational cost.

Overall, we believe $\mathit{sodt}_{\text{rec}}$ and $\mathit{sodt}_{\text{vec}}$
are the most suitable methods for solving the HODT problem. However,
it remains unclear whether the hardware advantages of $\mathit{sodt}_{\text{vec}}$
will outweigh the computational efficiency of $\mathit{sodt}_{\text{rec}}$.
We summarized advantages and limitations of $\mathit{odt}_{\text{size}}$
algorithms in Table \ref{tab:Comparison of leaf-algs}.

Additionally, Table \ref{tab: ranks of different methods} ranks $\mathit{odt}$
algorithms—$\mathit{odt}_{\text{rec}}$, $\mathit{odt}_{\text{vec}}$,
$\mathit{odt}_{\text{kperms}}$, and $\mathit{odt}_{\text{depth}}$—based
on \emph{theoretical computational efficiency}, \emph{memory usage},
\emph{parallelizability}, and \emph{recursive evaluability}. Notably,
$\mathit{odt}_{\text{kperms}}$ receives the lowest recursive evaluability
rating of 5 because adding a new rule, such as inserting $r_{3}$
at the front of $\left[r_{1},r_{2}\right]$ may invalidates previously
computed results, as the new tree structure may place $r_{1}$ and
$r_{2}$ on different branches of $r_{3}$.

In Part II, we implement optimized versions of $\mathit{sodt}_{\text{rec}}$
and $\mathit{sodt}_{\text{vec}}$ using more efficient array-based
tree representations instead of tree structures. Experimental results
align with expectations: although $\mathit{sodt}_{\text{rec}}$ is
theoretically more efficient $\mathit{sodt}_{\text{vec}}$ outperforms
it in both sequential and parallel settings due to better hardware
compatibility. Future work may explore hardware tailored to $\mathit{sodt}_{\text{rec}}$,
enabling a more balanced comparison between these approaches.

\section{Geometric foundation\label{sec:Geometric-foundation}}

Having developed a comprehensive algorithmic theory for ODT problems,
along with four distinct solutions for both size- and depth-constrained
cases, we note that these theories rely on the assumption that the
splitting rules ($\mathit{rs}$) or the method for generating these
rules ($\mathit{gen}_{\mathit{splits}}$) are known. However, in machine
learning, such rules are typically not readily available; they must
be inferred from data. Moreover, our theory is grounded in the axioms
of the proper decision tree, but not all splitting rules satisfy these
axioms—for example, those used in ODT problems over binary feature
data.

In Section \ref{subsec: combinatorial geometry}, we present a detailed
analysis of the combinatorial and geometric aspects of hypersurfaces
by drawing on results from the theory of \emph{hyperplane arrangements}
and \emph{algebraic geometry}. We show that, when data points are
in general position, hypersurfaces defined by \emph{fixed-degree polynomials}
can be represented as combinations of data items and can be exhaustively
generated in polynomial time. Furthermore, we demonstrate that such
hypersurfaces satisfy the axioms of the proper decision tree.

Section \ref{subsec: combinatorial geometry} explains the construction
of two splitting rule generators, that is, a concrete implementation
of $\mathit{gen}_{\mathit{splits}}$ for \emph{hypersurface} and \emph{mixed
	splitting rules}. Finally, Subsection \ref{subsec:Complexity-of ODT problems}
presents a detailed combinatorial analysis of the complexity of various
ODT problems.

\subsection{Classification meet combinatorial geometry \label{subsec: combinatorial geometry}}

Polynomial hypersurfaces, and their special case—hyperplanes—are natural
candidates for the splitting rules of a proper decision tree, as they
naturally satisfy Axioms 1, 2, and 3. However, for Axiom 4, a subtle
question arises: why can't the relations $h_{i}^{M}\searrow h_{j}^{M}$,
$h_{i}^{M}\swarrow h_{j}^{M}$ coexist? At first glance, by Bézout's
theorem, any two polynomial hypersurfaces must intersect at some points.
Intuitively, this suggests that a hypersurface could lie on both sides
of another by being ``cut'' into different regions along their intersection
points, thus both $h_{i}^{M}\searrow h_{j}^{M}$ and $h_{i}^{M}\swarrow h_{j}^{M}$
should coexist because $h_{i}$ can cut $h_{j}$ and $h_{j}$ can
cut $h_{i}$ into two parts. While this is true in general, the situation
is different in the context of machine learning, where we are primarily
concerned with the partition of data induced by hypersurfaces rather
than the hypersurfaces themselves.

By enforcing a natural assumption on the data distribution, we show
that the hypersurfaces generated from data can only lie on one side
of another hypersurface, ensuring that the coexistence of $h_{i}^{M}\searrow h_{j}^{M}$
and $h_{i}^{M}\swarrow h_{j}^{M}$ is impossible.

\subsubsection{Polynomial hypersurface and the $W$-tuple Veronese embedding}

A \emph{monomial} $M$ with respect to a $D$-tuple $x=\left(x_{1},x_{2}\ldots,x_{D}\right)$
is a product of the form $M\left(x\right)=x_{1}^{\alpha_{1}}\cdot x_{2}^{\alpha_{2}}\ldots\cdot x_{D}^{\alpha_{D}},$
where $\alpha_{1},\alpha_{2}\ldots,\alpha_{D}$ are \emph{non-negative}
integers, denote $\boldsymbol{\alpha}=\left(\alpha_{1},\alpha_{2}\ldots,\alpha_{D}\right)$.
The\emph{ total degree} of this monomial is the sum $\left|\boldsymbol{\alpha}\right|=\alpha_{1}+\cdot\cdot\cdot+\alpha_{n}$.
\begin{definition}
	\emph{Polynomial}. A polynomial $P\left(x\right)$, or $P$ in short,
	is a finite linear combination (with coefficients in the field $\mathbb{R}$)
	of monomials in the form 
	\begin{equation}
		P\left(x\right)=\sum_{i}w_{i}x^{\boldsymbol{\alpha}_{i}},w_{i}\in\mathbb{R}.
	\end{equation}
	We denote the set of all polynomials with variables $x_{1},x_{2}\ldots,x_{D}$
	and coefficients in $\mathbb{R}$ is denoted by $\mathbb{R}\left[x_{1},x_{2}\ldots,x_{D}\right]$
	or $\mathbb{R}\left[x\right]$.
\end{definition}
The \emph{maximal degree} of $P$, denoted $deg\left(P\right)$, is
the maximum $\left|\boldsymbol{\alpha}_{i}\right|$ such that the
coefficient $\alpha_{j}$ is nonzero. For instance, the polynomial
$P\left(x\right)=5x_{1}^{2}+3x_{1}x_{2}+x_{2}^{2}+x_{1}+x_{2}+3$
for $x\in\mathbb{R}^{2}$ has six terms and maximal degree two.

The number of possible monomial terms of a degree $K$ polynomial
is equivalent to the number of ways of selecting $K$ variables from
the multisets of $D+1$ variables\footnote{There are $D+1$ variables considering polynomials in homogeneous
	coordinates, i.e. \emph{projective space} $\mathbb{P}^{D}$.}. This is equivalent to the \emph{size $K$ combinations of $D+1$
	elements taken with replacement}. In other words, selecting $K$ variables
from the variable set $\left(x_{0},x_{1},\ldots,x_{D}\right)$ in
homogeneous coordinates with repetition. This lead to the following
fact.
\begin{fact}
	\emph{If polynomial $P$ in $\mathbb{R}\left[x_{1},x_{2}\ldots,x_{D}\right]$
		has maximal degree $M$, then polynomial $P$ has $\left(\begin{array}{c}
			D+M\\
			D
		\end{array}\right)$ monomial terms (including constant term) at most.}
\end{fact}
Given polynomial $P\in\mathbb{R}\left[x_{1},x_{2}\ldots,x_{D}\right]$
of degree $M$, a hypersurface $h^{M}$ defined by $P$ is the set
of \emph{zeros} of $P$ is $h^{M}=\left\{ \left(x_{1},x_{2}\ldots,x_{D}\right)\in\mathbb{R}^{D}:P\left(x_{1},x_{2}\ldots,x_{D}\right)=0\right\} $.
Geometrically, the zero set of a polynomial $P$ determines a subset
of $\mathbb{R}^{D}$. This subset is called a \emph{surface} if its
dimension is less than $D$ and a \emph{hypersurface} if its dimension
is $D-1$. Specifically, if this subset is an affine subspace defined
by a \emph{degree-one} polynomial, it is referred to as a \emph{flat}.
A $D-1$-dimensional affine flat is known as a \emph{hyperplane}.

\subsubsection{0-1 loss linear classification theorem}

A hyperplane arrangement is a dissection of $\mathbb{R}^{D}$ induced
by a finite set of hyperplanes. At first glance, a hyperplane arrangement
may appear to carry more structural information than a set of data
points (a \emph{point configuration}). However, an effective approach
to studying geometric structures involving points and hyperplanes
is to examine the transformations between these two representations.
By analyzing the \textbf{duality} between point configurations and
hyperplane arrangements, it becomes evident that both possess equally
rich combinatorial structures.

Although the space of all possible hyperplanes seems to exhibit infinite
combinatorial complexity—since each hyperplane is defined by a continuous
normal vector $\boldsymbol{w}_{k}$\LyXZeroWidthSpace —the finiteness
of data imposes strong constraints. Specifically, the number of \emph{distinct
	data partitions} induced by hyperplanes is \emph{finite}, which introduces
an \textbf{equivalence relation} among hyperplanes.

In 1965, Cover provided a precise formula for counting the number
of possible partitions induced by affine hyperplanes. Given $N$ data
points in general position, the number of \emph{linearly separable
	dichotomies} is:\emph{ }
\begin{equation}
	\mathcal{S}_{\mathrm{Cover}}\left(N,D+1\right)=2\sum_{d=0}^{D}\left(\begin{array}{c}
		N-1\\
		d
	\end{array}\right).
\end{equation}
This result implies that the optimal 0-1 loss solution obtained by
exhaustively searching over Cover's space $\mathcal{S}_{\mathrm{Cover}}$
is equivalent to the optimal solution obtained by searching the entire
continuous space $\mathbb{R}^{D}$. Formally, given $N$ data points
in general position, we have 
\begin{equation}
	\mathit{minR}{}_{E_{\text{0-1}}}\left\{ \boldsymbol{w}\mid\boldsymbol{w}\in\mathcal{S}_{\mathrm{Cover}}\left(N,D+1\right)\right\} \subseteq\mathit{minR}{}_{E_{\text{0-1}}}\left\{ \boldsymbol{w}\mid\boldsymbol{w}\in\mathbb{R}^{D}\right\} ,\label{eq: cover's equivalence}
\end{equation}
where $\mathit{minR}{}_{E_{\text{0-1}}}$ selects one of the optimal
solutions with respect to $E_{\text{0-1}}$, and $E_{\text{0-1}}\left(\boldsymbol{w},\mathit{xs}\right)=\sum_{\left(x,t\right)\in\mathcal{D}}\mathbf{1}\left[\boldsymbol{w}^{T}x\neq t\right]$,
we omit dataset $\mathit{xs}$ when evaluating $\boldsymbol{w}$ in
$\mathit{minR}{}_{E_{\text{0-1}}}$. Equation (\ref{eq: cover's equivalence})
shows that exhaustive enumeration of $\mathcal{S}_{\mathrm{Cover}}\left(N,D+1\right)$
yields an optimal solution without searching an infinitely large hypothesis
space.

However, Cover's formula only \textbf{counts} the number of possible
partitions; it does \textbf{not} provide an explicit method for enumerating
the equivalence classes of hyperplanes. To address this, \citet{he2023efficient}
introduced a \emph{dual transformation} that maps a set of data points
into a hyperplane arrangement while preserving the incidence relations
between hyperplanes and data points, revealing that hyperplanes and
points are isomorphic in a combinatorial sense. In particular, we
have following results.
\begin{theorem}
	\emph{Under the }general position\emph{ assumption and when optimizing
		the 0-1 loss, for any partition induced by a hyperplane $h$ in $\mathbb{R}^{D}$
		with respect to a dataset $\mathit{xs}$, there exists $D$ points
		in $\mathit{xs}$ such that the partition induced by hyperplanes passing
		through these $D$ points coincides with the partition induced by
		$h$}
\end{theorem}
Consequently, the 0-1 loss linear classification problem can be solved
by exhaustively enumerating $\left(\begin{array}{c}
	N\\
	D
\end{array}\right)=O\left(N^{D}\right)$ possible combinations of data points. More formally, we have the
following theorem.
\begin{theorem}
	\emph{0-1 loss linear classification theorem. Consider a data list
		$\mathit{xs}:\left[\mathbb{R}^{D}\right]$ of $N$ points in general
		position, along with their associated labels. Then 
\begin{equation}
	\begin{aligned}
		& \mathit{minR}_{E_{\text{0-1}}}\left\{ \min_{E_{\text{0-1}}}\left(\boldsymbol{w}_{s}, -\boldsymbol{w}_{s}\right) \mid s \in \mathcal{S}_{\text{kcombs}}\left(D, \mathit{xs}\right) \right\} \subseteq \\
		& \quad \mathit{minR}_{E_{\text{0-1}}}\left\{ \boldsymbol{w} \mid \boldsymbol{w} \in \mathcal{S}_{\mathrm{Cover}}\left(N, D+1\right) \right\} \subseteq \mathit{minR}_{E_{\text{0-1}}}\left\{ \boldsymbol{w} \mid \boldsymbol{w} \in \mathbb{R}^{D} \right\}
	\end{aligned}
\end{equation}
		and $\boldsymbol{w}_{s}\in\mathbb{R}^{D}$ denotes the normal vector
		determined uniquely by $D$ data points (in general position, a unique
		hyperplane passes through any such $D$ points, which can be computed
		via matrix inversion).\label{0-1 loss linear classification thoerm}}
\end{theorem}
\begin{proof}
	See \citet{he2023efficient}.
\end{proof}

\subsubsection{0-1 loss polynomial hypersurface classification theorem}

\begin{table}
	\begin{centering}
		\tiny%
		\begin{tabular}{>{\centering}p{2.5cm}cc>{\centering}p{2.4cm}}
			\hline 
			Polynomial hypersurfaces & Embedded data & Polynomial expression & Number of data required\tabularnewline
			\hline 
			Axis-parallel hyperplane & $\left(0,\ldots,x_{i},\ldots0\right)$ & $x_{i}=a$, $\forall i\in\left[D\right]$ & 1\tabularnewline

			Affine hyperplane & $\left(x_{1},\ldots,x_{D},1\right)$ & $w_{1}x_{1}+w_{2}x_{2}+\ldots+w_{D}x_{D}+1=0$,$w\bar{\boldsymbol{x}}=0$, & $D$\tabularnewline

			Hypersphere & $\left(x_{1},\ldots,x_{D},x_{1}^{2},\ldots,x_{D}^{2},1\right)$ & $w_{1}x_{1}+\ldots w_{D}x_{D}\ldots+w_{2D}x_{D}^{2}+1=0$ & $2D$\tabularnewline

			Conic section & $\left(x_{1},\ldots,x_{1}x_{D},\ldots,x_{D}^{2},1\right)$ & $w_{1}x_{1}+\ldots w_{g}x_{1}x_{2}\ldots+w_{G}x_{D}^{2}+1=0$,$w\bar{\boldsymbol{x}}=0$, & $G=\left(\begin{array}{c}
				D+2\\
				2
			\end{array}\right)-1$\tabularnewline

			Arbitrary degree-$M$ polynomials & $\left(x_{1},\ldots,x_{1}x_{2}^{M-1},\ldots,x_{D}^{M},1\right)$ & $w_{1}x_{1}+\ldots w_{g}x_{1}x_{2}^{M-1}\ldots+w_{G}x_{D}^{M}+1=0$ & $G=\left(\begin{array}{c}
				D+M\\
				D
			\end{array}\right)-1$\tabularnewline

		\end{tabular}
		\par\end{centering}
	\caption{Examples of polynomial hypersurfaces with the corresponding polynomial
		expression, and the number data item requires to obtain the embedding.\label{tab:Examples-of-polynomial}}
\end{table}

Based on the point-hyperplane duality, equivalence relations for linear
classifiers on finite sets of data were established above. However,
a linear classifier is often too restrictive in practice, as many
problems require more complex decision boundaries. It is natural to
ask whether it is possible to extend the theory to non-linear classification.
This section explore a well-known concept in algebraic geometry, the
\emph{$M$-tuple Veronese embedding}, which allows the generalization
of the previous strategy for solving classification problem with \emph{hyperplane
	classifier }to problems involving \emph{hypersurface classifiers},
with a worst-case time complexity $O\left(N^{G}\right)$, where\emph{
}$G=\left(\begin{array}{c}
	D+M\\
	D
\end{array}\right)-1$, and $M$ is the degree of the polynomial defining the hypersurface.
If both $M$ and $D$ are fixed constants, this again gives a polynomial-time
algorithm for solving the 0-1 loss hypersurface classification problem.
\begin{theorem}
	\emph{Given variables $x_{1},\ldots x_{D}$ in space $\mathbb{R}^{D}$
		(which is isomorphic to the affine space $\mathbb{R}^{D}$ when ignoring
		the points at infinity \citep{cox1997ideals}), let $M_{0},M_{1},\ldots M_{G}$
		be all monomials of degree $M$ with variables $x_{0},x_{1},\ldots x_{D}$,
		where $G\leq\left(\begin{array}{c}
			D+M\\
			D
		\end{array}\right)-1$. Define an embedding $\rho_{M}:\mathbb{P}^{D+1}\to\mathbb{P}^{G}$
		which sends the point $\bar{x}=\left(1,x_{1},\ldots x_{D}\right)\in\mathbb{R}^{D+1}$
		to the point $\rho_{M}\left(\bar{x}\right)=\left(M_{0}\left(\bar{x}\right),M_{1}\left(\bar{x}\right),\ldots M_{G}\left(\bar{x}\right)\right)$,
		we refer to the data list $\rho_{M}\left(\mathit{xs}\right)=\left[\left(M_{0}\left(\bar{x}\right),M_{1}\left(\bar{x}\right),\ldots M_{G}\left(\bar{x}\right)\right)\mid x\leftarrow\mathit{xs}\right]$
		as the }embedded\emph{ }data list\emph{. The hyperplane classification
		over the embedded data lists $\rho_{M}\left(\mathit{xs}\right)$ is
		isomorphic to the polynomial hypersurface classification (defined
		by a degree $M$ polynomial) over the original data list $\mathit{xs}$.\label{thm:Polynomial-embedding.-Given}}
\end{theorem}
\begin{proof}
	First, we need to prove the existence of the hypersurface. For any
	$G$ distinct points $\mathit{s}$ in $\mathbb{R}^{D+1}$, there exist
	a polynomial hypersurface $S$ defined by a polynomial $P$ with degree
	at most $M$ such that $\mathit{s}$ lies on $S$. This is because
	the vector space $V$ of a polynomial $P$ in $\mathbb{R}\left[x_{1},\ldots x_{D}\right]$of
	degree $M$ has a dimension $\left(\begin{array}{c}
		D+W\\
		D
	\end{array}\right)$. For any data list $xs$ of size $G<\left(\begin{array}{c}
		D+M\\
		D
	\end{array}\right)$, the linear map from $V$ to $\mathbb{R}^{G}$ must has non-trivial
	kernel (the linear system has more variables than equations), thus
	the polynomial hypersurface $S$ must exist. Next, consider a hyperplane
	$\boldsymbol{w}^{T}\bar{y}=0$ defined over variables $\bar{y}=\left(1,y_{1},\ldots y_{G}\right)\in\mathbb{R}^{G+1}$,
	we have $\text{\ensuremath{\boldsymbol{w}^{T}\bar{y}}\ensuremath{\ensuremath{\geq}}0}$
	or $\text{\ensuremath{\boldsymbol{w}^{T}\bar{y}}<0}$ implies $\boldsymbol{w}^{T}\rho_{M}\left(\bar{x}\right)\geq0$
	or $\boldsymbol{w}^{T}\rho_{M}\left(\bar{x}\right)<0$, because the
	closure of the real numbers under polynomial operation, thus $M\left(\bar{x}\right)\in\mathbb{R}$
	if $\bar{x}\in\mathbb{R}^{D+1}$, for any monomial $M$.
\end{proof}
Theorem \ref{thm:Polynomial-embedding.-Given} show that the $M$-degree
polynomial hypersurface classification problem over data lists in
\emph{$\mathbb{R}^{D}$} is equivalent to a linear classification
problem in the embedded space $\mathbb{R}^{G}$.\emph{ }By analogy
with the linear classification result in Theorem \ref{0-1 loss linear classification thoerm},
we can formalize the 0-1 Loss Polynomial Hypersurface Classification
Theorem as follows.
\begin{theorem}
	\emph{0-1 loss polynomial hypersurface classification theorem. Consider
		a data list $\mathit{xs}$ of $N$ data points in $\mathbb{R}^{D}$
		in general position, along with their associated labels. Denote $E_{\text{0-1}}\left(\boldsymbol{w}\right)$
		as the 0-1 loss for evaluating normal vector $\boldsymbol{w}\in\mathbb{R}^{G}$
		with respect to data lists $\rho_{M}\left(\mathit{xs}\right)$. Then
		we have}
	
	\emph{
		\begin{equation}
			\mathit{minR}{}_{E_{\text{0-1}}}\left\{ min_{E_{\text{0-1}}}\left(\boldsymbol{w}_{s},-\boldsymbol{w}_{s}\right)\mid s\in\mathcal{S}_{\text{kcombs}}\left(G,\rho_{M}\left(\mathit{xs}\right)\right)\right\} \subseteq\mathit{minR}{}_{E_{\text{0-1}}}\left\{ \boldsymbol{w}\mid\boldsymbol{w}\in\mathbb{R}^{G}\right\} 
		\end{equation}
		where $\boldsymbol{w}_{s}\in\mathbb{R}^{G}$ denote as the normal
		vector determined by $G$ data points in $\mathcal{S}_{\text{kcombs}}\left(G,\mathit{xs}\right)$.
		\label{thm: 0-1-loss hypersurface classify}}
\end{theorem}

\subsubsection{Hypersurface decision trees are proper}

We now show that the hyperplane classification tree satisfies the
axioms of a proper decision tree, as verified by the following lemma.
\begin{lemma}
	\emph{Given a set of data $\mathit{xs}$ and list of $K$ hyperplanes
		$\mathit{hs}_{K}=\left[h_{1},h_{2},\ldots h_{K}\right]$ defined by
		normal vectors $ws=\left[w_{1},w_{2},\ldots w_{K}\right]$ in $\mathbb{R}^{D+1}$,
		where each hyperplane passes through exactly $D$ data points. Define
		the positive and negative regions of each hyperplane as }$h_{i}^{+}=\left\{ x\mid x\in\mathbb{R}^{D},\boldsymbol{w}_{i}^{T}\bar{x}\geq0\right\} $
	\emph{and }$h_{i}^{-}=\left\{ x\mid x\in\mathbb{R}^{D},\boldsymbol{w}_{i}^{T}\bar{x}<0\right\} $.\emph{
		Then, the decision tree constructed using these hyperplanes as splitting
		rules is a proper decision tree.}
\end{lemma}
\begin{proof}
	Axioms 1 and 2 follow directly from the definition of $h^{\pm}$.
	Axioms 3 and 4 hold because, for any hyperplane $h$ defined by a
	$D$-tuple of data points $\mathit{xs}_{D}$ we have
	\[
	\mathit{xs}_{D}\subseteq\bigcap_{p\in P}h_{p}^{\pm}\implies\mathit{xs}_{D}\subseteq h_{p}^{\pm},\forall p\in P
	\]
	and
	\[
	\mathit{xs}_{D}\subseteq h^{\pm}\implies\mathit{xs}_{D}\nsubseteq h^{\mp}.
	\]
	Thus axiom 3, 4 holds.
\end{proof}
By the equivalence between hyperplane and hypersurface, as established
in Theorem \ref{thm:Polynomial-embedding.-Given}, it follows immediately
that decision trees whose splitting rules are characterized by hypersurfaces
passing through $G$-combinations of data points are also proper.
We next show that the optimal decision tree constructed from these
hypersurfaces, with $G$ data points lying on them, is also optimal
among all decision trees constructed from general hypersurfaces in
\emph{$\mathbb{R}^{D}$.}
\begin{lemma}
	0-1 classification theorem for hypersurfaces decision tree model\emph{.
		Consider a data list $xs$ of $N$ data points of dimension $D$ in
		general position, along with their associated labels. The optimal
		decision tree constructed using $M$-degree polynomial hypersurfaces,
		each passing through $G$ points, achieves the same minimum 0-1 loss
		as the optimal decision tree constructed using arbitrary hypersurfaces
		in $\mathbb{R}^{D}$.}
\end{lemma}
\begin{proof}
	We prove this by induction. The base case, involving a single linear
	hypersurface, has already been established in Theorem \ref{thm: 0-1-loss hypersurface classify}.
	For the inductive step, assume two subtrees $t_{1}$ and $t_{2}$
	of size $i$ and $K-i-1$, respectively, are optimal for the corresponding
	subregions $\mathit{xs}^{+}$ and $\mathit{xs}^{-}$, determined by
	root $h$. We show that the tree $t=DN\left(t_{1},h,t_{2}\right)$
	is optimal with respect to root $h$. Since $t_{1}$ and $t_{2}$
	are optimal, any subtree of the same size constructed in $\mathit{xs}^{+}$
	and $\mathit{xs}^{-}$ will have a equal or greater 0-1 loss than
	$t_{1}$ and $t_{2}$. By definition of the objective function, we
	have 
	\[
	E\left(t^{\prime},xs\right)=E\left(t_{1}^{\prime},xs^{+}\right)+E\left(t_{2},xs^{-}\right)\geq E\left(t,xs\right)=E\left(t_{1},xs^{+}\right)+E\left(t_{2},xs^{-}\right)
	\]
	Thus, $t=DN\left(t_{1},h,t_{2}\right)$ is optimal for the fixed root
	$h$. Furthermore, for any arbitrary hypersurface defined by normal
	vector $\boldsymbol{w}\in\mathbb{R}^{G}$, there exists a hypersurface
	passing through $G$ points that induces the same partition $\mathit{xs}^{+}$
	and $\mathit{xs}^{-}$. Therefore, the optimal decision tree constructed
	using hypersurfaces through $G$ points attains the same 0-1 loss
	as the optimal tree constructed from arbitrary hypersurfaces.
\end{proof}
Although we assume that the data are in general position—which excludes
axis-parallel hyperplanes, as they do not satisfy the general position
assumption—the results also hold for axis-parallel hyperplanes, since
they can be infinitesimally shifted to become hyperplanes in general
position without affecting the predictions. Therefore, for axis-parallel
hyperplane splitting rule, a single data point is required to characterize
the hyperplane. By definition, an axis-parallel hyperplane, $x_{i}=c$
is uniquely determined by the $i$-th coordinate of any point lying
on it. That is, for any point $\left(x_{1},x_{2},\ldots,x_{D}\right)\in\mathbb{R}^{D}$,
the hyperplane parallel to the $i$-th axis passing through it is
uniquely determined by $x_{i}$. Since each data point has $D$ coordinates,
there are $N\times D$ possible axis-parallel splitting rules in total.

\subsection{Splitting rule generators\label{subsec:Splitting-rule-generator}}

Note that both $\mathit{odt}_{\text{size}}:\left[\mathcal{R}\right]\to\mathit{DTree}\left(\mathcal{R},\mathcal{D}\right)$
and $\mathit{odt}_{\text{depth}}:\mathcal{D}\to\mathit{DTree}\left(\mathcal{R},\mathcal{D}\right)$
cannot yet be directly applied to solve the optimal decision tree
problem in machine learning. This is because we have not explained
how to generate the input $\mathit{rs}:\left[\mathcal{R}\right]$
for $\mathit{odt}_{\text{size}}$ or the function $\mathit{gen}_{\mathit{splits}}:\mathcal{D}\to\left[\mathcal{R}\right]$
used in $\mathit{odt}_{\text{depth}}$. It turns out that both functions
are the same function. In this section, we describe how to construct
$\mathit{gen}_{\mathit{splits}}$ for solving various ML problems,
including axis-parallel, hyperplane, hypersurface, and mixed-splitting
ODT problems.

Specifically, the general paradigm for solving the size-constrained
optimal decision tree problem requires composing $\mathit{odk}_{\text{size}}\left(K\right)$
with $\mathit{gen}_{\mathit{splits}}^{M}$. In other words, we obtain
the following specification:
\begin{equation}
	\mathit{odt}_{\text{size}}\left(K\right)\circ\mathit{gen}_{\mathit{splits}}=\mathit{min}_{E_{\text{0-1}}}\circ\mathit{concatMap}_{\mathit{sodt}\left(\mathit{xs}\right)}\circ\mathit{kcombs}\left(K\right)\circ\mathit{gen}_{\mathit{splits}}\label{eq:specification of odt in ml}
\end{equation}
where $\mathit{gen}_{\mathit{splits}}$ is precisely the splitting
rule generator used in the definition of $\mathit{odk}_{\text{depth}}$.

In principle, as long as the rules $\mathit{rs}:\left[\mathcal{R}\right]$
generated by $\mathit{gen}_{\mathit{splits}}^{M}$ satisfy the decision
tree axioms \ref{def:Axioms-for Proper-DT}, $\mathit{odt}_{\text{size}}\left(K\right)$
can solve the size-constrained ODT problem exactly with respect to
$\mathit{rs}$. This \textbf{modularized} design provides substantial
flexibility for generalizing our algorithm to different problems.
In practice, we need only redefine $\mathit{gen}_{\mathit{splits}}^{M}$
to adapt the program to new problems, while the main program $\mathit{odt}_{\text{size}}\left(K\right)$
remains unchanged, which demonstrates one of the primary advantage
of adopting a general formalism in algorithm design.

For example, in the axis-parallel decision tree problem, $\mathit{gen}_{\mathit{splits}}^{0}:\mathcal{D}\to\left[\mathcal{H}^{0}\right]$
can be defined as the set of all possible axis-parallel hyperplanes
passing through each data point in $\mathit{xs}:\mathcal{D}$. Thus,
there are $N\times D$ such hyperplanes in total. More concretely,
given a dataset $\mathit{xs}:\mathcal{D}$ in $\mathbb{R}^{D}$,

\begin{equation}
	\mathit{gen}_{\mathit{splits}}^{0}\left(\mathit{xs}\right)=\left[x_{i}\mid x_{i}\leftarrow x,x\leftarrow\mathit{xs}\right],
\end{equation}
By substituting $\mathit{gen}_{\mathit{splits}}^{0}$ into \ref{def:Axioms-for Proper-DT}
\[
\mathit{aodt}_{K}=\mathit{odt}_{\text{size}}\left(K\right)\circ\mathit{gen}_{\mathit{splits}}^{0},
\]
where $\mathit{aodt}$ is short for ``axis-parallel optimal decision
tree problem,'' which yields an efficient solution for solving the
classical optimal decision tree problem in machine learning.

Constructing generators for more complex splitting rules is more involved,
but their structured form allows the composed function $\mathit{kcombs}\left(K\right)\circ\mathit{gen}_{\mathit{splits}}$
to be \emph{fused} into a more efficient algorithm. In this section,
we detail the construction of fused programs for both r $\mathit{gen}_{\mathit{splits}}^{M}$—the
generator of\emph{ hypersurface splitting rules}—and $\mathit{gen}_{\mathit{splits}}^{m}$
($0\leq m\leq M$) —the generator of \emph{mixed splitting rules}
that combine hypersurface, hyperplane, and axis-parallel rules.

\subsubsection{Generating fixed polynomial degree hypersurface splitting rules}

\begin{algorithm}
	\footnotesize
	\begin{enumerate}
		\item \textbf{Input}: $\mathit{xs}$: input data list of length $N$; $K$:
		Outer combination size (nested combination); $G$: inner-combination
		size (ordinary combination), determined by polynomial degree $M$
		of hypersurface
		\item \textbf{Output}: Array of $\left(K,G\right)$-nested-combinations\emph{}\\
		\item $\mathit{css}=\left[\left[\left[\:\right]\right],\left[\right]^{k}\right]$
		// initialize combinations
		\item $\mathit{ncss}=\left[\left[\left[\:\right]\right],\left[\right]^{k}\right]$
		// initialize nested-combinations
		\item $\mathit{asgn}^{+},\mathit{asgn}^{-}=\mathit{empty}\left(\left(\begin{array}{c}
			N\\
			D
		\end{array}\right),N\right)$ // initialize $\mathit{asgn}^{+},\mathit{asgn}^{-}$ as two empty
		$\left(\begin{array}{c}
			N\\
			D
		\end{array}\right)\times N$ matrix
		\item \textbf{for} $n\leftarrow\mathit{range}\left(0,N\right)$ \textbf{do}:
		//\textbf{$\mathit{range}\left(0,N\right)=\left[0,1,\ldots,N-1\right]$}
		\item $\quad$\textbf{for} $j\leftarrow\mathit{reverse}\left(\mathit{range}\left(G,n+1\right)\right)$
		\textbf{do}:
		\item $\quad$$\quad$$\mathit{updates}=\mathit{reverse}\left(\mathit{map}\left(\cup\rho_{M}\left(\mathit{xs}\right)\left[n\right],\mathit{css}\left[j-1\right]\right)\right)$
		\item $\quad$$\quad$$\mathit{css}\left[j\right]=\mathit{css}\left[j\right]\cup\mathit{updaets}$
		// update $\mathit{css}$ to generate combinations in revolving door
		ordering,
		\item $\quad$$\mathit{asgn}^{+},\mathit{asgn}^{-}=\mathit{genModels}\left(\mathit{css}\left[G\right],\mathit{asgn}^{+},\mathit{asgn}^{-}\right)$
		// use $G$-combination to generate the positive prediction and negative
		prediction of each hyperplanes and stored in $\mathit{asgn}^{+},\mathit{asgn}^{-}$
		\item $\quad$$C_{1}=\left(\begin{array}{c}
			n\\
			G-1
		\end{array}\right)$, $C_{2}=\left(\begin{array}{c}
			n\\
			G
		\end{array}\right)$
		\item $\quad$\textbf{for} $i\leftarrow\mathit{range}\left(C_{1},C_{2}\right)$
		\textbf{do}:
		\item $\quad$$\quad$$\quad$\textbf{for} $k\leftarrow\mathit{reverse}\left(\mathit{range}\left(K,i+1\right)\right)$
		\textbf{do}:
		\item $\quad$$\quad$$\quad$$\quad$$\mathit{ncss}\left[k\right]=\mathit{map}\left(\cup\left[i\right],\mathit{ncss}\left[k-1\right]\right)\cup\mathit{ncss}\left[k\right]$
		// update nested combinations
		\item $\quad$$\mathit{ncss}\left[K\right]=\left[\:\right]$
		\item \textbf{return} $\mathit{ncss}\left[K\right]$
	\end{enumerate}
	\caption{$\mathit{nestedCombs}$\label{alg:nested combination generator -sequential}}
\end{algorithm}

Due to the isomorphism between hyperplane classification and hypersurface
classification established in Theorem \ref{thm:Polynomial-embedding.-Given},
we can define the optimal hypersurface decision tree algorithm $\mathit{hodt}^{M}:\mathbb{N}\times\mathcal{D}\to\mathit{DTree}\left(\mathcal{H}^{M},\mathcal{D}\right)$
as the composition of $\mathit{odk}_{\text{size}}\left(K\right)$
with a \emph{hypersurface splitting rule generator} $\mathit{gen}_{\mathit{splits}}^{M}$,
together with an embedding function $\mathit{embed}_{M}$ (the programmatic
definition of $\rho_{M}$). In other words, we obtain the following
specification:
\begin{equation}
	\begin{aligned}\mathit{hodt}_{K}^{M} & =\mathit{odk}_{\text{size}}\left(K\right)\circ\mathit{gen}_{\mathit{splits}}^{M}\circ\mathit{embed}_{M}\\
		& =\mathit{min}_{E_{\text{0-1}}}\circ\mathit{concatMap}_{\mathit{sodt}\left(\mathit{xs}\right)}\circ\mathit{kcombs}\left(K\right)\circ\mathit{gen}_{\mathit{splits}}^{M}\circ\mathit{embed}_{M}
	\end{aligned}
	\label{eq: hodt specification}
\end{equation}

As established in Theorems \ref{0-1 loss linear classification thoerm}
and \ref{thm: 0-1-loss hypersurface classify}, both hyperplane and
hypersurface decision boundaries can be exhaustively enumerated by
generating $G$-combinations of data points. Consequently, the search
space of hypersurface decision boundaries can be represented as $\mathcal{S}_{\text{kcombs}}\left(G,\mathit{ds}\right)$,
and the generator for this search space can be implemented directly
using a combination generator: 
\begin{equation}
	\mathit{gen}_{\mathit{splits}}^{M}\left(\mathit{ds}\right)=!\left(G\right)\circ\mathit{kcombs}\left(G,\mathit{ds}\right),
\end{equation}
where $\mathit{kcombs}\left(G,\mathit{ds}\right):\mathbb{N}\times\mathcal{D}\to\left[\left[\mathcal{D}\right]\right]$
generate all combinations with size smaller than $G$, and $!\left(G\right)$
extract combinations of size exactly $G$. Hence, the composition
\begin{equation}
	\mathit{kcombs}\left(K\right)\circ\mathit{gen}_{\mathit{splits}}^{M}=\mathit{kcombs}\left(K\right)\circ!\left(G\right)\circ\mathit{kcombs}\left(G,\mathit{ds}\right),
\end{equation}
corresponds to a combination-of-combinations generator, i.e., \emph{nested
	combinations}, denote as $\left(K,G\right)$-nested combinations.

Once $G$-combinations are used to generate $K$-combinations, these
intermediate combinations are no longer needed and can be discarded.
By introducing an additional set-empty operation, $\mathit{setEmpty}\left(G,\mathit{xs}\right)$
which sets the $G$th element of the list $\mathit{xs}$ to an empty
list, we can modify the nested combination generator $\mathit{kcombs}\left(K\right)\circ\mathit{gen}_{\mathit{splits}}^{M}$
by introducing $\mathit{setEmpty}$. This modified generator, denoted
as $\mathit{nestedCombs}\left(K,G\right)$, can be compositionally
specified as:
\begin{equation}
	\begin{aligned}\mathit{nestedCombs}\left(K,G\right) & =\left\langle \mathit{setEmpty}\left(G\right)\circ\mathit{kcombs}\left(G\right),\mathit{kcombs}\left(K\right)\circ!\left(G\right)\circ\mathit{kcombs}\left(G\right)\right\rangle \\
		& =\left\langle \mathit{setEmpty}\left(G\right),\mathit{kcombs}\left(K\right)\circ!\left(G\right)\right\rangle \circ\mathit{kcombs}\left(G\right)
	\end{aligned}
	\label{eq: nested-combiantion specification}
\end{equation}
where $\left\langle f,g\right\rangle \left(x\right)=\left(f\left(x\right),g\left(x\right)\right)$,
and $G=1$ if $M=0$ and $G=\left(\begin{array}{c}
	M+D\\
	D
\end{array}\right)-1$ if $M\geq1$.

Then we can rewrite \ref{eq: hodt specification} by following 
\begin{align*}
	\mathit{hodt}_{K}^{M} & =\mathit{min}_{E_{\text{0-1}}}\circ\mathit{concatMap}_{\mathit{sodt}\left(\mathit{xs}\right)}\circ\mathit{kcombs}\left(K\right)\circ\mathit{gen}_{\mathit{splits}}^{M}\circ\mathit{embed}_{M}\\
	& \subseteq\mathit{min}_{E_{\text{0-1}}}\circ\mathit{concatMap}_{\mathit{sodt}\left(\mathit{xs}\right)}\circ\mathit{nestedCombs}\left(K,G\right)\circ\mathit{embed}_{M}
\end{align*}
However, the naive implementation \ref{eq: nested-combiantion specification}
is inefficient due to the overhead of storing intermediate results:
$\mathit{kcombs}\left(G\right)$ produces $O\left(N^{G}\right)$ combinations
for an input of size $N$, which is both memory-intensive and computationally
costly. Additionally, its non-recursive nature prevents the application
of thinning techniques.

Our goal is now to fuse the two-step definition of $\mathit{nestedCombs}\left(K,G\right)$,
given in \ref{eq: nested-combiantion specification} into a single
recursive function. \citet{he2025ROF} presents both divide-and-conquer
(D\&C) and sequential versions for the fused program of $\mathit{nestedCombs}$.
While the D\&C version is more easy to be parallelized, it complicates
element ordering. Since ordering is critical in our applications—where
combinations represent hyperplanes that encode essential predictive
information—we adopt the sequential version of $\mathit{nestedCombs}$
developed by \citet{he2025ROF}. The pseudo-code for this sequential
version is shown in Algorithm \ref{alg:nested combination generator -sequential}.

One source of efficiency in $\mathit{nestedCombs}$ is the representation
of inner combinations as integers ranging from $1$ to $\left(\begin{array}{c}
	N\\
	G
\end{array}\right)$. This is achieved through the clever use of the $\mathit{reverse}$
function, which reverses the order of elements in a list. With this
modification, combinations are arranged in\emph{ revolving-door order}
\citet{he2025CGs}, avoiding explicit representation of combinations.
Instead, their ``ranks'' are used, allowing the information of each
hyperplane to be indexed in $O\left(1\right)$ time using the integer
representation.

\subsubsection{Generating mixed polynomial degree hypersurface splitting rules}

\begin{algorithm}
	\footnotesize
	\begin{enumerate}
		\item \textbf{Input}: $\mathit{xs}$: input data list of length $N$; $K$:
		Outer combination size (nested combination); $\mathit{Gs}=\left[G_{1},G_{2},\ldots G_{l}\right]$:
		a list of inner-combination sizes (ordinary combination), corresponds
		to a list of polynomial degree $\mathit{Ms}=\left[M_{1},M_{2},\ldots M_{l}\right]$
		\item \textbf{Output}: List of $\left(K,Gs\right)$-nested-combinations
		arrays\emph{}\\
		\item $\mathit{css}=\left[\left[\left[\:\right]\right],\left[\right]^{k}\right]$
		// initialize combinations
		\item $\mathit{ncss}=\left[\left[\left[\:\right]\right],\left[\right]^{k}\right]$
		// initialize nested-combinations
		\item $\mathit{asgn}^{+},\mathit{asgn}^{-}=\mathit{empty}\left(\left(\begin{array}{c}
			N\\
			G_{1}
		\end{array}\right)+\left(\begin{array}{c}
			N\\
			G_{2}
		\end{array}\right)+\ldots+\left(\begin{array}{c}
			N\\
			G_{l}
		\end{array}\right),N\right)$ // initialize $\mathit{asgn}^{+},\mathit{asgn}^{-}$ as two empty
		$\left(\begin{array}{c}
			N\\
			D
		\end{array}\right)\times N$ matrix
		\item \textbf{for} $n\leftarrow\mathit{range}\left(0,N\right)$ \textbf{do}:
		//\textbf{$\mathit{range}\left(0,N\right)=\left[0,1,\ldots,N-1\right]$}
		\item $\quad$$\mathit{candInds}=\left[\:\right]$
		\item $\quad$\textbf{for} $j\leftarrow\mathit{reverse}\left(\mathit{range}\left(G_{l},n+1\right)\right)$
		\textbf{do}:
		\item $\quad$$\quad$$\mathit{updates}=\mathit{reverse}\left(\mathit{map}\left(\cup\rho_{M}\left(\mathit{xs}\right)\left[n\right],\mathit{css}\left[j-1\right]\right)\right)$
		\item $\quad$$\quad$$\mathit{css}\left[j\right]=\mathit{css}\left[j\right]\cup\mathit{reverse}\left(\mathit{map}\left(\cup\rho_{M}\left[n\right],\mathit{updates}\right)\right)$
		// update $\mathit{css}$ to generate combinations in revolving door
		ordering,
		\item $\quad$$\quad$\textbf{if} $G\in\mathit{Gs}$ \textbf{do}:
		\item $\quad$$\quad$$\quad$$\mathit{asgn}^{+},\mathit{asgn}^{-}=\mathit{genModels}\left(\mathit{updates},\mathit{asgn}^{+},\mathit{asgn}^{-}\right)$
		// use $G$-combination to generate the positive prediction and negative
		prediction of each hyperplanes
		\item $\quad$$\quad$$\quad$$C_{1}=\left(\begin{array}{c}
			n\\
			G-1
		\end{array}\right)$, $C_{2}=\left(\begin{array}{c}
			n\\
			G
		\end{array}\right)$
		\item $\quad$$\quad$$\quad$\textbf{for} $G^{\prime}\leftarrow\mathit{Gs}$
		\textbf{and $G^{\prime}<G$ do}:
		\item $\quad$$\quad$$\quad$$\quad$$C_{1}=C_{1}+\left(\begin{array}{c}
			N\\
			G^{\prime}
		\end{array}\right)$, $C_{2}=C_{2}+\left(\begin{array}{c}
			N\\
			G^{\prime}
		\end{array}\right)$
		\item $\quad$$\quad$$\quad$$\mathit{candInds}=\mathit{candInds}+\!\!+\mathit{range}\left(C_{1},C_{2}\right)$
		\item $\quad$\textbf{for} $i\leftarrow\mathit{candInds}$ \textbf{do}:
		\item $\quad$$\quad$$\quad$\textbf{for} $k\leftarrow\mathit{reverse}\left(\mathit{range}\left(K,i+1\right)\right)$
		\textbf{do}:
		\item $\quad$$\quad$$\quad$$\quad$$\mathit{ncss}\left[k\right]=\mathit{map}\left(\cup\left[i\right],\mathit{ncss}\left[k-1\right]\right)\cup\mathit{ncss}\left[k\right]$
		// update nested combinations
		\item $\quad$$\mathit{ncss}\left[K\right]=\left[\:\right]$
		\item \textbf{return} $\mathit{ncss}\left[K\right]$
	\end{enumerate}
	\caption{$\mathit{nestedCombs}^{\prime}$\label{alg:nested combination generator with multiple inner -sequential}}
\end{algorithm}

SiSince axis-parallel hyperplanes, general hyperplanes, and hypersurfaces
can all be represented as combinations of data items, it is natural
to ask whether we can construct a single generator capable of producing
all these rules. A decision tree built from such a rule list would
incorporate mixed splitting rules, thereby offering much greater flexibility.

Consider a list of polynomial degrees $\mathit{Ms}=\left[M_{1},M_{2},\ldots M_{l}\right]$,
which corresponds to a list of integers $\mathit{Gs}=\left[G_{1},G_{2},\ldots G_{l}\right]$
ordered in ascending fashion ($G_{1}<G_{2}\ldots<G_{l}$), with $G_{l}$
the largest. We can specify $\mathit{modt}^{\mathit{\mathit{Ms}}}:\mathbb{N}\times\mathcal{D}\to\mathit{DTree}\left(\mathcal{H}^{Ms},\mathcal{D}\right)$
(where $\mathcal{H}^{Ms}=\mathcal{H}^{M_{1}}\cup\mathcal{H}^{M_{1}}\ldots\cup\mathcal{H}^{M_{1}}$
) as
\[
\mathit{modt}_{M}=\mathit{min}_{E_{\text{0-1}}}\circ\mathit{concatMap}_{\mathit{sodt}\left(\mathit{xs}\right)}\circ\mathit{kcombs}\left(K\right)\circ\mathit{gen}_{\mathit{splits}}^{\mathit{Ms}}\circ\mathit{embed}_{\mathit{Ms}}
\]
where $\mathit{gen}_{\mathit{splits}}^{\mathit{Ms}}:\mathcal{D}\to\left[\mathcal{H}^{Ms}\right]$
is the generator that produces all possible splitting rules in $\mathcal{H}^{Ms}$.

Similar to the hypersurface generator, the composed function $\mathit{kcombs}\left(K\right)\circ\mathit{gen}_{\mathit{splits}}^{\mathit{Ms}}$
can also be fused. A key observation about Algorithm \ref{alg:nested combination generator -sequential}
is that, it can generate not only $G$-combinations but all $G^{\prime}$-combinations
for $0\leq G^{\prime}\le G$. However, the nested combination generator
$\mathit{nestedCombs}\left(K,G\right)$ uses only the $G$-combinations
and ignores smaller $G^{\prime}$-combination ($G^{\prime}<G$). To
accomodate this, we need to modify specification \ref{eq: nested-combiantion specification}
to the following 
\begin{equation}
	\mathit{nestedCombs}^{\prime}\left(K,\mathit{Gs}\right)=\left\langle \mathit{setEmpty}\left(G_{l}\right),\mathit{kcombs}\left(K\right)\cdot!!\left(\mathit{Gs}\right)\right\rangle \cdot\mathit{kcombs}\left(G\right),\label{eq: nested-combs-multi-inner specification}
\end{equation}
where $!!\left(\mathit{Gs},\mathit{xss}\right)=\mathit{concat}\left[!\left(G,\mathit{xss}\right)\mid G\leftarrow\mathit{Gs}\right]$
and $\mathit{Gs}=\left[G_{1},G_{2},\ldots G_{l}\right]$ is an ascending
list of integers ($G_{1}<G_{2}\ldots<G_{l}$), with $G_{l}$ being
the largest, corresponds to a list of polynomial degree $\mathit{Ms}=\left[M_{1},M_{2},\ldots M_{l}\right]$.
Specification \ref{eq: nested-combs-multi-inner specification} selects
several $G$-combinations, such that $G\in\mathit{Gs}$, and constructs
$K$-nested combinations from them.

As usual, the naive specification \ref{eq: nested-combs-multi-inner specification}
is inefficient. \citet{he2025CGs} provide an efficient solution by
defining the recursive pattern of $\mathit{nestedCombs}^{\prime}\left(K,\mathit{Gs}\right)$
as 
\begin{equation}
	\begin{aligned} & \bigg<\mathit{setEmpty}\left(G_{l}\right)\cdot\mathit{KcombsAlg}\left(K\right)\cdot\mathit{Ffst},\\
		& \quad\quad\mathit{\mathit{\mathit{KcombsAlg}}\left(K\right)}\cdot\left\langle \mathit{Kcombs}\left(K\right)\cdot!!\left(\mathit{Gs}\right)\cdot convol_{\text{new}}\left(\cup,\circ\right)\cdot\mathit{Ffst},\mathit{KcombsAlg}\left(K\right)\cdot\mathit{Fsnd}\right\rangle \bigg>
	\end{aligned}
\end{equation}
where the detailed definitions for $\mathit{Fst}$, $\mathit{Fsnd}$,
$\mathit{setEmpty}$, $\mathit{KcombsAlg}$ and $\mathit{setEmpty}$
are provided in \citet{he2025CGs} and are omitted here due to space
constraints. Specifically, $convol_{\text{new}}$ is defined as
\begin{equation}
	\begin{aligned}
		& \mathrm{convol}_{\text{new}}\left(\cup, \circ, \mathit{css}_{1}, \mathit{css}_{2}\right) = \\
		& \quad \left[ \mathit{concat}\left( \mathit{zipWith}\left( \circ, \mathit{init}\left( \mathit{tail}\left( \mathinner{css} \right) \right), \mathit{init}\left( \mathit{tail}\left( \mathit{css}_{2} \right) \right) \right) \right) \mid \mathinner{css} \leftarrow \mathit{inits}\left( \mathit{css}_{1} \right) \right].
	\end{aligned}
\end{equation}
The sequential definition of $\mathit{nestedCombs}^{\prime}$ is provided
in Algorithm \ref{alg:nested combination generator with multiple inner -sequential}.
Unlike the original $\mathit{nestedCombs}$, which represents inner
combinations using integers ranging from $1$ to $\left(\begin{array}{c}
	N\\
	G
\end{array}\right)$. $\mathit{nestedCombs}^{\prime}$must account for chains of $G_{1}$-/$G_{2}$...-ombinations
connected consecutively. For a $G$-combinations with rank $r$, its
index in the algorithm is $i=r+\mathit{sum}\left[\left(\begin{array}{c}
	N\\
	G^{\prime}
\end{array}\right)\mid G\in\mathit{Gs},G^{\prime}\le G\right]$, which corresponds to the calculation in lines 13–16 of the algorithm.

Composing $\mathit{nestedCombs}^{\prime}\left(K,\mathit{Gs}\right)$
with $\mathit{min}_{E_{\text{0-1}}}\circ\mathit{concatMap}_{\mathit{sodt}\left(\mathit{xs}\right)}$,
we obtain the algorithm for solving the \emph{optimal decision tree
	problem with mixed splitting rules }(MODT)\emph{
	\begin{equation}
		\mathit{modt}_{K}^{\mathit{Ms}}=\mathit{min}_{E_{\text{0-1}}}\circ\mathit{concatMap}_{\mathit{sodt}\left(\mathit{xs}\right)}\circ\mathit{kcombs}\left(K\right)\circ\mathit{nestedCombs}^{\prime}\left(K,\mathit{Gs}\right)\circ\mathit{embed}_{\mathit{Ms}}
	\end{equation}
}

\subsection{Complexity of solving optimal decision tree problems\label{subsec:Complexity-of ODT problems}}

\subsubsection{Complexity of size-constrained trees}

We present a theoretical analysis of the computational complexity
of size-constrained decision trees. The complexity of depth-constrained
trees can be upper-bounded by noting that a tree of depth $d$ has
at most $2^{d+1}-1$ splitting rules.

\paragraph{Complexity of the hyperplane decision tree problem}

The computational complexity of the optimal hypersurface decision
tree problem is relatively straightforward to analyze. There are $O\left(N^{GK}\right)$
possible combinations of degree $M$ hypersurfaces. The subprogram
$\mathit{sodt}$ has a worst-case complexity of $O\left(K!\right)$.
Thus, solving the HODT problem requires
\begin{equation}
	O\left(N^{G}\times D^{3}\times N+K!\times N\times N^{GK}\right)=O\left(N^{GK+1}\right).
\end{equation}
where $O\left(N^{G}\times D^{3}\times N\right)$ represents the time
required to compute hyperplane predictions. Although this complexity
appears daunting, Part II demonstrates that the actual complexity
of the hypersurface decision tree problem is significantly lower.

Furthermore, the average-case complexity of $\mathit{sodt}_{\text{rec}}$
is substantially lower than its theoretical upper bound. The worst-case
complexity of $\mathit{sodt}_{\text{rec}}$ occurs only when the tree
forms a ``chain'' structure, which, as shown in the empirical analysis
in Part II, is a rare occurrence.

\paragraph{Complexity of the axis-parallel decision tree problem}

For the axis-parallel decision tree (AODT) problem, there are $O\left(N\times D\right)$
possible splits defined by axis-parallel hyperplanes. Consequently,
the complexity of the AODT problem is
\begin{equation}
	O\left(N\times D+K!\times N\times\left(N\times D\right)^{K}\right)=O\left(N^{K}\right).
\end{equation}
Notably, when $K$ splitting rules are fixed, both hypersurface and
axis-parallel decision trees have a worst-case complexity bounded
by $O\left(K!\right)$. However, the number of feasible trees in the
hypersurface case is significantly smaller than in the axis-parallel
case. This difference arises because, in axis-parallel decision trees,
every pair of hyperplanes is \emph{mutually ancestral }(a pair of
hyperplanes is mutually ancestral if both $h_{i}\left(\swarrow\vee\searrow\right)h_{j}$
and $h_{j}\left(\swarrow\vee\searrow\right)h_{i}$ are viable, a case
that rarely occurs in hypersurface decision trees).

\paragraph{Complexity of the decision tree problem over binary feature data}

The complexity of decision trees with binary feature data differs
significantly. Due to axiom 4, if $h_{1}\swarrow h_{2}$ or $h_{1}\searrow h_{2}$,
then $h_{1}\searrow h_{2}$ or $h_{1}\swarrow h_{2}$ is impossible,
i.e. if $h_{2}$ is in the left branch of $h_{1}$, it cannot lie
in the right branch of it.

In contrast, for binary feature data, hyperplanes represent the selection
or non-selection of features, allowing any pair of branch nodes to
be \emph{mutually ancestral} and \emph{symmetric}. Symmetry implies
that any branch node can serve as either the left or right child of
the root. For example, when selecting features $x_{i}$ and $x_{j}$,
both $x_{i}\swarrow x_{j}$ and $x_{i}\searrow x_{j}$ are possible
when $x_{i}$ is the root, and $x_{j}\swarrow x_{i}$ and $x_{j}\searrow x_{i}$
are possible with $x_{j}$ as root. This holds because, in all previous
cases, hyperplanes are characterized by data points, and similarly,
in the axis-parallel tree case, only \textbf{one} data item is needed
to define each hyperplane, and this item always lies on one side of
another hyperplane. However, for binary feature data, this method
of characterization splitting rules as data points are no longer applicable.
In the case of binary feature data, it is only necessary to consider
whether a particular feature is selected or not and there are $K$
features in total used to define decision trees of size $K$. Hence,
for binary feature data, the number of possible decision trees depends
not only on the shape of the tree (which can be counted using the
Catalan number) but also the labels.

Therefore, assume $D$ is much larger than $K$, the worst-case complexity
for the decision tree problem with binary feature data is
\begin{equation}
	O\left(\left(\begin{array}{c}
		D\\
		K
	\end{array}\right)+K!\times\mathrm{Catalan}\left(K\right)\times\left(\begin{array}{c}
		D\\
		K
	\end{array}\right)\right)=O\left(D^{K}\right),
\end{equation}
where $\mathrm{Catalan}\left(n\right)$ is the Catalan number of $n$.

\subsubsection{Complexity of depth-constrained trees}

The combinatorial complexity of the depth-constrained ODT problem
is challenging to compute precisely. A key observation is that, for
arbitrary decision tree of depth-$d$ decision tree it should contains
at least $d$ internal nodes (splitting rules). Consequently, the
search space of depth-$d$ decision trees with respect to data list
$\mathit{xs}$, denoted $\mathcal{S}_{\text{depth}}\left(d,\mathit{xs}\right)$,
is at least as large as the search space of decision trees with $d$
splitting rules, $\mathcal{S}_{\text{size}}\left(d,\mathit{xs}\right)$.
In other words, we have 
\begin{equation}
	\mathcal{S}_{\text{depth}}\left(d,\mathit{xs}\right)\geq\mathcal{S}_{\text{size}}\left(d,\mathit{unique}\left(\mathit{gen}_{\mathit{splits}}^{0}\left(\mathit{xs}\right)\right)\right).
\end{equation}
where $\mathit{unique}\left(\mathit{gen}_{\mathit{splits}}^{0}\left(\mathit{xs}\right)\right)$
denote the all possible \emph{unique} splitting rules generated by
$\mathit{gen}_{\mathit{splits}}^{0}\left(\mathit{xs}\right)$, as
there might have duplicate splitting rules (for categorical data,
distinct data points may have same feature values).

For the axis-parallel decision tree problem, any pair of splitting
rules $r_{i}$ and $r_{j}$ defined by two data points $x_{i}$ and
$x_{j}$, satisfies both $r_{i}\swarrow r_{j}$ and $r_{i}\searrow r_{j}$.
Consequently, by Axiom 4, the relation $r_{i}\overline{\left(\swarrow\vee\searrow\right)}r_{j}$
does not hold for any pair of splitting rules. This implies that any
$d$-combination of splitting rules can construct at least one valid
decision tree, thus 
\begin{equation}
	\mathcal{S}_{\text{size}}\left(d,\mathit{unique}\left(\mathit{gen}_{\mathit{splits}}^{0}\left(\mathit{xs}\right)\right)\right)\geq\mathcal{S}_{\text{kcombs}}\left(d,\mathit{unique}\left(\mathit{gen}_{\mathit{splits}}^{0}\left(\mathit{xs}\right)\right)\right).
\end{equation}
Therefore, the following inequality holds:

\begin{equation}
	\begin{aligned}
		& \mathcal{S}_{\text{depth}}\left(d, \mathit{xs}\right) \geq \mathcal{S}_{\text{size}}\left(d, \mathit{unique}\left(\mathit{gen}_{\mathit{splits}}^{0}\left(\mathit{xs}\right)\right)\right) \geq \\
		& \quad \mathcal{S}_{\text{kcombs}}\left(d, \mathit{unique}\left(\mathit{gen}_{\mathit{splits}}^{0}\left(\mathit{xs}\right)\right)\right) = \left(\begin{array}{c}
			\left|\mathit{unique}\left(\mathit{gen}_{\mathit{splits}}^{0}\left(\mathit{xs}\right)\right)\right| \\
			d
		\end{array}\right)
	\end{aligned}
\end{equation}
As discussed, for axis parallel decision tree problem, $\left|\mathit{unique}\left(\mathit{gen}_{\mathit{splits}}\left(\mathit{xs}\right)\right)\right|$
is the sum of all distinct features with respect to the training data
list $\mathit{xs}$, which is at most $N\times D$ for a size $N$
data list in $\mathbb{R}^{D}$.

\paragraph{Remarks on Quant-BnB algorithm}

In the original source code of \citet{mazumder2022quant}'s Quant-BnB
algorithm, there is a function $\texttt{count\_remain\_3D}$, which
computes the combinatorial complexity of the search space, i.e., the
number of remaining trees in each iteration. The initial call to $\texttt{count\_remain\_3D}$,
calculates this complexity directly from the input data without bounding
parameters. For almost all datasets we have tested, $\texttt{count\_remain\_3D}$
computed results that is \emph{strictly smaller than} $\left(\begin{array}{c}
	\left|\mathit{unique}\left(\mathit{gen}_{\mathit{splits}}^{0}\left(\mathit{xs}\right)\right)\right|\\
	d
\end{array}\right)$.

For instance, in the $\texttt{bidding}$ dataset from \citet{mazumder2022quant}'s
GitHub repository, we observed the numbers of unique values per dimension
as $\mathit{l}=\left[452,377,3,4680,49,22,4589,72,5\right]$ for this
nine-dimensional dataset. Thus, the total number of distinct splitting
rules is $\mathit{unique}\left(\mathit{gen}_{\mathit{splits}}^{0}\left(\mathit{xs}\right)\right)=\mathit{sum}\left(\mathit{l}\right)=10249$.
A depth-3 decision tree should contain at least three splitting rules,
so the search space should contain at least $\left(\begin{array}{c}
	10249\\
	3
\end{array}\right)=179376727124$ possible decision trees. However, the first call to $\texttt{count\_remain\_3D}$
(before any search-space reduction) reports ``$\texttt{Total number of trees = 3781512036}$''.
This contradicts our lower bound, since every $3$-combination of
splitting rule should construct at least one possible decision tree
with depth smaller than 3.

Our analysis suggests a potential discrepancy in the algorithmic approach
used by \citet{mazumder2022quant} for solving the AODT problem, or
it may indicate that their work addresses a specialized variant of
the ODT problem. The lack of clarity in their problem definition makes
it challenging to verify their claims. Further clarification of their
counting formula would be valuable to confirm the validity of their
approach and advance the understanding of this problem.

\section{Conclusion\label{sec:Conclusion}}

The primary distinction between our research and other studies on
ODT algorithms using combinatorial methods lies in philosophy: we
argue that deriving optimal algorithms requires the highest level
of rigor and clarity. Consequently, in this first of two papers, we
introduce a generic algorithmic framework for deriving optimal algorithms
directly from problem specifications. We begin by establishing a formal
standard for designing optimal algorithms, identifying several key
principles that highlight the level of rigor required for problem
definitions in the study of optimal algorithms. These principles resemble
the notion of a ``standard form'' in the study of MIP solvers. Importantly,
when a problem specification satisfies these principles, proving optimality
becomes straightforward. This is analogous to general-purpose MIP
solvers: once a problem is expressed in a standardized form, the solver
can compute the solution directly, while the challenging part lies
in producing the appropriate specification.

To demonstrate the importance and effectiveness of these standards,
we present a concrete example by outlining the algorithmic and geometric
foundations for solving the ODT problem. In the algorithmic section,
we provide four generic and formal definitions of optimal decision
tree problems using the ``proper decision tree'' datatype, a recursive
datatype with four associated axioms. Building on this, we develop
four formal definitions of the ODT problem over arbitrary splitting
rules and objective functions within a specified form. Furthermore,
these formulations yield four solutions to the ODT problem: we prove
the existence of a dynamic programming solution for two and the non-existence
of such a solution for the remaining two. To the best of our knowledge,
these definitions and their corresponding solutions are both generic
and novel.

Note that in this part, we have not yet presented a concrete implementation
of any algorithm for solving optimal decision tree problems. Although
most algorithms presented in this part are ready for implementation,
a critical assumption remains unaddressed—the generation of the ancestry
relation matrix $\boldsymbol{K}$ for a given $K$ splitting rules.
In the discussion of $\mathit{sodt}_{\text{rec}}$, we assume $\boldsymbol{K}$
is pre-stored in memory; however, generating it actually requires
an additional computational process. However, the construction of
the ancestry relation matrix is problem-dependent, which is why we
have not yet defined it, as our goal is to provide a generic solution
rather than algorithms for special cases.

Indeed, in the second paper (Part II) of this series, one of the central
focuses is to develop an incremental algorithm for efficiently constructing
the ancestry relation matrix for hypersurface splitting rules, which
lays the foundation for solving the hypersurface optimal decision
tree problem. Building on the algorithmic and geometric foundations
established here, Part II will present the first two algorithms designed
to address optimal hypersurface decision tree problems.

Specifically, we focus on two algorithms, $\mathit{odt}_{\text{vec}}$
and $\mathit{odt}_{\text{rec}}$—the former offering superior hardware
compatibility, while the latter achieving stronger theoretical efficiency.
We provide comprehensive experimental results analyzing the computational
efficiency of $\mathit{odt}_{\text{vec}}$ and $\mathit{odt}_{\text{rec}}$.
Additionally, we conduct a thorough assessment of the generalization
capabilities of hypersurface decision tree models and axis-parallel
decision tree algorithms, including those learned by CART and \citet{brita2025optimal}'s
ConTree algorithm, using both synthetic and real-world datasets.

\newpage

\vskip 0.2in
\bibliography{Bibliography}

\begin{thebibliography}{44}
\providecommand{\natexlab}[1]{#1}
\providecommand{\url}[1]{\texttt{#1}}
\expandafter\ifx\csname urlstyle\endcsname\relax
  \providecommand{\doi}[1]{doi: #1}\else
  \providecommand{\doi}{doi: \begingroup \urlstyle{rm}\Url}\fi

\bibitem[Aglin et~al.(2020)Aglin, Nijssen, and Schaus]{aglin2020learning}
Ga{\"e}l Aglin, Siegfried Nijssen, and Pierre Schaus.
\newblock Learning optimal decision trees using caching branch-and-bound
  search.
\newblock In \emph{Proceedings of the AAAI Conference on Artificial
  Intelligence}, volume~34, pages 3146--3153, 2020.

\bibitem[Aglin et~al.(2021)Aglin, Nijssen, and Schaus]{aglin2021pydl8}
Ga{\"e}l Aglin, Siegfried Nijssen, and Pierre Schaus.
\newblock Pydl8. 5: a library for learning optimal decision trees.
\newblock In \emph{Proceedings of the Twenty-Ninth International Conference on
  International Joint Conferences on Artificial Intelligence}, pages
  5222--5224, 2021.

\bibitem[Arora et~al.(2016)Arora, Basu, Mianjy, and
  Mukherjee]{arora2016understanding}
Raman Arora, Amitabh Basu, Poorya Mianjy, and Anirbit Mukherjee.
\newblock Understanding deep neural networks with rectified linear units.
\newblock \emph{arXiv preprint arXiv:1611.01491}, 2016.

\bibitem[Bentley(1975)]{bentley1975multidimensional}
Jon~Louis Bentley.
\newblock Multidimensional binary search trees used for associative searching.
\newblock \emph{Communications of the ACM}, 18\penalty0 (9):\penalty0 509--517,
  1975.

\bibitem[Bertsimas and Dunn(2017)]{bertsimas2017optimal}
Dimitris Bertsimas and Jack Dunn.
\newblock Optimal classification trees.
\newblock \emph{Machine Learning}, 106:\penalty0 1039--1082, 2017.

\bibitem[Bird and De~Moor(1996)]{bird1996algebra}
Richard Bird and Oege De~Moor.
\newblock The algebra of programming.
\newblock \emph{NATO ASI DPD}, 152:\penalty0 167--203, 1996.

\bibitem[Bird and Gibbons(2020)]{bird2020algorithm}
Richard Bird and Jeremy Gibbons.
\newblock \emph{Algorithm Design with Haskell}.
\newblock Cambridge University Press, 2020.

\bibitem[Bird(1987)]{bird1987introduction}
Richard~S Bird.
\newblock An introduction to the theory of lists.
\newblock In \emph{Logic of Programming and Calculi of Discrete Design:
  International Summer School directed by FL Bauer, M. Broy, EW Dijkstra, CAR
  Hoare}, pages 5--42. Springer, 1987.

\bibitem[Breiman et~al.(1984)Breiman, Friedman, Stone, and
  Olshen]{breiman1984classification}
L.~Breiman, J.~Friedman, C.J. Stone, and R.A. Olshen.
\newblock \emph{Classification and Regression Trees}.
\newblock Taylor \& Francis, 1984.
\newblock ISBN 9780412048418.
\newblock URL \url{https://books.google.co.uk/books?id=JwQx-WOmSyQC}.

\bibitem[Breiman(2001)]{breiman2001statistical}
Leo Breiman.
\newblock Statistical modeling: the two cultures (with comments and a rejoinder
  by the author).
\newblock \emph{Statistical science}, 16\penalty0 (3):\penalty0 199--231, 2001.

\bibitem[Brita et~al.(2025)Brita, van~der Linden, and
  Demirovic]{brita2025optimal}
C{\u{a}}t{\u{a}}lin~E Brita, Jacobus~GM van~der Linden, and Emir Demirovic.
\newblock Optimal classification trees for continuous feature data using
  dynamic programming with branch-and-bound.
\newblock In \emph{Proceedings of the AAAI Conference on Artificial
  Intelligence}, volume~39, pages 11131--11139, 2025.

\bibitem[Church(1936)]{church1936unsolvable}
Alonzo Church.
\newblock An unsolvable problem of elementary number theory.
\newblock \emph{American journal of mathematics}, 58\penalty0 (2):\penalty0
  345--363, 1936.

\bibitem[Cox et~al.(1997)Cox, Little, O'shea, and Sweedler]{cox1997ideals}
David Cox, John Little, Donal O'shea, and Moss Sweedler.
\newblock \emph{Ideals, varieties, and algorithms}, volume~3.
\newblock Springer, 1997.

\bibitem[De~Berg(2000)]{de2000computational}
Mark De~Berg.
\newblock \emph{Computational geometry: algorithms and applications}.
\newblock Springer Science \& Business Media, 2000.

\bibitem[De~Moor(1994)]{de1994categories}
Oege De~Moor.
\newblock Categories, relations and dynamic programming.
\newblock \emph{Mathematical Structures in Computer Science}, 4\penalty0
  (1):\penalty0 33--69, 1994.

\bibitem[Demirovic et~al.(2022)Demirovic, Lukina, Hebrard, Chan, Bailey,
  Leckie, Ramamohanarao, and Stuckey]{demirovic2022murtree}
Emir Demirovic, Anna Lukina, Emmanuel Hebrard, Jeffrey Chan, James Bailey,
  Christopher Leckie, Kotagiri Ramamohanarao, and Peter~J Stuckey.
\newblock Murtree: optimal decision trees via dynamic programming and search.
\newblock \emph{Journal of Machine Learning Research}, 23\penalty0
  (26):\penalty0 1--47, 2022.

\bibitem[Dijkstra(1972)]{dijkstra1972humble}
Edsger~W Dijkstra.
\newblock The humble programmer.
\newblock \emph{Communications of the ACM}, 15\penalty0 (10):\penalty0
  859--866, 1972.

\bibitem[Fan et~al.(2018)Fan, Li, and Sisson]{fan2018binary}
Xuhui Fan, Bin Li, and Scott Sisson.
\newblock The binary space partitioning-tree process.
\newblock In \emph{International Conference on Artificial Intelligence and
  Statistics}, pages 1859--1867. PMLR, 2018.

\bibitem[Gandy(1988)]{gandy1988confluence}
Robin Gandy.
\newblock The confluence of ideas in 1936.
\newblock In \emph{A half-century survey on The Universal Turing Machine},
  pages 55--111. 1988.

\bibitem[G{\"o}del(1931)]{godel1931formal}
Kurt G{\"o}del.
\newblock {\"U}ber formal unentscheidbare s{\"a}tze der principia mathematica
  und verwandter systeme i.
\newblock \emph{Monatshefte f{\"u}r mathematik und physik}, 38\penalty0
  (1):\penalty0 173--198, 1931.

\bibitem[Gupta et~al.(2000)Gupta, Mukhopadhyay, and Sinha]{gupta2000automatic}
Manish Gupta, Sayak Mukhopadhyay, and Navin Sinha.
\newblock Automatic parallelization of recursive procedures.
\newblock \emph{International Journal of Parallel Programming}, 28\penalty0
  (6):\penalty0 537--562, 2000.

\bibitem[He(2025{\natexlab{a}})]{he2025FoODT_II}
Xi~He.
\newblock Foundational theory for optimal decision tree problems. ii. optimal
  hypersurface decision tree algorithm, 2025{\natexlab{a}}.
\newblock URL \url{https://arxiv.org/abs/2509.12057}.

\bibitem[He(2025{\natexlab{b}})]{he2025ROF}
Xi~He.
\newblock \emph{Recursive optimization: exact and efficient combinatorial
  optimization algorithm design principles with applications to machine
  learning}.
\newblock PhD thesis, University of Birmingham, 2025{\natexlab{b}}.

\bibitem[He and Little(2023)]{he2023efficient}
Xi~He and Max~A Little.
\newblock An efficient, provably exact algorithm for the 0-1 loss linear
  classification problem.
\newblock \emph{ArXiv preprint ArXiv:2306.12344}, 2023.

\bibitem[He and Little(2025{\natexlab{a}})]{he2025CGs}
Xi~He and Max.~A. Little.
\newblock Combination generators with optimal cache utilization and
  communication free parallel execution, 2025{\natexlab{a}}.
\newblock URL \url{https://arxiv.org/abs/2507.03980}.

\bibitem[He and Little(2025{\natexlab{b}})]{he2025odt}
Xi~He and Max~A. Little.
\newblock Proper decision trees: An axiomatic framework for solving optimal
  decision tree problems with arbitrary splitting rules, 2025{\natexlab{b}}.
\newblock URL \url{https://arxiv.org/abs/2503.01455}.

\bibitem[Hooker(2021)]{hooker2021hardware}
Sara Hooker.
\newblock The hardware lottery.
\newblock \emph{Communications of the ACM}, 64\penalty0 (12):\penalty0 58--65,
  2021.

\bibitem[Hu et~al.(2019)Hu, Rudin, and Seltzer]{hu2019optimal}
Xiyang Hu, Cynthia Rudin, and Margo Seltzer.
\newblock Optimal sparse decision trees.
\newblock \emph{Advances in Neural Information Processing Systems}, 32, 2019.

\bibitem[Ibaraki(1977)]{ibaraki1977power}
Toshihide Ibaraki.
\newblock The power of dominance relations in branch-and-bound algorithms.
\newblock \emph{Journal of the ACM}, 24\penalty0 (2):\penalty0 264--279, 1977.

\bibitem[Laurent and Rivest(1976)]{laurent1976constructing}
Hyafil Laurent and Ronald~L Rivest.
\newblock Constructing optimal binary decision trees is np-complete.
\newblock \emph{Information Processing Letters}, 5\penalty0 (1):\penalty0
  15--17, 1976.

\bibitem[Lin et~al.(2020)Lin, Zhong, Hu, Rudin, and
  Seltzer]{lin2020generalized}
Jimmy Lin, Chudi Zhong, Diane Hu, Cynthia Rudin, and Margo Seltzer.
\newblock Generalized and scalable optimal sparse decision trees.
\newblock pages 6150--6160. Proceedings of Machine Learning Research, 2020.

\bibitem[Mazumder et~al.(2022)Mazumder, Meng, and Wang]{mazumder2022quant}
Rahul Mazumder, Xiang Meng, and Haoyue Wang.
\newblock Quant-bnb: A scalable branch-and-bound method for optimal decision
  trees with continuous features.
\newblock In \emph{International Conference on Machine Learning}, pages
  15255--15277. PMLR, 2022.

\bibitem[Meertens(1986)]{meertens1986algorithmics}
Lambert Meertens.
\newblock Algorithmics: towards programming as a mathematical activity.
\newblock In \emph{Proceedings of the CWI symposium on Mathematics and Computer
  Science}, volume~1, pages 289--334, 1986.

\bibitem[Nijssen and Fromont(2007)]{nijssen2007mining}
Siegfried Nijssen and Elisa Fromont.
\newblock Mining optimal decision trees from itemset lattices.
\newblock In \emph{Proceedings of the 13th ACM SIGKDD international conference
  on Knowledge discovery and data mining}, pages 530--539, 2007.

\bibitem[Nijssen and Fromont(2010)]{nijssen2010optimal}
Siegfried Nijssen and Elisa Fromont.
\newblock Optimal constraint-based decision tree induction from itemset
  lattices.
\newblock \emph{Data Mining and Knowledge Discovery}, 21:\penalty0 9--51, 2010.

\bibitem[Ordyniak and Szeider(2021)]{ordyniak2021parameterized}
Sebastian Ordyniak and Stefan Szeider.
\newblock Parameterized complexity of small decision tree learning.
\newblock In \emph{Proceedings of the AAAI Conference on Artificial
  Intelligence}, volume~35, pages 6454--6462, 2021.

\bibitem[Pilanci and Ergen(2020)]{pilanci2020neural}
Mert Pilanci and Tolga Ergen.
\newblock Neural networks are convex regularizers: Exact polynomial-time convex
  optimization formulations for two-layer networks.
\newblock In \emph{International Conference on Machine Learning}, pages
  7695--7705. PMLR, 2020.

\bibitem[Quinlan(2014)]{quinlan2014c4}
J~Ross Quinlan.
\newblock \emph{C4. 5: programs for machine learning}.
\newblock Elsevier, 2014.

\bibitem[T{\'o}th(2005)]{toth2005binary}
Csaba~D T{\'o}th.
\newblock Binary space partitions: recent developments.
\newblock \emph{Combinatorial and Computational Geometry}, 52:\penalty0
  525--552, 2005.

\bibitem[Turing et~al.(1936)]{turing1936computable}
Alan~Mathison Turing et~al.
\newblock On computable numbers, with an application to the
  entscheidungsproblem.
\newblock \emph{J. of Math}, 58\penalty0 (345-363):\penalty0 5, 1936.

\bibitem[van~der Linden et~al.(2023)van~der Linden, de~Weerdt, and
  Demirovic]{van2023necessary}
Jacobus van~der Linden, Mathijs de~Weerdt, and Emir Demirovic.
\newblock Necessary and sufficient conditions for optimal decision trees using
  dynamic programming.
\newblock \emph{Advances in Neural Information Processing Systems},
  36:\penalty0 9173--9212, 2023.

\bibitem[Verwer and Zhang(2019)]{verwer2019learning}
Sicco Verwer and Yingqian Zhang.
\newblock Learning optimal classification trees using a binary linear program
  formulation.
\newblock In \emph{Proceedings of the AAAI Conference on Artificial
  Intelligence}, volume~33, pages 1625--1632, 2019.

\bibitem[Wadler(2015)]{wadler2015propositions}
Philip Wadler.
\newblock Propositions as types.
\newblock \emph{Communications of the ACM}, 58\penalty0 (12):\penalty0 75--84,
  2015.

\bibitem[Zhang et~al.(2023)Zhang, Xin, Seltzer, and Rudin]{zhang2023optimal}
Rui Zhang, Rui Xin, Margo Seltzer, and Cynthia Rudin.
\newblock Optimal sparse regression trees.
\newblock In \emph{Proceedings of the AAAI Conference on Artificial
  Intelligence}, volume~37, pages 11270--11279, 2023.

\end{thebibliography}

\end{document}